\documentclass[journal]{IEEEtran}

\usepackage{times}
 
\usepackage[bookmarks=true]{hyperref}

\usepackage{amssymb}		% to get all AMS symbols
\usepackage{graphicx}		% to insert figures
\usepackage{amsmath}
\usepackage{amsfonts}
\usepackage{amsthm}
\usepackage[utf8]{inputenc}
\usepackage[linesnumbered,ruled,noend]{algorithm2e}
\usepackage{algpseudocode}
\usepackage{xcolor}
\usepackage{cite}
\usepackage{comment}
\usepackage{multirow}
\usepackage{multicol}
\usepackage{wrapfig}
\usepackage{ifthen}
\usepackage{eqparbox}
\usepackage{subcaption}

\begin{document}

% paper title
\title{
% Near-Optimal, Pareto-complete Kinodynamic Motion Planning with Multi-Objectives
Multi-Objective Kinodynamic Motion Planning with Asymptotic Pareto Optimality\\
% Near-Pareto-Optimal Multi-Objective Kinodynamic Planning\\
% Near-Optimal Multi-Objective Kinodynamic Motion Planning
% Multi-Objective Kinodynamic Motion Planning
}
% Multi-objective motion planning: From Lexicographic optimal to Pareto Front

% Authors
% \author{Anonymous Authors}
\author{Yusif Razzaq, Anne Theurkauf, Nisar Ahmed, and Morteza Lahijanian%
% \thanks{Manuscript received: September 8, 2024; Revised: December 4, 2024; Accepted: January 2, 2025.}%
% \thanks{This paper was recommended for publication by Editor Lucia Pallottino upon evaluation of the Associate Editor and Reviewers' comments.}%
% \thanks{This work was supported by NASA STTR award 80NSSC20C0314.}%
\thanks{Authors are with the University of Colorado Boulder, CO, USA. 
        {\tt\small \{firstname.lastname\}@colorado.edu}}%
% \thanks{$^{2}$Author is with Boston Dynamics, MA, USA. 
        % {\tt\small atheurkauf@bostondynamics.com}}%
}

\maketitle

% Environments
\newtheorem{theorem}{Theorem}
\newtheorem{problem}{Problem}
\newtheorem{myexam}{Example}
\newtheorem{assumption}{Assumption}
\newtheorem{proposition}{Proposition}
\newtheorem{corollary}{Corollary}
\newtheorem{lemma}{Lemma}
\newtheorem{definition}{Definition}
\newtheorem{remark}{Remark}

% Common Commands
\newcommand{\reals}{\mathbb{R}}
\newcommand{\A}{\mathcal{A}}
\newcommand{\B}{\mathcal{B}}
\newcommand{\C}{\mathcal{C}}
\newcommand{\E}{\mathcal{E}}
\newcommand{\N}{\mathcal{N}}
\newcommand{\Pcal}{\mathcal{P}}
\newcommand{\R}{\mathbb{R}}
\newcommand{\Scal}{\mathcal{S}}
\newcommand{\U}{\mathbb{U}}
\newcommand{\W}{\mathcal{W}}
\newcommand{\X}{\mathbb{X}}
\newcommand{\Y}{\mathcal{Y}}

% Paper Specific
\newcommand{\traj}{\pi(x)}
\newcommand{\dom}{\succ}
\newcommand{\lexDom}{\succ_{\text{lex}}}

\newcommand{\sol}{\Pi_\text{sol}}
\newcommand{\Sols}{\Pi^\textsc{alg}_\text{sol}}
\newcommand{\poSol}{\Pi^*_\text{sol}}
\newcommand{\poDSol}{\Pi^*_{\delta,\text{sol}}}
\newcommand{\algSol}{\pi_n^{\textsc{alg}}}
\newcommand{\algSols}{\Pi_n^{\textsc{alg}}}
\newcommand{\algRVs}{\Y_n^{\textsc{alg}}}
\newcommand{\algRV}{Y_n^{\textsc{alg}}}
\newcommand{\lexSols}{\Pi^\epsilon_\text{lex}}
\newcommand{\conSols}{\Pi_\text{co}}

\newcommand{\repSet}{R_s}
\newcommand{\lexSet}{L^{\epsilon}}
\newcommand{\repLex}{R^\text{lex}_s}
\newcommand{\repCon}{R^\text{co}_s}
\newcommand{\repPO}{R^\text{PO}_s}

\newcommand{\wit}{s}
\newcommand{\wits}{S}
\newcommand{\witEps}{\vec\varepsilon}

\newcommand{\lex}{\text{lex}}
\newcommand{\epsMin}{\C^{\min}}
\newcommand{\lexMin}{\pi^*_{\text{lex}}}
\newcommand{\conMin}{\pi^*_{\text{co}}}
\newcommand{\lexBall}{x^*_{\text{lex},j}}
\newcommand{\conBall}{x^*_{\text{co},j}}

\newcommand{\lexApprox}{\pi^\epsilon_\text{lex}}
\newcommand{\conApprox}{\pi_\text{co}}

\newcommand{\SST}{\textsc{SST}\xspace}
\newcommand{\lexSST}{\textsc{lexSST}\xspace}
\newcommand{\conSST}{\textsc{coSST}\xspace}
\newcommand{\poSST}{\textsc{poSST}\xspace}
\newcommand{\wSST}{\textsc{ws-SST}\xspace}
\newcommand{\CL}{\textsc{MaxMinClearance}\xspace}
\newcommand{\PL}{\textsc{PathLength}\xspace}
\newcommand{\CF}{\textsc{GaussianCost}\xspace}
\newcommand{\DI}{\textsc{Integrator}\xspace}
\newcommand{\BI}{\textsc{Bicycle}\xspace}
\newcommand{\ws}{\texttt{WS-1}\xspace}
\newcommand{\wss}{\texttt{WS-2}\xspace}
\newcommand{\wsss}{\texttt{WS-3}\xspace}
\newcommand{\wssss}{\texttt{WS-4}\xspace}

\newcommand{\dps}{\delta_\text{PS}}
\newcommand{\ballNodes}{V^\text{PO}_{\mathcal{B}_s}}
\newcommand{\ball}{\mathcal{B}_{\delta_s}}

% Collaborators
\newcommand{\ml}[1]{\textcolor{blue}{[ML: #1]}}
\newcommand{\yr}[1]{\textcolor{red}{[YR: #1]}}

\begin{abstract}
% Motion planning often yields numerous valid solutions, but their quality can vary significantly when evaluated across multiple objectives. 
In this paper, we address the challenge of multi-objective motion planning for systems under kinodynamic constraints. We consider three problem classes: (i) lexicographic optimization, in which objectives are minimized according to a strict priority ordering, (ii) constrained optimization, in which a primary objective is minimized subject to bounds on the remaining costs, and (iii) Pareto front optimization, in which the goal is to approximate the full set of optimal trade-offs among competing objectives. 
% We first prove that no scalar utility function can represent lexicographic dominance over continuous cost spaces, establishing a fundamental limitation of scalarization-based approaches. 
% To overcome this, we introduce $\epsilon$-equivalence, which relaxes strict cost equality to a user-defined tolerance, enabling meaningful secondary-objective refinement in the sampling-based setting. 
We first show that established cost scalarization methods for multi-objective problems \emph{cannot} be extended to continuous-domain systems with correctness guarantees. Then, we propose a unified algorithmic framework built upon the Stable Sparse-RRT (\SST) algorithm, in which the single representative maintained at each witness neighborhood is replaced by a \emph{representative set} of locally Pareto-optimal nodes. This structure gives rise to three distinct algorithms: \lexSST for lexicographic minimization, \conSST for constrained optimization, and \poSST for Pareto-front approximation. 
% We prove that these algorithms are probabilistically $\delta$-robustly Pareto-complete and asymptotically near-Pareto-optimal. 
% planner for continuous dynamical systems with an arbitrary number of cost objectives. 
We provide theoretical guarantees for the completeness and optimality of our algorithms and demonstrate their effectiveness through extensive empirical evaluations.
% We prove that \lexSST is probabilistically $\delta$-robustly complete and asymptotically near-optimal for bi-objective problems; that \coSST is complete and near-optimal for constrained problems, resolving the structural incompleteness of single-representative planners; and that \poSST is the first probabilistically $\delta$-robustly Pareto-complete and asymptotically near-Pareto-optimal planner for continuous, simulatable dynamic systems with an arbitrary number of cost objectives. 
% Finally, we demonstrate the effectiveness of our approach through extensive empirical evaluations, showing that our methods consistently recover high-quality solutions across diverse environments and dynamic models, particularly in scenarios where scalarization techniques fail.
\end{abstract}


\IEEEpeerreviewmaketitle

% \section{Paper Story Outline}
% \input{sections/outline}

\section{Introduction}
\label{sec:intro}

\IEEEPARstart{M}{otion} planning is a core capability for intelligent autonomous systems, governing how dynamic agents move safely and efficiently through their environments. Beyond collision avoidance, effective motion planning requires a notion of trajectory \emph{quality}, which is classically captured by a \emph{single} cost metric, e.g., path length. In practice, however, objectives are rarely singular: an autonomous vehicle must balance passenger comfort against travel time, and a mobile manipulator must trade off workspace clearance against task completion speed.
% , and a UAV must weigh fuel efficiency against sensor coverage. 
When \emph{multiple} competing objectives are present, no single trajectory is universally optimal. Rather, the space of high-quality solutions becomes a rich, multidimensional landscape of trade-offs, where the goal shifts to identifying \emph{Pareto-optimal} solutions, i.e., those that cannot be improved in one objective without sacrificing another.  This introduces several computational challenges, especially under kinodynamic constraints. In this work, we address these challenges and develop a \emph{foundational} framework for \emph{multi-objective kinodynamic motion planning} that enables efficient exploration, approximation, and reasoning over this landscape.

\begin{figure}[t]
    \centering
    \begin{subfigure}[b]{0.24\textwidth}
        \centering
        \includegraphics[width=\textwidth]{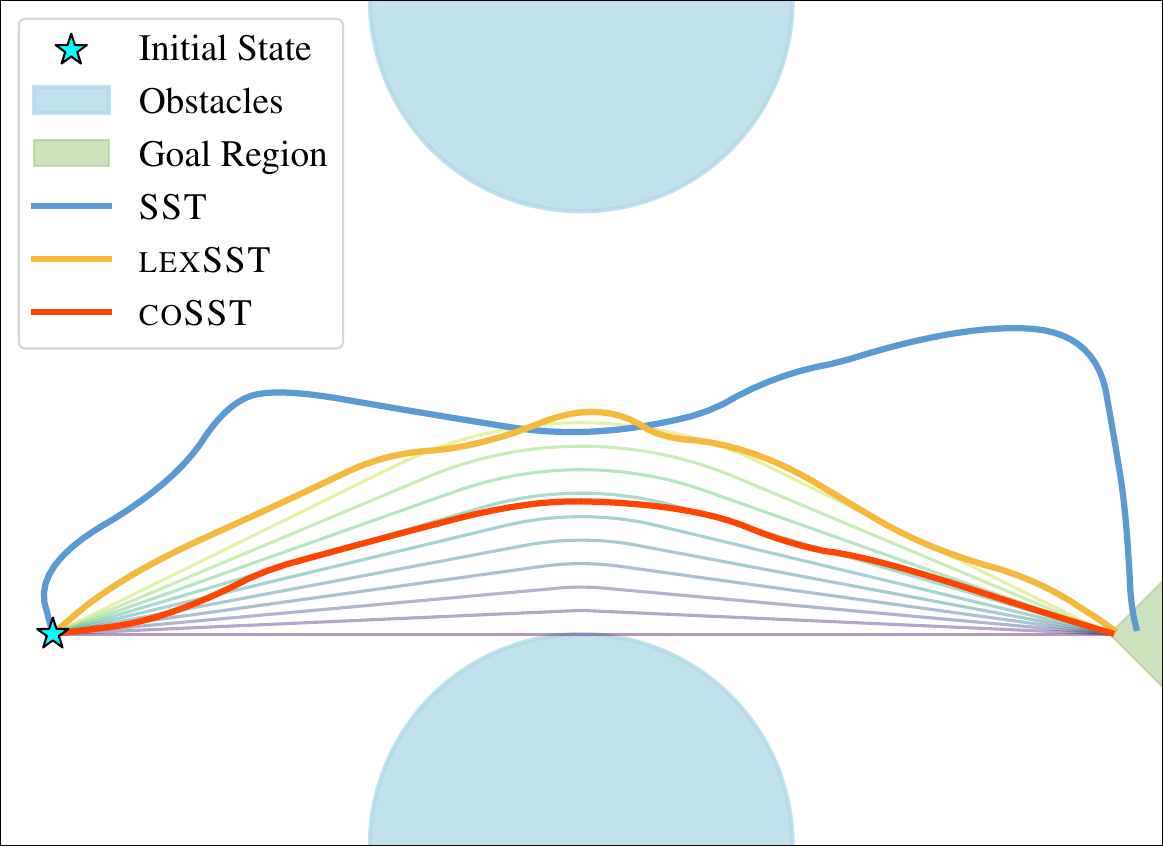}
        \caption{Solution paths in workspace.}
        \label{fig:Ex1a}
    \end{subfigure}
    \hfill
    \begin{subfigure}[b]{0.24\textwidth}
        \centering
        \includegraphics[width=\textwidth]{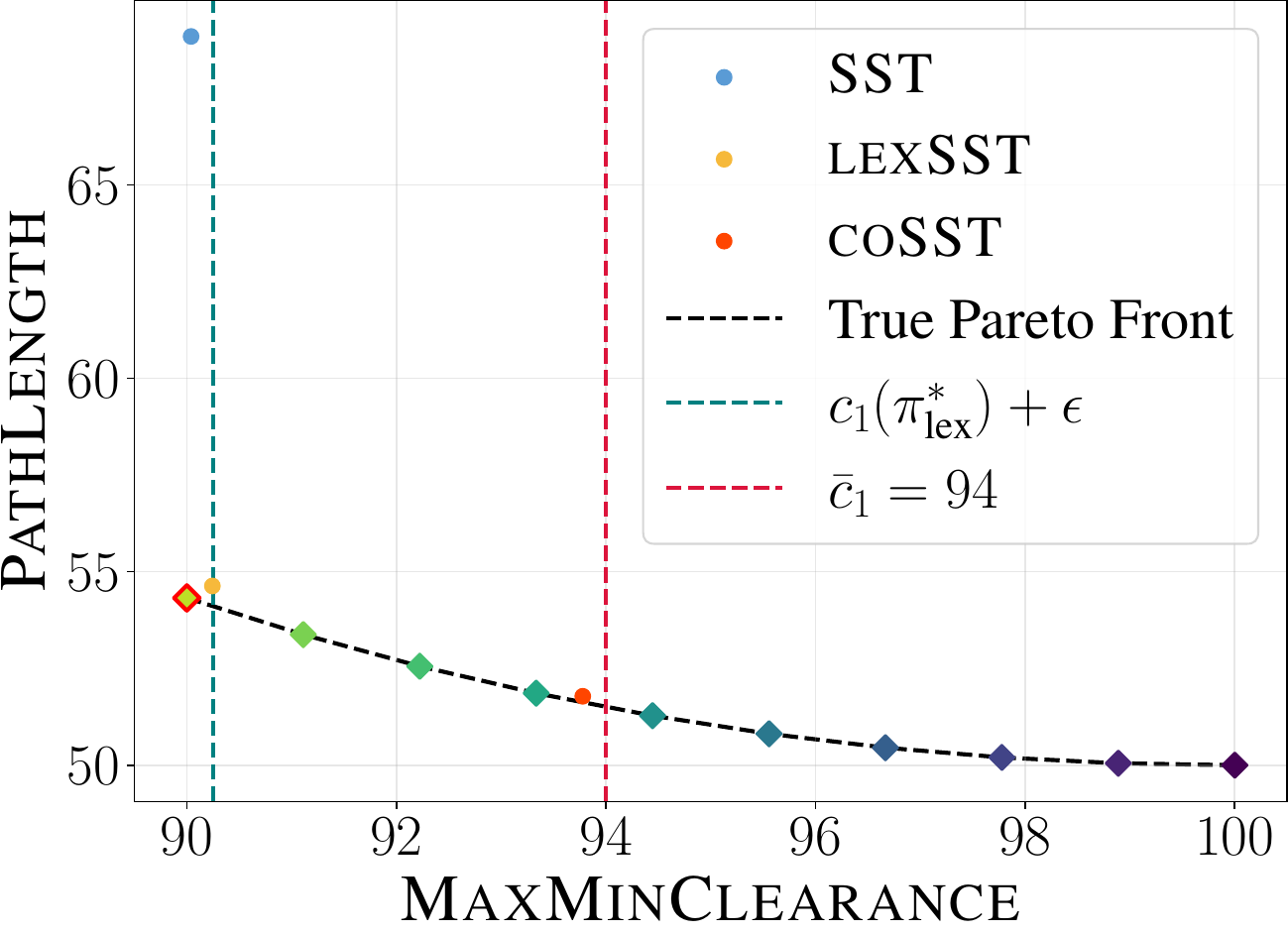}
        \caption{Solutions in objective space.}
        \label{fig:Ex1b}
    \end{subfigure}
    \caption{
    Two-objective motion planning with obstacle clearance ($c_1$) and path length ($c_2$) objectives. (a) Workspace showing obstacles (light blue), goal region (green), theoretical Pareto-optimal trajectories (lines in purple-green color scale), and three planner solutions (thick lines). (b) The bi-objective space with color-coded solution costs and the true Pareto front. 
    % \SST (blue) fails to consider $c_2$, highlighting the need for multi-objective reasoning.
    % Two-objective motion planning with obstacle clearance ($c_1$) and path length ($c_2$). (a) Workspace showing obstacles (light blue), the goal region (green), theoretical Pareto-optimal trajectories (blue--green colorscale), and the solutions returned by three planners (thick lines). (b) Corresponding bi-objective cost space, with the color-coded true Pareto front and the same three solutions plotted as points.
    }
    \label{fig:Ex1}
\end{figure}

Graph-based algorithms provide the earliest foundations for optimal path planning, with A$^*$~\cite{Hart1968} and its variants guaranteeing shortest paths over edge traversal costs. MOA$^*$~\cite{Stewart1991} and NAMOA$^*$~\cite{Mandow2005, Mandow2010} generalize this optimality to multi-objective problems, returning the full set of Pareto-optimal paths. More recent work has dramatically improved efficiency with BOA$^*$~\cite{Hernandez2023} for bi-objective problems, EMOA$^*$~\cite{Ren2025} for arbitrarily many objectives, A$^*$pex~\cite{Zhang2022} for improved pruning of the Pareto front, and multi-objective search under linear temporal logic tasks~\cite{Amorese:IROS:2023}. However, all of these methods require a finite graph representation, whereas real robotic platforms operate in continuous state spaces with complex dynamics.

For continuous systems, constraint-based trajectory optimization methods have been introduced, including CHOMP~\cite{Ratliff2009, Zucker2013}, TrajOpt~\cite{Schulman2014}, and STOMP~\cite{Kalakrishnan2011}. However, these methods are designed for a single cost function and do not natively support multiple objectives (beyond scalarization). Further, they typically struggle with non-convex problems, which are common in practice.

An alternative and more practical paradigm for planning in continuous domains is sampling-based algorithms~\cite{Kavraki1996, LaValle1998, Kuffner2000, Karaman2011}, which have been extended to kinodynamic systems via tree structures in which nodes (states) are expanded by sampling a control and propagating it through the dynamics. Notably, the Stable Sparse-RRT (\SST) planner~\cite{SST, Li2015} established that asymptotic near-optimality can be achieved by maintaining a sparse set of locally optimal nodes within witness neighborhoods, and subsequent work~\cite{Littlefield2018, Dobson2014} further improved efficiency through dominance-informed pruning and sparse roadmaps.  However, these planners remain limited to single-objective optimization.

To extend an existing single-cost planner to multiple objectives, the prevailing approach is to scalarize the costs into a single objective, then apply a standard cost-aware planner. Weighted-sum (WS) scalarization is the most common choice~\cite{Christianos2023, Wilde2020, Lu2020}, but it is structurally limited to recovering solutions on the convex hull of the Pareto front~\cite{branke2008multiobjective}, missing solutions in non-convex regions.
% Weighted-maximum (WM) scalarization~\cite{Moffaert2013, WeightedMax} addresses this limitation and can, in principle, reach any Pareto-optimal point, but the mapping from weights to the resulting trade-off remains opaque, making it impractical to target a specific trade-off, especially in kinodynamic settings.
Weighted-maximum (WM) scalarization~\cite{Moffaert2013, WeightedMax} addresses this limitation and can, in principle, reach any Pareto-optimal point. However, the mapping from weights to the resulting trade-off remains opaque, making it impractical to target a specific trade-off, especially in kinodynamic settings. Recent work~\cite{Botros2024} proposes principled weight-sampling strategies for Pareto-front approximation via scalarization, but each weight vector still requires a separate planning run. To our knowledge, no scalarization approach guarantees asymptotic optimality or completeness for multi-objective planning in continuous domains.

% ACT III: Proposed solution
% \subsection*{Contributions}

% In this work, we address this gap above by introducing a unified framework for multi-objective kinodynamic motion planning that is applicable to general dynamical systems and arbitrary smooth cost functions. 
% We build upon the \SST algorithm, which achieves asymptotic near-optimality for a single objective by maintaining a sparse motion tree whose nodes are locally optimal within small neighborhoods of the state space. Our key insight is the extension from a single representative per witness neighborhood to a \emph{representative set} of locally Pareto-optimal nodes, enabling the simultaneous maintenance and refinement of non-dominated subtrajectories within a single search tree. This structure underlies three proposed algorithms — \lexSST, \conSST, and \poSST — each tailored to a distinct class of multi-objective problems but sharing a common algorithmic foundation. Specifically, our contributions are as follows:

In this work, we address the gap above by introducing a unified framework for multi-objective sampling-based kinodynamic motion planning that is applicable to general dynamical systems and arbitrary smooth cost functions. We specifically consider three common problems in multi-objective optimization: 
% lexicographic ordering, constrained optimization, and Pareto-front computation. 
(i) lexicographic optimization, in which objectives are minimized according to a strict priority ordering, (ii) constrained optimization, in which a primary objective is minimized subject to bounds on the remaining costs, and (iii) Pareto front optimization, in which the goal is to approximate the full set of optimal trade-offs among competing objectives. 
We first show that cost-scalarization methods cannot solve these problems with correctness and completeness guarantees. 

Our key insight is that to approach Pareto optimality, a \emph{representative set} of locally Pareto-optimal nodes must be maintained in the motion tree. This enables the simultaneous tracking and refinement of non-dominated subtrajectories within a single search tree. We then build upon \SST and, by incorporating this idea, propose three algorithms: \lexSST, \conSST, and \poSST, each tailored to one of the aforementioned problems. We prove that these algorithms are probabilistically $\delta$-robustly complete and asymptotically near-optimal. We illustrate their power through eight case studies, comparing them against scalarization-based baselines.

In short, our contributions are sixfold:
\begin{itemize}
    \item We prove that no scalar utility function can correctly represent lexicographic dominance over continuous cost spaces, establishing a fundamental limitation of scalarization-based approaches.
    \item We introduce the notion of $\epsilon$-equivalence, which relaxes the strict equality required by lexicographic dominance to a user-defined tolerance, enabling objective refinement in continuous settings.
    \item We propose \lexSST, the first probabilistically $\delta$-robustly complete and asymptotically near-optimal motion planner for lexicographic minimization of continuous dynamic systems with bi-objective costs.
    \item We propose \conSST, the first probabilistically $\delta$-robustly complete and asymptotically near-optimal motion planner for constrained optimization of continuous dynamic systems, resolving the structural incompleteness of single-representative planners for this problem class.
    \item We propose \poSST, the first probabilistically $\delta$-robustly Pareto-complete and asymptotically near-Pareto-optimal motion planner for approximating the full Pareto front of kinodynamic systems with multiple objectives.
    \item We validate these theoretical properties through extensive empirical evaluations, demonstrating the effectiveness and computational efficiency of our methods against scalarization-based baselines.
    % across diverse environments and dynamic models.
\end{itemize}

% In short, our contributions are as follows:
% \begin{itemize}
%     \item We prove that no scalar utility function can represent lexicographic dominance over continuous cost spaces, a fundamental impossibility result for scalarization, and introduce $\epsilon$-equivalence, which relaxes lexicographic dominance to a user-defined tolerance, enabling objective refinement in continuous settings.
%     \item We propose three probabilistically $\delta$-robust complete and asymptotically near-optimal planners: \lexSST, the first for lexicographic minimization with bi-objective costs; \conSST, the first for constrained optimization, resolving the structural incompleteness of single-representative planners; and \poSST, the first to approximate the full Pareto front for an arbitrary number of objectives. Formal guarantees and proofs accompany all three.
%     \item We validate these properties through extensive experiments, demonstrating effectiveness, robustness, and efficiency against scalarization-based baselines across diverse environments and dynamics.
% \end{itemize}

% \ml{take a look at the references in the weighted-maximum paper, and other multi-objective papers, and see which ones we should discuss here.}

\section{Problem Formulation}
\label{sec:prob}

% \subsection{Problem Formulation}
We consider a robot with motion dynamics given by:
\begin{equation}
\label{eq:dynamics}
    \dot x(t)=f(x(t),u(t)),
\end{equation}
where $x \in \X \subseteq \mathbb{R}^d$ is the state, $u \in \U \subset \mathbb{R}^m$ is the control input, and $f: \X \times \U \to \mathbb{R}^d$ is the vector field.
% where $x \in \X \subseteq \mathbb{R}^n$ and $u\in \U \subset \mathbb{R}^m$ are the sate and control spaces respectively. 
We assume that the dynamics satisfy standard smoothness and regularity conditions. Specifically, the second derivative of $x$ is uniformly bounded, i.e., $\|\ddot{x}\| \leq M$ for some constant $M \in \mathbb{R}_{\geq 0}$, and $f$ is Lipschitz continuous in both $x$ and $u$, i.e., 
for all $x,x' \in \X$ and $u,u' \in \U$,
there exist $K_x,K_u \geq 0$ such that $$\|f(x,u)-f(x',u')\| \leq K_x \|x-x'\| + K_u \|u - u'\|.$$

% The robot evolves in a bounded workspace with static obstacles and has constraints on state variables, e.g., a velocity limit.  We refer to these constraints collectively as the obstacle region $\X_O \subset \X$ in the state space. The collision-free portion of the state space is denoted as $\X_f = \X \setminus \X_O$. 

The robot operates in a bounded workspace containing static obstacles and is subject to state constraints, such as velocity limits. We collectively represent these constraints as the obstacle set $\X_O \subset \X$ in the state space. The collision-free portion of the state space is then defined as $\X_f = \X \setminus \X_O$.  In the classical kinodynamic motion planning problem, the aim is to find a robot trajectory from an initial state $x_0$ to a given goal region $\X_G \subseteq \X_f$ such that every state along the trajectory remains within $\X_f$.

% \ml{don't mention approach.  Just talk about smoothness and regularity properties.}
% Since the proposed approach relies on closed balls used to define neighborhoods in the state space $\X$, certain regularity conditions must be satisfied to ensure desirable theoretical properties. 
% This work considers state spaces that satisfy certain smoothness and regularity conditions. Specifically, $\X$ must be a subset of $n$-dimensional Euclidean spaces, where the $\mathbb{L}_2$ Euclidean norm $\| . \|$ can be defined. 

% We consider a problem in which an agent must navigate from an initial state $x_0$ to a goal region $\X_G \subset \X$, such that every state along its trajectory $\pi$ remains within the collision-free subset $\X_f$. The motion of the agent is described as a trajectory $\pi$:
% 
\begin{definition}[Trajectory]
    Given an initial state $x_0 \in \X$ and control signal $U: [0, T] \rightarrow \U$, a robot trajectory $\pi: [0, T] \rightarrow \X$ is a continuous function such that $\pi(t) = x_0 + \int_{0}^T f(\pi(\tau),U(\tau))d\tau$.
    A trajectory $\pi$ is \emph{valid} if it remains entirely within the collision-free set, i.e.,
    $\pi(t) \in \X_f$ for all $t \in [0, T].$  The set of all valid trajectories is denoted by $\Pi$.
    % 
    % 
    % A trajectory $\pi: [0, T] \rightarrow \X$ is a continuous function over the time interval $[0, T]$, representing the agent's state evolution under the control signal $U: [0, T] \rightarrow \U$, as determined by forward integration of \eqref{eq:dynamics}. A trajectory $\pi$ is \emph{valid} if it remains entirely within the collision-free set:
    % \[\pi(t) \in \X_f \quad \forall t \in [0, T].\]
\end{definition}
\begin{definition}[Motion Plan]
    A valid trajectory $\pi \in \Pi$ is called a \emph{motion plan} if it begins at $x_0$ and reaches a given goal region $\X_G$, i.e., $\pi(t)\in \X_G$ for some $t \in [0, T]$.  The set of all motion plans is denoted by $\Pi_{sol}$.
\end{definition}
In this work, we consider settings in which the robot is subject to multiple motion-cost criteria. That is, each trajectory $\pi$ is evaluated using multiple cost functions, each capturing a different aspect of the trajectory's quality.
\begin{definition}[Cost Function]  
A cost function $ c: \Pi \rightarrow \mathbb{R}_{\geq 0} $ maps a trajectory $ \pi $ to a non-negative scalar cost $c(\pi) \in \mathbb{R}_{\geq 0}$.
\end{definition}
% 
% \ml{move this to the prelims when motion tree is introduced}
% \noindent We denote a trajectory connecting state $x$ to $x'$ as $\overline{x \rightarrow x'}$. For brevity, when a trajectory originates from $x_0$, we apply $c$ directly to the endpoint, that is: $c(x') := c(\overline{x_0 \rightarrow x'})$.

As in \cite{SST}, we assume each cost function is smooth, monotonic, and non-degenerate as formalized below.

% \begin{assumption}
% \label{assumption:CostFunciton}
%     The cost function $c(\pi)$ is:
%     \begin{itemize}
%     \item \emph{Lipschitz continuous} for trajectories with the same initial state, i.e.,
%     $\exists K_c > 0$ s.t.  $\forall \pi_1, \pi_2 \in \Pi \text{ with } \pi_1(0) = \pi_2(0),$
%     $|c(\pi_1) - c(\pi_2)| \leq K_c \sup_{t \in [0, T]} \|\pi_1(t) - \pi_2 (t)\|. $
%     \item \emph{Additive} under concatenation, i.e., for a trajectory $\pi_3 = \pi_1 . \pi_2$ that is a concatenation of two trajectories $\pi_1$ and $\pi_3$, then $c(\pi_3) = c(\pi_1) + c(\pi_2)$.
%     % \ml{what's the following for ?}
%     % \[c(\pi_1) \leq c(\pi_1 . \pi_2).\]
%     \item \emph{Non-degenerate}, i.e.,  
%     % It holds that
%     there exists $M_c > 0$ such that, for all $t_2>t_1$,
%     $(t_2 - t_1) \leq M_c \ |c(\pi_2) - c(\pi_1)| \geq $ 
%     \end{itemize}
% \end{assumption}

\begin{assumption}[\!\!{\cite{SST}}]
    \label{assumption:CostFunction}
    The cost function $c$ is \emph{Lipschitz continuous}, i.e, there exists constant $K_c > 0$ such that
    \[
    |c(\pi) - c(\pi')| \leq K_c \ \sup_{t \in [0,T]} \|\pi(t) - \pi'(t)\|
    \]
    for all $\pi,\pi'\in \Pi$ 
    with the same start state, i.e., $\pi(0) = \pi'(0)$. Furthermore, consider a trajectory $\pi \in \Pi$ with duration $T>0$, and 
    for $t \in [0, T]$, define $\pi^{t}$ and $\pi_{t}$ 
    % \yr{Alternate notation: $\pi_{<t}$ and $\pi_{\geq t}$} 
    as its \emph{prefix} up to time $t$ and \emph{suffix} from time $t$, respectively, so that their concatenation $\pi^{t} \cdot \pi_{t} = \pi$.  
    Then, it holds that, for all $t \in [0,T]$:
    \begin{align*}
        &\text{Additivity:} && c(\pi) = c(\pi^t) + c(\pi_t), \\
        &\text{Monotonicity:} && c(\pi^t) \leq c(\pi),\\
        &\text{Non-degeneracy:} && \exists K'_c > 0, \;\; T - t \leq K'_c \, |c(\pi) - c(\pi^t)|.
    \end{align*}
\end{assumption}

Given $N \in \mathbb{N}_{\geq 2}$ scalar cost functions ${c_1, \ldots, c_N}$ that satisfy Assumption~\ref{assumption:CostFunction}, we aggregate them into a cost vector
\[C(\pi)=(c_1(\pi), \dots, c_N(\pi)) \in \mathbb{R}^N_{\geq 0}.\]
We are interested in high-quality motion plans that account for all pertinent cost metrics.
This naturally motivates a multi-objective framework, in which we specifically focus on three problem classes: \emph{lexicographic-optimal}, \emph{constrained-optimal}, and \emph{Pareto-optimal} planning.

\subsection{Lexicographic-Optimal Motion Planning}

% The ordering of a cost vector $C$ indicates the priority of the constituent cost components as determined by the user. 
% In other words, we prioritize $c_1$ over $c_2$, and then $c_2$ over $c_3$, etc.
% To formalize this,
% we introduce the notion of lexicographic dominance.
Here, we assume that the ordering of a cost vector $C$ reflects the user-defined priorities among its components. Specifically, $c_1$ is prioritized over $c_2$, $c_2$ over $c_3$, and so on. We formalize this prioritization using the notion of lexicographic dominance.

% we introduce the notion of \textit{lexicographic dominance}, denoted by the operator $\lexDom$, as:
% \begin{align}
% \label{eq:lex dominance}
%     \pi \lexDom \pi' \quad &\text{if} \quad \exists \, i \in \{1, \dots, N\} \nonumber \\
%     &\text{s.t.} \quad c_i(\pi) < c_i(\pi') \nonumber \\
%     &\text{and} \quad c_j(\pi) = c_j(\pi') \quad \forall j < i 
% \end{align}
% 

\begin{definition}[Lexicographic dominance]
    \label{def: lexic}
    Given valid trajectories $\pi,\pi' \in \Pi$ with associated cost vectors $C(\pi),C(\pi') \in \mathbb{R}^N_{\geq 0}$, then $\pi$ \emph{lexicographically dominates} $\pi'$, denoted $\pi \lexDom \pi'$, if there exists index $i \in \{1, \dots, N\}$ such that $c_i(\pi) < c_i(\pi')$ and for all $j < i$, $c_j(\pi) = c_j(\pi')$.  
\end{definition}

\begin{definition}[Lexicographic optimal plan]
    A motion plan $\lexMin \in \sol$ is \emph{lexicographically optimal} if it is \emph{not} lexicographically dominated by any other motion plan, i.e., $\nexists \pi' \in \sol$ such that $\pi' \lexDom \pi$.
\end{definition}

Intuitively, $\lexMin$ is a motion plan that minimizes $c_1$ and, among those that do, minimizes $c_2$, and so on. In Fig.~\ref{fig:Ex1a}, $\lexMin$—which maximizes clearance ($c_1$), then path length ($c_2$)—is shown in light green, with its cost vector $C(\lexMin)$ marked as a red-outlined diamond in Fig.~\ref{fig:Ex1b}. Note that the blue trajectory (\SST) minimizes $c_1$ but fails to refine $c_2$; as a result, it is lexicographically dominated by $\lexMin$ (i.e., $\lexMin \lexDom \pi_\text{\SST}$).
% and is not lexicographically optimal.

This motivates our first problem formulation.
\begin{problem}[Lexicographic-optimal motion planning]
    \label{prob:lexicoMP}
    Consider a robot with dynamics described by \eqref{eq:dynamics}, obstacle set $\X_O \subset \X$, an initial state $x_0 \in \X_f = \X \setminus \X_O$, and a goal region $\X_G \subseteq \X_f$. Given a prioritized vector of cost functions $C = (c_1, \ldots, c_N)$, each satisfying Assumption~\ref{assumption:CostFunction}, compute a lexicographic-optimal motion plan $\lexMin \in \sol$.
\end{problem}

\subsection{Constrained-Optimal Motion Planning}
In the second setting, we consider a constrained optimization problem.  That is, we assume user-defined upper bounds on $N - 1$ costs, and our goal is to minimize the last one. 

\begin{problem}[Constrained-Optimal Motion Planning]
    \label{prob:conOptMP}
    % Consider the setting described in Problem~\ref{prob:lexicoMP}. Given a constraint vector $\vec{c} \in \mathbb{R}^N$ where element $i$ is unconstrained (i.e., $\vec{c}_i = \infty$), find a motion plan $\pi^* = \arg\min_{\pi \in \Pi} c_i(\pi)$ such that $c_j(\pi^*) \leq \vec{c}_j \quad \forall j$.
    Consider the setting described in Problem~\ref{prob:lexicoMP}. Given $N$ cost functions $c_1, \dots, c_N$, each satisfying Assumption~\ref{assumption:CostFunction}, and $N-1$ cost upper-bounds $\bar{c}_1, \dots, \bar{c}_{N-1} \in \mathbb{R}^{N-1}_{\geq 0}$, compute an optimal motion plan $\conMin = \arg\min_{\pi \in \sol} c_N(\pi)$ subject to $c_i(\pi^*) \leq \bar{c}_i$ for all $i \in \{1,\ldots,N-1\}$.
\end{problem}

In the bi-objective setting shown in Fig.~\ref{fig:Ex1}, an upper bound $\bar{c}_1 = 94$ is imposed. The solution to Problem~\ref{prob:conOptMP} is then the trajectory whose cost lies at the intersection of the constraint boundary $\bar{c}_1 = 94$ (red) and the true Pareto front (black).

% In the bi-objective setting shown in Figure~\ref{fig:Ex1}, an upper bound of $\bar{c}_1 = 94$ is applied. Then, the solution to Problem~\ref{prob:conOptMP} would return a trajectory with a cost vector that lies at the intersection of the line at $\bar{c}_1 = 94$ (red) and the true Pareto front (black).

% This formulation optimizes along a single cost dimension, as the remaining ones need only to satisfy their respective constraints. 
% We show this problem is computationally difficult with $N \geq 3$.
% 
% However, when two or more objectives are unconstrained, the notion of optimality shifts to a set of optimal trade-offs rather than a single solution. In general, a multi-objective planning problem admits up to $2^N$ possible permutations of constrained elements. Among these, the most comprehensive formulation arises when no constraints are applied, corresponding to the Pareto-optimal case.

\subsection{Pareto-Optimal Motion Planning}
% We now consider the broader case of finding motion plans with Pareto optimal (optimal trade-off) costs.  That is, for any index $i \in \{1,\ldots,N\}$, cost $c_i$ cannot be made more optimal without making cost $c_j$, $j\neq i$, suboptimal.  
Lastly, we consider the setting in which no prioritization or constraints are imposed. Instead, the goal is to find a \emph{set} of motion plans that simultaneously optimize all objectives. Since costs are often competing — e.g., minimizing path length may reduce obstacle clearance — no single trajectory can, in general, optimize every objective. The goal, then, is to characterize the trade-offs among objectives by identifying the set of Pareto-optimal solutions. We formalize this using the notion of dominance.

% A plan is Pareto-optimal if, for any index $i \in \{1, \ldots, N\}$, improving cost $c_i$ necessarily leads to a degradation in at least one other cost $c_j$, where $j \neq i$. Again, 

% Given cost vectors $C,C' \in \mathbb{R}^N_{\geq 0}$, vector $C$ dominates $C'$ ($C \dom C'$) iff 
% 
\begin{definition}[Trajectory dominance]
    A valid trajectory $\pi \in \Pi$ is said to dominate another trajectory $\pi' \in \Pi$, denoted $\pi \dom \pi'$, iff, for every $i \in \{1, \ldots, N\}$, it holds that $c_i(\pi) < c_i(\pi')$.
    % A valid trajectory $\pi \in \Pi$ dominates another trajectory $\pi' \in \Pi$, denoted $\pi \dom \pi'$,
    % iff, for ever $i \in \{1,\ldots,N\}$, cost $c_i(\pi) < c_i(\pi')$.
    % , i.e.,
    % its cost vector dominates that of $\pi'$: 
    % \begin{align*}
    % \label{eq:pareto dominance}
       % \pi \dom \pi' \iff C(\pi) \dom C(\pi') \iff &c_i(\pi) < c_i(\pi') \\ &\forall i \in [1,N]
    % \end{align*}
\end{definition}

We extend the notion of dominance to cost vectors by:  $C(\pi) \dom C(\pi') \iff \pi \dom \pi'$. 

% A trajectory $\pi^* \Pi$ is called \textit{Pareto optimal} if it is not dominated by another trajectory, i.e., $\nexists \pi' \Pi$ s.t. $\pi' \dom \pi^*$.  
% The set of all Pareto 
% For a given motion planning problem, the Pareto front $\sol^*$ is defined as the set of all solutions with non-dominated costs.

% \begin{definition} (Pareto Front)
% The Pareto front $\sol^* \subseteq \sol$ is the set of motion plans such that no element in $ \sol^* $ is dominated by another trajectory in $ \sol $:
% \[ \sol^* = \{ \pi \in \sol \mid \nexists \, \pi' \in \sol \text{ s.t. } \pi' \dom \pi \}. \]
% \end{definition}

\begin{definition}[Pareto optimal \& Pareto front]
    A motion plan $\pi^* \in \sol$ is called \emph{Pareto optimal} if it is not dominated by any other motion plan, i.e., $\nexists \pi' \in \sol$ s.t. $\pi' \dom \pi^*$.  
    The set of all Pareto optimal motion plans is denoted by $\sol^* \subseteq \sol$. The set of cost vectors corresponding to all Pareto-optimal solutions $\C^*$ is called the \emph{Pareto front},
    % We denote the corresponding cost vector with $C(\pi^*)$, and the set of all Pareto-optimal cost vectors is called the \emph{Pareto front}.
    i.e., \[\C^* = \{C(\pi^*) \in \R^N_{\geq 0} \mid \pi^* \in \sol^*\}.\]
\end{definition}

% Each trajectory maps to a point in the objective space by evaluating it with respect to each objective, $C(\pi)$. The Pareto front can then be represent by a set of $N$-dimensional points:
% \[\C^* = \{C(\pi) \in \R^N \mid \pi \in \sol^*\}\]

% \ml{make sure the colors are correct.}
In Fig.~\ref{fig:Ex1a}, samples of Pareto-optimal trajectories are shown using a purple-green color scale. Their corresponding cost vectors, $C(\pi^*) \in \C^*$, are shown as diamonds in Fig.~\ref{fig:Ex1b}.

% Pareto front $C^*$ captures all possible optimal trade-offs of costs.  For a designer, it is powerful to know what optimal choices are possible and which plans achieve them.  In the second problem, we aim to obtain that exactly.

A Pareto front captures all possible optimal trade-offs among the cost objectives. For a designer, it is valuable to understand the range of optimal choices and the motion plans that realize them. We formalize the problem as follows. 
% We formalize the problem to compute the Pareto front by identifying the corresponding set of Pareto-optimal motion plans.

\begin{problem}[Pareto-Optimal Motion Planning]
    \label{prob:paretoOptMP}
    % Given dynamics as in \eqref{eq:dynamics}, state space $\X$ with obstacles $\X_O$ and goal $\X_G$, and an initial state $x_0 \in \X_f$, and $N$ cost objective, find the Pareto front of motion plans. 
    % \begin{align}
    %     \sol^* = &\{\pi \in \sol \mid \nexists \, \pi' \in \sol \text{ s.t } \pi' \dom \pi \}.
    % \end{align}    
    Consider the setting in Problem~\ref{prob:lexicoMP}.  Compute the Pareto front $\C^*$ and the corresponding set of Pareto-optimal motion plans $\sol^*$.
\end{problem}

% \ml{Add a paragraph on the overview of the approach.  Say that computing Lexicographic-optimal and Pareto-optimal motion plans is hard.  Instead, you build on existing near-optimal motion planners to address Problems 1 \& 2. Also, say that it is impossible to compute the entire Pareto front in finite time since it is continuous and requires an infinite number of points.  Instead, we compute a sparse version of it.}

% Multi-objective optimization Problems 1 and 2 are hard.  In fact, finding optimal solutions to kinodynamic motion problems is challenging even in the single-objective case, as it requires exploring both the state space for feasible trajectories and the objective space for cost improvement. In principle, an infinitesimal change to a motion plan could yield a better solution, so guaranteeing true optimality would require enumerating all feasible trajectories—hence, intractable.

% State-of-the-art planners (\cite{SST}, \cite{SPARS2}) address this by introducing sparsity into the search, retaining only trajectories that differ sufficiently in cost or state. This effectively transforms the problem into finding near-optimal solutions, where suboptimality is bounded by a user-defined sparsity metric.

\subsubsection*{Approach Overview}

% Problems~\ref{prob:lexicoMP} and \ref{prob:paretoOptMP} are multi-objective optimization problems subject to (nonlinear) kinodynamic constraints, which are challenging to solve.  Even in the single-objective case, kinodynamic motion planning is challenging because it requires exploring both the state space for feasible trajectories and the objective space for cost improvements. Computing a truly optimal motion plan is generally intractable since even small changes to a plan could yield better solutions. State-of-the-art planners (e.g., \cite{SST, SPARS2}) address this by introducing sparsity into the search, retaining only trajectories that differ sufficiently in cost or state, and using sampling and forward propagation to deal with kinodynamic constraints. This effectively transforms the problem into one of finding near-optimal solutions, where suboptimality is bounded by a user-defined sparsity metric, yielding asymptotically near-optimal planners.

Problems~\ref{prob:lexicoMP}-\ref{prob:paretoOptMP} are multi-objective optimization problems subject to nonlinear kinodynamic constraints; hence, they are nonconvex and difficult to solve. Even in the single-objective setting, kinodynamic motion planning is challenging as it requires searching both the state space for feasible trajectories and the objective space for cost improvements. In particular, computing a truly optimal motion plan is intractable, since even small perturbations to a candidate trajectory can yield strict improvements. 

State-of-the-art planners (e.g., \cite{SST, SPARS2}) mitigate this by introducing sparsity into the search, retaining only trajectories that differ sufficiently in cost or state, while relying on sampling and forward propagation to handle kinodynamic constraints. This relaxes the problem to finding near-optimal solutions with suboptimality bounded by a user-defined sparsity metric.
% , yielding asymptotically near-optimal algorithms.
% 
% We extend this idea to the multi-objective cases in Problems \ref{prob:lexicoMP} and \ref{prob:paretoOptMP}, ensuring near-optimality for all objectives and approximating the Pareto front with arbitrarily tight bounds to capture any desired trade-off.
% 
% Our approach extends this idea to the multi-objective setting, 

We extend these ideas to the multi-objective cases for Problems \ref{prob:lexicoMP}-\ref{prob:paretoOptMP},
by guaranteeing that near-optimal solutions are found with respect to the Pareto-point(s) of interest. For Problem \ref{prob:paretoOptMP}, we approximate the entire Pareto front with arbitrarily tight bounds, enabling accurate representation of any desired trade-off among objectives.

\section{Preliminaries}
\label{sec:prelim}

We begin by reviewing the Stable Sparse-RRT (SST)~\cite{SST} algorithm, which forms the foundation for our proposed algorithms.
% extensions to multi-objective planning. 
% We then discuss the limitations of framing a multi-objective optimization problem as a single-objective problem via cost scalarization, as is commonly considered in the literature.  
We then highlight the fundamental limitations of reducing multi-objective optimization to a single-objective problem through cost scalarization, as commonly adopted in the literature.

\subsection{Stable Sparse RRT (SST)}
\label{subsec:SST}

SST~\cite{SST} is a sampling-based kinodynamic planner that achieves asymptotic near-optimality for a single objective (cost) by growing a sparse motion tree in the state space $\X$ with cost-informed selection and pruning. Like RRT-based planners, it iteratively selects a node, propagates it forward under randomized controls, and adds a new state if feasible. To guide the search, SST samples a random state in $\X$ and selects the best-cost node (state) within a radius of $\delta_\text{BN} > 0$ for extension. 
% In addition, \textsc{SST} employs \textit{witness} nodes, i.e., each a neighborhood defined by a ball of radius $\delta_s > 0$ centered at $x \in \X$ (denoted $\mathcal{B}_{\delta_s}(x) = \{x' \in \X \mid \|x'-x\| \leq \delta_s \}$) is represented by the best cost node in the ball.
In addition, \textsc{SST} employs \textit{witness} nodes: each defining a neighborhood (hyper-ball) of radius $\delta_s > 0$ centered at $x \in \X$ (denoted by $\mathcal{B}_{\delta_s}(x) = \{x' \in \X \mid \|x' - x\| \leq \delta_s\}$), and represented by the lowest-cost node within that ball.
A new state becomes a \textit{representative} if it improves upon the cost of an existing representative, or if it lies outside all existing neighborhoods and creates a new one. This process ensures sparsity while monotonically improving trajectories toward low-cost solutions.

% In order to improve convergence, SST maintains a set of neighborhoods using \emph{witness} nodes $S \in \mathbb{S}$, defined as a ball $\mathcal{B}_{\delta_s}(x) \in \X$ of radius $\delta_s$ around a center state $S.x$. The best-cost node within the neighborhood of a witness is called its \emph{representative} ($S.\text{rep}$). Only such nodes are eligible for extension, finite tree spanning the state space.

% \ml{make this sentence simple and clear.}
% As is standard in the literature, SST solves a slight modification to the classical motion planning problem that finds a $\delta$-similar motion plan to a solution with a dynamic clearance of $\delta$ \cite[Lemma~6]{SST}
% \ml{what do you mean by this?}
% .

Due to the sparsity of the search tree, \textsc{SST} can only probabilistically guarantee the generation of a motion plan that is $\delta$-similar to an optimal solution $\pi^*$. 

\begin{definition}[$\delta$-Similar trajectories] 
    \label{def:delta similar}
    Two trajectories $\pi, \pi'  \in \Pi$ with time durations $T_\pi$ and $T_{\pi'}$, respectively, are \emph{$\delta$-similar} if, for a continuous, non-decreasing scaling function $\sigma:[0, T_\pi] \rightarrow [0, T_{\pi'}]$, it is true that $\pi'(\sigma(t)) \in \mathcal{B}_\delta(\pi(t))$.
\end{definition}

The \emph{obstacle clearance} of a valid trajectory $\pi \in \Pi$
is the minimum distance from $\pi$ to the obstacle set $X_O$. The \emph{dynamic clearance} of $\pi$ is the maximum distance $\delta_d$ that the start and end points of $\pi$ can be displaced such that a new, $\delta_d$-similar trajectory is feasible according to the dynamics in \eqref{eq:dynamics} 
% $\delta_d$ 
% from every point along $\pi$ that can be reached by some admissible controls while remaining in $X_\free$ 
(see Definition~4 and Lemma~6 in \cite{SST} for details).  
A trajectory is called \textit{$\delta$-robust} if both its obstacle and dynamic clearances exceed $\delta$.

% Effectively, this requires that a solution exist with at least $\delta$ clearance from any obstacles along the trajectory, defined as $\delta$-robust. 

% The near-optimality of SST directly extends from this result. Given that SST can find a  $\delta$-similar trajectory to the optimal solution. 

% Additionally, the theoretical analysis of SST establishes it is provably $\delta$-robustly complete and asymptotically near-optimal.

% \begin{definition}
%     ($\delta$-Robust Feasible Motion Planning) Given a dynamical system, the collision free subset, an initial state $x_0$, a goal region $\mathbb{X_G}$, and that a $\delta$-robust trajectory connecting $x_0$ and an $x_g \in \mathbb{X_G}$ exists, find a solution $\pi$ for which $\pi(0)=x_0$ and $\pi(t_\pi) \in \mathbb{X_G}$.
% \end{definition}

% For brevity, we will only provide the main results of the SST proofs here, and encourage the reader to consult \cite{SST} for a more detailed treatment.
    
Consequently, SST ensures the approximation of $\delta$-robust solutions and is therefore probabilistically $\delta$-robustly complete and asymptotically $\delta$-robustly near-optimal, as defined below.

% \begin{definition}[$\delta$-Robust Feasible Motion Planning] 
%     \label{def:deltaRobust}
%     Consider a robot with dynamics described by \eqref{eq:dynamics}, an obstacle region $\X_O \in \X$, an initial state $x_0 \in \X_f$, and a goal region $\X_G \subseteq \X$. Given that a $\delta$-robust trajectory connecting $x_0$ and $\X_G$ exists, find a valid motion plan.
% \end{definition}

\begin{definition}[Probabilistic $\delta$-Robust Completeness]
    \label{def:delta-robust-completeness}
    Let $\Pi_n^\textsc{alg}$ denote the set of trajectories discovered by an algorithm \textsc{alg} at iteration $n$. Algorithm \textsc{alg} is probabilistically $\delta$-robustly complete if, for every motion planning problem where there exists at least one $\delta$-robust solution trajectory, the following holds for all independent runs:
    \begin{equation*}
        \liminf_{n \to \infty} \;\;\; \mathbb{P}\left( \exists \pi \in \Pi_n^{\textsc{alg}} \text{ s.t. } \pi \in \sol \right) = 1.
    \end{equation*}
\end{definition}

\begin{definition}[Asymptotic $\delta$-robust Near-Optimality]
    \label{def:asymp_delta_near_opt}
    Consider a motion planning problem with a single cost and assume there exists at least one $\delta$-robust motion plan. Denote by $c_\delta^*$ the minimum achievable cost over all $\delta$-robust solution trajectories, and let $Y_n^{\textsc{alg}}$ denote a random variable representing the minimum cost among all trajectories returned by algorithm $\textsc{alg}$ after iteration $n$. The algorithm $\textsc{alg}$ is \emph{asymptotically $\delta$-robustly near-optimal} if for all independent runs:
    \begin{align*}
        \mathbb{P}\left(\limsup_{n \to \infty} \;\; Y_n^{\textsc{alg}} \leq h(c_\delta^*, \delta)\right) = 1,
    \end{align*}
    where $h: \mathbb{R}_{\geq 0} \times \mathbb{R}_{\geq 0} \rightarrow \mathbb{R}_{\geq 0}$  is a function of the optimum cost $c_\delta^*$ and the $\delta$ clearance such that $h(c_\delta^*, \delta) \geq c_\delta^*$.
\end{definition}

% Then, given that a $\delta$-robust solution trajectory $\pi$ exists, the analysis of SST shows that a $\delta$-similar trajectory can be found, and is thus probabilistically $\delta$-robust complete. Further, it can find a $\delta$-similar trajectory to the optimal solution $\pi^*$ as the iterations $n\rightarrow \infty$. 

% Since this solution can only deviate from the optimal path by $\delta$, and the cost function is assumed to be Lipschitz continuous (Assumption~\ref{assumption:CostFunciton}), a bound on the sub-optimality can be defined.
% 
% \begin{equation*}
%     C(\pi) \leq (1 +\frac{K_c\cdot \delta}{C_\Delta})\cdot C(\pi^*)  
% \end{equation*}
% 

The probabilistic completeness and near-optimality analyses of our algorithms build on these properties of \SST, extending them to the multi-cost setting. In Section~\ref{sec:analysis}, we show that our approach
maintains the ability to find $\delta$-similar trajectories to any $\delta$-robust solution of interest, which, most generally, includes 
all non-dominated (Pareto-optimal) solutions.

\subsection{Scalarized Cost Function}
\label{subsec:scalarization}

A common approach to multi-objective optimization is to aggregate multiple costs into a single scalar objective, thereby enabling the use of standard cost-aware planners (e.g., SST). In this subsection, we briefly review two common scalarization methods, namely, \emph{weighted sum} (WS) and \emph{weighted maximum} (WM), and explain why they fail to address Problems~\ref{prob:conOptMP} and~\ref{prob:paretoOptMP}. Further, in Section~\ref{subsec:lexPref}, we show that no scalarization method can recover the lexicographic minimum in continuous domains, making them unsuitable for Problem~\ref{prob:lexicoMP}.

\subsubsection{Weighted-Sum (WS) Scalarization}

The most common scalarization is the weighted sum, where the quality of a trajectory $\pi$ is described by a scalar cost 
\begin{align}
    \label{eq: WS cost}
    c_\text{ws}(\pi) = \sum_{i=1}^N w_i c_i(\pi), \qquad w_i \in \mathbb{R}_{\geq 0},
\end{align}    
where weights $w_i$ encode the relative importance of objectives.

Minimizing $c_\text{ws}$ produces a single Pareto-optimal trade-off, $\pi^* = \arg\min_{\pi \in \sol} c_\text{ws}(\pi)$. However, because this scalarization is a linear combination of objectives, it can only recover Pareto points lying on the convex hull of the Pareto front. Consequently, WS minimization cannot capture non-convex Pareto points~\cite {branke2008multiobjective} and is thus incomplete for Problems~\ref{prob:conOptMP} and~\ref{prob:paretoOptMP}.

\subsubsection{Weighted-Maximum (WM) Scalarization}

An alternative approach is the WM scalarization~\cite{WeightedMax}, which substitutes the summation of WS with a maximization operator, 
\begin{align}
    \label{eq: WM cost}
    c_\text{wm}(\pi) = \max_{i \in \{1,\dots,N\}} w_i c_i(\pi).
\end{align}

While WM avoids the convexity limitations of weighted sums, it still yields only a single Pareto point for each weight selection. For Problem~\ref{prob:conOptMP}, although suitable weights may exist to target a particular Pareto point (i.e., $\conMin$), determining them a priori is generally not possible. For Problem~\ref{prob:paretoOptMP}, computing the full Pareto set would require repeatedly reinitializing a planner with different weights and then solving multiple independent planning problems. 
In contrast, our approach maintains and refines non-dominated trajectories within a single search tree, enabling simultaneous approximation of the entire Pareto set.

\section{Lexicographic-Optimal Motion Planning}
\label{sec:lexSST}
% 
% 
% In this section, we present our approach to Problem~\ref{prob:lexicoMP}, lexicographic-optimal motion planning. We first demonstrate that such an ordering relation cannot be captured by any cost scalarization, and then introduce our approach to approximate the solution with bounded error in bi-objective settings. We then extend this method in Section~\ref{sec:poSST} to address any $N$-dimensional cost problems.
In this section, we present our approach to Problem~\ref{prob:lexicoMP}. We first show that a lexicographic ordering relation cannot be captured by any cost scalarization, then introduce a method to approximate the lexicographic-optimal solution with bounded error in bi-objective settings. In Section~\ref{sec:poSST}, we extend this approach to address $N$-dimensional cost spaces.
% and Problem~\ref{prob:paretoOptMP}
% —which targets the full set of Pareto-optimal solutions—
% by allowing users to select their preferred trade-off.

\subsection{Limitations of Extending Lexicographic Ordering to Continuous Domains}
\label{subsec:lexPref}

% Based on answer from https://economics.stackexchange.com/questions/6889/lexicographic-preference-relation-cannot-be-represented-by-a-utility-function.

When costs are limited to \emph{integers}, i.e., vector $C(\pi) \in \mathbb{N}^N$, lexicographic ordering can be neatly captured by a utility function that assigns a unique scalar value to each cost vector according to a fixed priority among dimensions. Specifically, for a set of cost vectors $\mathcal{C} = \{C_1, C_2, \dots, C_N\}$ with each $C_i \in \mathbb{N}^N$, a lexicographic dominance relation $\succ_{\text{lex}}$ can be encoded using a utility function $u : \mathbb{N}^N \to \mathbb{R}$ that respects the ordering, i.e., $C \succ_{\text{lex}} C' \iff u(C) > u(C')$. 

A simple construction can exploit the discreteness of the domain (i.e., $\mathbb{N}$) of each cost component by assigning strictly decreasing scaling coefficients, so that any unit improvement in a higher-priority dimension outweighs the maximum possible variation across all lower-priority dimensions. Thus, the first component determines the primary ordering, the second breaks ties, and so on. For example, for a weight $M \geq \max_{\pi \in \sol} \| C(\pi) \|_\infty$, the utility function
\begin{equation}
% u(C(\pi)) = (M+1)^{N-1} c_1(\pi)  + (M+1)^{N-2} c_2(\pi)  + \cdots + c_N(\pi).
u(C(\pi)) = \sum_{i = 1}^N (M+1)^{N-i} c_i(\pi).
\end{equation}
%
% where $M \in \mathbb{N}_{\geq 0} $ bounds the largest component value.
ensures lexicographic ordering.
However, when costs are defined in \emph{continuous domains}, i.e., $ C(\pi) \in \mathbb{R}^N$, which is the case in this work, lexicographic ordering $\succ_{\text{lex}}$ can \emph{not} be represented by any utility function  $u: \mathbb{R}^N \to \mathbb{R}$. The following theorem formalizes this limitation:

\begin{theorem}
\label{thrm:lexOrder}
% Consider two cost vectors $C, C' \in \mathbb{R}^N_{\geq0}$ and utility function $u : \mathbb{R}^N_{\geq0} \rightarrow \mathbb{R}$, then \emph{there does not exist any function} that represents $\succ_{\text{lex}}$ such that $C \succ_{\text{lex}} C' \iff u(C) > u(C')$ for all $C, C' \in \mathbb{R}^N_{\geq0}$.
There does \emph{not} exist a utility function $u : \mathbb{R}^N \rightarrow \mathbb{R}$ that represents the lexicographic dominance $\succ_{\text{lex}}$, i.e., $\nexists u$ s.t., for every $C,C' \in \reals^{N}_{\geq 0}$, \ $C \succ_{\text{lex}} C' \iff u(C) > u(C')$.
% for all $C, C' \in \mathbb{R}^N_{\geq0}$.
\end{theorem}

The proof is provided in Appendix~\ref{proof:lexOrder}. Intuitively, it relies on the fact that a real-valued utility function has insufficient ``bandwidth" to encode any more than one continuous cost component. To faithfully capture lexicographic ordering, the utility must dedicate every element of its range to resolving differences in the first cost (i.e., $c_1$); then, no values would be available to encode secondary considerations in $c_2$, and so on. 

\begin{corollary}
    % There does not exist a selection of weights $W = \{w_i\}_{i=0}^N$, such that optimization problem $\pi^* = \arg\min_{\pi \in \sol} u^W(\pi)$ is guaranteed to return the lexicographic optimal solution, where $u^W := \{c^W_\text{ws}, c^W_\text{wm}, \dots\}$.
    There do not exist weights $w_1,\ldots,w_N \in \reals_{\geq 0}$ 
    such that the minimization of
    the weighted-sum cost $c_\text{ws}$ in~\eqref{eq: WS cost} or the weighted-max cost $c_\text{wm}$ in~\eqref{eq: WM cost} 
    guarantees the lexicographically optimal solution $\lexMin$.
\end{corollary}

This result implies that Problem~\ref{prob:lexicoMP} cannot be solved by any scalarization method (e.g., WS and WM). Additionally, algorithmic approaches that rely on scalar cost metrics (e.g., \textsc{SST}) or optimization methods that use smooth or continuous utility functions cannot properly model lexicographic objectives. Our proposed solution, detailed below, retains information about the entire cost vector and can therefore explicitly handle the desired preference structure.

% \subsection{$\epsilon$-Equivalence for Lexicographic Minimization in $\R^N$}
\subsection{\texorpdfstring{$\epsilon$}{epsilon}-Equivalence for Lexicographic Minimization in \texorpdfstring{$\R^N$}{R\string^N}}
\label{subsec:e-eqivalence}
Before presenting our approach, we first argue that, given the nature of sampling-based methods, a strict notion of lexicographic dominance is unsuitable; instead, we introduce a notion of $\epsilon$-equivalence, formalized below. 

When sampling from a continuous control space and integrating the dynamics in \eqref{eq:dynamics}, the likelihood of obtaining two trajectories $\pi$ and $\pi'$ that yield identical costs $c_i(\pi) = c_i(\pi')$ is vanishingly small. Consequently, lexicographic dominance effectively reduces to a scalar comparison on the highest-priority cost $c_1$, precluding meaningful consideration of lower-priority costs.

To overcome this, we introduce the notion of $\epsilon$-equivalence, wherein cost components that differ by less than an $\epsilon > 0$ tolerance are treated as ``equivalent''. 
% This allows the planner to look beyond the primary cost $c_1$ when appropriate and incorporate trade-offs among lower-priority objectives. 
%
Specifically, let $\Pi^{\epsilon}_0 \subseteq \sol$ be a set of non-dominated trajectories, and let the minimum primary cost in this set be denoted by $\hat c^*_1 = \min_{\pi\in \Pi^{\epsilon}_0} c_1(\pi)$.
% such that 
% $\bar c^*_1 = \min_{\pi\in \Pi}c_1(\pi)$ 
We first construct a set of $\epsilon$-equivalent trajectories in $c_1$ as 
$$\Pi^\epsilon_1 = \{\pi \in \Pi^{\epsilon}_0 \mid c_1(\pi) \leq \hat c^*_1 + \epsilon_1 \}.$$ 
Then, we can approximate $\lexMin$ by finding the minimum secondary cost within this set, i.e., $\lexApprox = \min_{\pi\in \Pi^\epsilon_1} c_2(\pi)$. 

This procedure can be recursively applied for $N>2$ objectives such that $\Pi^\epsilon_i = \{\pi \in \Pi^\epsilon_{i-1} \mid c_i(\pi) \leq \hat c^*_i + \epsilon_i \}$, and the lexicographic $\epsilon$-optimal trajectory computed by $$\lexApprox = \arg\min_{\pi\in \Pi^\epsilon_{N-1}} c_N(\pi).$$ 
We refer to $\Pi^{\epsilon}_{N-1}$ as the \textit{$\epsilon$-minimal set}.

We now examine the properties of using $\epsilon$-equivalence in bi-objective settings and its limitations for systems with $N>2$ objectives.

\subsubsection{Case of $N = 2$} 
\label{subsec:eps-equivalence-2}
\begin{wrapfigure}{r}{0.225\textwidth}
    \centering
    \vspace{-10pt}
    \includegraphics[width=0.20\textwidth]{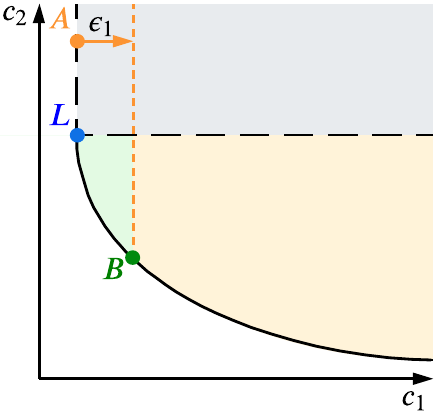}
    \caption{Approximating $\lexMin$ using $\epsilon$-equivalence when $N=2$. 
    % Minimizing $c_2$ returns a point in the green region if $\pi_L$ is found.
    }
    \label{fig:2d_lex}
    \vspace{-10pt}
\end{wrapfigure}

Consider a 2D Pareto front as illustrated in Fig.~\ref{fig:2d_lex}. Let point $A$ represent the cost of a solution trajectory $\pi_A$ whose first cost component $c_1(\pi_A)$ achieves the global minimum $c_1^*$, but whose secondary cost $c_2(\pi_A)$ is significantly suboptimal. 
Let point $L$ denote the cost of the lexicographic optimum $\pi_L = \lexMin$.

Assume a motion planner finds solutions $\hat \pi_A$ and $\hat \pi_L$ that approximate $\pi_A$ and $\pi_L$ arbitrarily closely. Then, $\hat \pi_L$ may not strictly dominate $\hat \pi_A$, but rather satisfy $c_1(\hat \pi_L) \in [c_1^*, c_1^* + \epsilon_1]$. Consequently, $\hat \pi_L$ is guaranteed to be included in the (first-stage) $\epsilon$-equivalent set $\Pi^\epsilon_1$, and minimizing $c_2$ over $\Pi^\epsilon_1$ will ensure a trajectory $\lexApprox$ that lies in the green region (Fig.~\ref{fig:2d_lex}).
% \[
% c_1(\lexApprox) \leq c_1^* + \epsilon_1, \quad c_2(\lexApprox) \leq c_2(\pi_{\text{lex}}).
% \]

Therefore, for a algorithm that can approximate a solution trajectory $\pi \in \sol$ with bounded (cost) error $$\vec e_\textsc{alg} = (|c_1(\pi) - c_1(\hat \pi)|, \ |c_2(\pi) - c_2(\hat \pi)|) \in \reals^{2}_{\geq 0},$$ this approach yields a solution with error $\vec e_\lex = \vec e_\textsc{alg} + (\epsilon_1, 0)$ from the true lexicographic optimal. Clearly,  as $\epsilon_1 \to 0$, $\vec e_\lex \to \vec e_\textsc{alg}$ monotonically.
% 
%  \begin{figure}[tpb]
%     \centering
%     \includegraphics[width=0.2\textwidth]{images/2DParetoFront.pdf} 
%     \caption{Approximating $\lexMin$ using $\epsilon$-equivalence when $N=2$. Minimizing $c_2$ returns a point in the green region if $L$ is found.}
%     \label{fig:2d_lex}
% \end{figure}

% 
\subsubsection{Case of $N > 2$}
\label{subsec:eps-equivalence-3}
Now consider a 3D Pareto front, as shown in Fig.~\ref{fig:3DExample}, where each point on the surface corresponds to a cost vector of a Pareto-optimal trajectory. Using the same procedure, we construct $\Pi^\epsilon_1$ by including all solutions with $c_1 < c_1^* + \epsilon_1$, visualized as a half-space bounded by the orange plane in Fig.~\ref{fig:3DExample half space}. Within this set, we compute $\hat{c}_2^*$, the best observed secondary cost (located at point $B$), and form $\Pi^\epsilon_2$ by retaining only those elements with $c_2 < \hat{c}_2^* + \epsilon_2$, represented by the green plane. 
% \ml{caption says green?}

\begin{figure}[t]
    \centering
    % \begin{subfigure}{0.32\linewidth}
    %         \caption{}
    % \end{subfigure}
    % \begin{subfigure}{0.32\linewidth}
    %         \includegraphics[width=1\textwidth]{images/3DPlane1.png}
    %         \caption{}
    % \end{subfigure}
    \begin{subfigure}{0.36\linewidth}
            \includegraphics[width=1\textwidth]{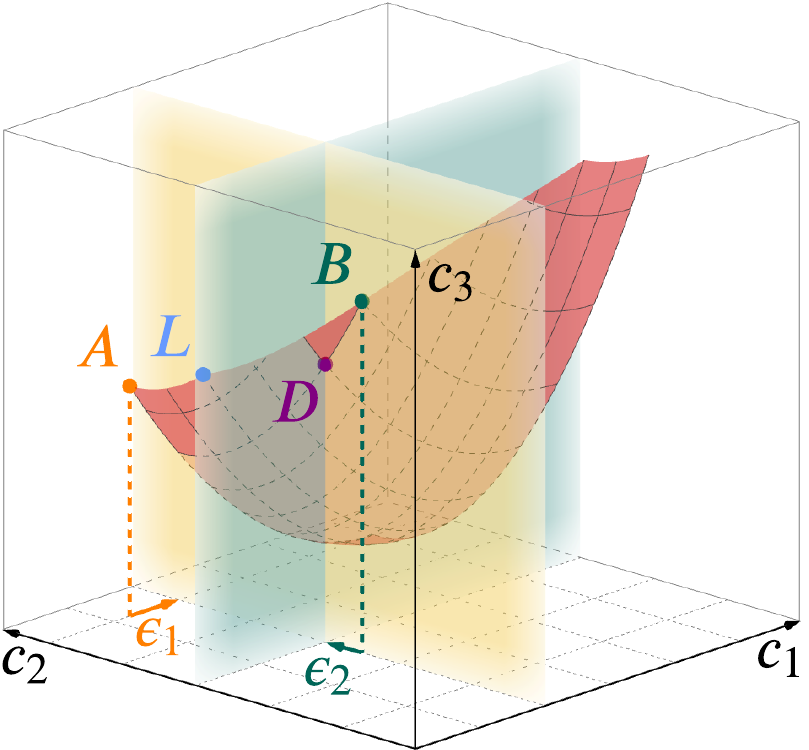}
            \caption{}
            \label{fig:3DExample half space}
    \end{subfigure}
    % \begin{subfigure}{0.32\linewidth}
    %         \includegraphics[width=\textwidth]{images/2DPlane1.png}    
    %         \caption{}
    % \end{subfigure}
    \begin{subfigure}{0.305\linewidth}
            \includegraphics[width=1\textwidth]{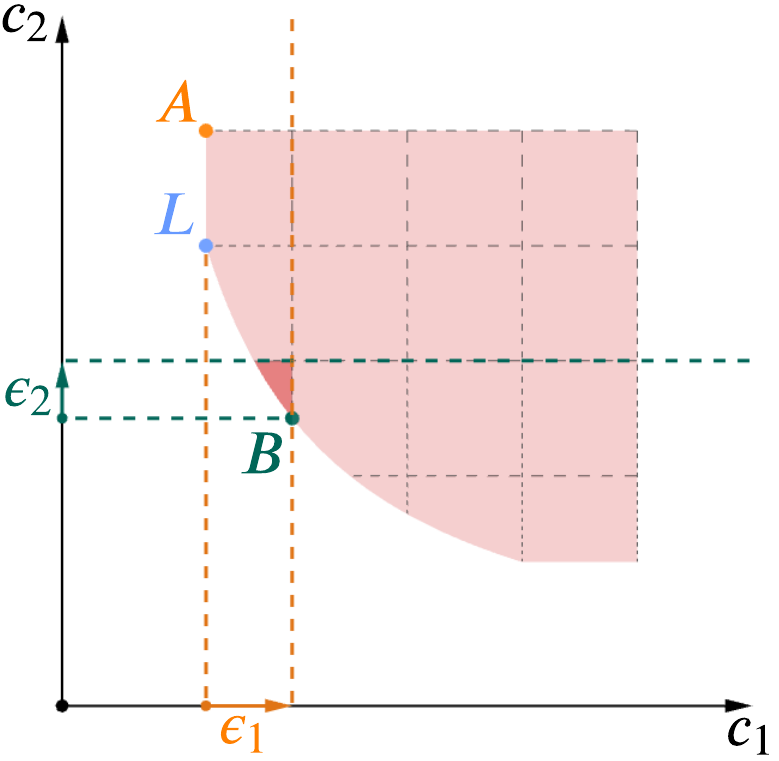}
            \caption{}
    \end{subfigure}
    \begin{subfigure}{0.305\linewidth}
            \includegraphics[width=1\textwidth]{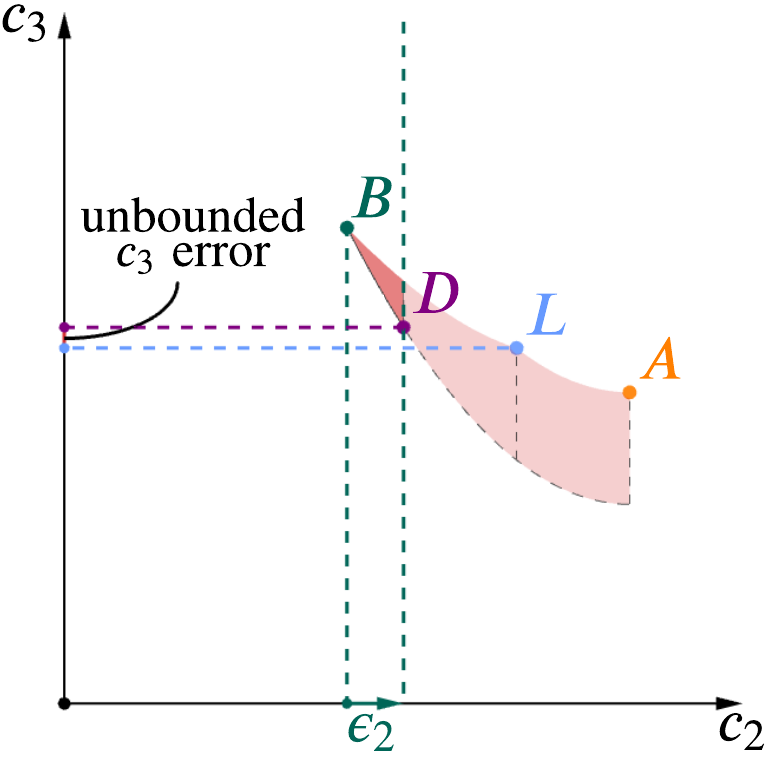}
            \caption{}
    \end{subfigure}
    \caption{Example in which $\epsilon$-equivalence fails to capture a bound on the true lexicographic optimum (point $L$). (a) Pareto front for three cost objectives. Half-space partitions representing filtering at $\Pi^\epsilon_1$ (orange) and $\Pi^\epsilon_2$ (green). (b) Projection onto the $c_1$-$c_2$ plane. (c) Projection onto the $c_2$-$c_3$ plane.} 
    % The procedure in Algorithm~\ref{alg:PruneLexSet} returns the point $D$, the minimization of the set of solutions bounded by the two half spaces $c_1 \leq c_1^* + \epsilon_1 \wedge c_2 \leq \bar{c}_2^* + \epsilon_2$.
    % \yr{Increase font size}\ml{Yes}
    \label{fig:3DExample}
\end{figure}

Crucially, because $\hat{c}_2^*$ may be less than  $c_2(\lexMin)$ (i.e., point $c_2(B) < c_2(L)$), the second filtering step may discard the true lexicographic optimal point altogether (Fig.~\ref{fig:3DExample}b). Consequently, the final set used to minimize $c_3$ may not include $\lexMin$, and the resulting solution $\lexApprox$ may have an arbitrarily poor $c_3$ cost (point D, Fig.~\ref{fig:3DExample}c). That is, $|c_3(\lexApprox) - c_3(\lexMin)| \in [0, \infty).$
% 
% \[
%  | c_3(\lexApprox) - c_3(\lexMin)| \in [0, \infty).
% \]
% 

The source of this failure lies in the compounded filtering: once $\lexMin$ is excluded at any intermediate step $i$, it can no longer influence the minimization of subsequent cost components $c_{i+1}, \dots, c_N$. 
% Since no guarantee exists that $c_2(\pi_{\text{lex}}) \leq \bar{c}_2^* + \epsilon_2$, even arbitrarily small values of $\epsilon_1$ and $\epsilon_2$ cannot prevent this exclusion, especially if the Pareto front has steep geometry. 
Thus, while $\epsilon$-equivalence offers a practical heuristic, it fails to guarantee bounded approximation error for problems with more than two objectives. Nevertheless, our approach for approximating the full Pareto set $\sol^*$ (detailed in Section~\ref{sec:poSST}) can recover a near-optimal solution with respect to the true lexicographic optimal, since, by definition, it is a Pareto-optimal solution (i.e., $\lexMin \in \sol^*$).

% The recursive nature of $\epsilon$-minimal filtering ensures bounded error only when a single threshold is applied—i.e., in two-objective settings. For higher-dimensional cost vectors, the guarantee of recovering the true lexicographic minimum fails unless all intermediate filters include $\pi_{\text{lex}}$, which cannot be ensured by any choice of $\vec\epsilon \neq \mathbf{0}$. Thus, while $\epsilon$-equivalence offers a practical heuristic, it fails to guarantee bounded approximation error for problems with more than two objectives.

\subsection{\lexSST}
\label{subsec:lexSST}

Having clarified these nuances of lexicographic optimization in continuous cost spaces, we now introduce our approach to Problem~\ref{prob:lexicoMP}, lexicographic \SST (\lexSST).

\lexSST is an adaptation of \SST, in that the algorithm iteratively grows a search tree $G = (V, E)$ and uses a set of witness nodes $S \subseteq V$ to define local neighborhoods $\mathcal{B}_{\delta_s}(s)$ about every $s \in S$ (refer to Section~\ref{subsec:SST}). However, instead of each witness maintaining a single representative (lowest cost) node, \lexSST maintains a \textit{representative set} $R_s$ of viable candidates in $\mathcal{B}_{\delta_s}(s)$. 
Our main contribution is the extension from a single representative in the single-cost case to a representative set in the multi-objective setting. 
All our proposed algorithms — \lexSST, \conSST, and \poSST — leverage this structure to solve their respective problems and differ only in their definitions of $R_s$.

In each case, the representative set begins by considering non-dominated nodes at witness $s$. We define the \textit{local Pareto set of nodes} within neighborhood $\mathcal{B}_{\delta_s}(s)$ as,
\[\ballNodes := \left\{ x \in V \cap \mathcal{B}_{\delta_s}(s) \;\middle|\; \nexists \, x' \in V \cap \mathcal{B}_{\delta_s}(s) \text{ s.t. } x' \succ x \right\}.\]

For \lexSST, we define the representative set as $\repLex := L^{\epsilon}(\ballNodes)$,
where $L^{\epsilon}(\cdot)$ returns the $\epsilon$-minimal set of nodes according to the procedure described in Section~\ref{subsec:e-eqivalence}.

\begin{algorithm}[t]
    \caption{\lexSST}
    \label{alg:lexSST}
    \KwIn{$\X$, $\U$, $x_0$, $\Sigma$, $T_{prop}$, $\N$, $\delta _{BN}$, $\delta _s$, $\vec\epsilon$}
    \KwOut{$\lexApprox$}
    $\Sols\leftarrow \emptyset$ \;
    $G = \{V\leftarrow \{x_0\}, E\leftarrow\emptyset\}$\;
    $s_0 \leftarrow x_0, s_0.R = \{x_0\}$, $S \gets s_0$\;
    \For{$\N$ iterations}
    {
        $x_{selected}\leftarrow\textsc{ParetoSelect}(\X, V,\dps)$\;
        $x_{new}\leftarrow\textsc{MonteCarloProp}(x_{selected}, \U, T_{prop})$\;
        \If{$\textsc{CollisionFree}(\overline{x_{selected}\rightarrow x_{new}})$}
        {
            $s \leftarrow \textsc{NearestWitness}(x_{new}, \wits, \delta_s)$ \;
            \If{$s \neq \emptyset$}
            {
                $s.R \leftarrow \textsc{PruneLexSet}(s.R \cup \{x_{new}\})$\;
                $s.R \leftarrow \textsc{PruneDominated}(s.R)$\;
                \If{$x_{new} \in s.R$}{
                    $\textsc{AddNewNode}(x_{new}, x_{selected}, G,\Sols)$
                }
            }
            \Else{
                $s_{new} \leftarrow x_{new}, s_{new}.R = \{x_{new}\}$\;
                $S \leftarrow S \cup \{s_{new}\}$\;
                $\textsc{AddNewNode}(x_{new}, G ,\Sols)$
            }
        }    
    }
    $\lexApprox \leftarrow \arg\min_{\pi \in \Sols} c_N(\pi)$\;
    \Return $\lexApprox$\;
\end{algorithm}

We detail \lexSST in Algorithm~\ref{alg:lexSST}, where $\overline{x \rightarrow x'}$ denotes the tree branch (trajectory) that connects node (state) $x$ to $x'$.
% , and 
% . For brevity, when a trajectory originates from $x_0$, we apply $c$ directly to the endpoint, that is: $c(x') := c(\overline{x_0 \rightarrow x'})$.
The algorithm begins by initializing the vertex $V$, edge $E$, witness $S$, and solution $\Sols$ sets. At each iteration, a node is selected for extension using the \textsc{ParetoSelect} routine (Algorithm~\ref{alg:ParetoSelect}), which uniformly samples a non-dominated node within a ball of radius $\dps$ centered at a randomly selected state $x_{rand}$ returned by \textsc{RandomSample}.

\begin{algorithm}[t]
    \caption{$\textsc{ParetoSelect}$}
    \label{alg:ParetoSelect}
    \KwIn{$\X$, $V$, $\dps$}
    $x_{rand}\leftarrow \textsc{RandomSample}(\X)$ \;
    $X_{near} \leftarrow \textsc{Near}(x_{rand},V,\dps)$\;
    $X_{pareto}\leftarrow \textsc{PruneDominated}(X_{near})$ \;
    $x_{selected}\leftarrow \textsc{RandomSample}(X_{pareto})$ \;
    \Return $x_{selected}$
\end{algorithm}

\begin{algorithm}[t]
    \caption{$\textsc{MonteCarloProp}$}
    \label{alg:MCProp}
    \KwIn{$x_{selected}$, $\U$, $T_{prop}$}
    $t_{prop} \leftarrow \textsc{Sample}([0, T_{prop}]); \ u \leftarrow \textsc{Sample}(\U)$\;
    \Return $x_{new}\leftarrow x_{selected} + \int_0^{t_{prop}} f(x(t), u) dt$ \;
\end{algorithm}

Once a node is selected, it is extended using \textsc{MonteCarloProp} (Algorithm~\ref{alg:MCProp}), which samples a random control and duration to propagate $x_\text{selected}$ according to \eqref{eq:dynamics}, yielding a new candidate state $x_\text{new}$. The nearest witness $s_\text{near}$ is then identified via \textsc{NearestWitness}: if $s_\text{near}$ lies farther than $\delta_s$ from $x_\text{new}$, a new witness is created and associated with $x_\text{new}$; otherwise, $x_\text{new}$ is evaluated against the current representative set of $s_\text{near}$.
% \ml{Alg 4 is too simple to have an explicit env for it.  It should be removed and just explained what it does here.} \yr{removed}

% \begin{algorithm}[htbp]
%     \caption{$\textsc{NearestWitness}$}
%     \label{alg:NearestWitness}
%     \KwIn{$x_{new}$, $\wits$, $\delta_s$}
%     $\wit_{near}\leftarrow \textsc{Nearest}(x_{new},\wits)$ \;
    
%     \If{$|\wit.x - \wit_{near}.x| < \delta_s$}
%     {
%         \Return $\wit$
%     }
%     \Else{
%         \Return $\emptyset$
%     }
% \end{algorithm}

If $x_\text{new}$ is not pruned by either \textsc{PruneLexSet} or \textsc{PruneDominated} (Algorithms~\ref{alg:PruneLexSet} and~\ref{alg:PruneDominated}, respectively), then it is added to $R^\text{lex}_{s_\text{near}}$ and the motion tree. As well, the new node may result in the pruning of existing nodes. Nodes entering the goal region are added to the solution set $\Sols$ (Algorithm~\ref{alg:AddNewNode}).

\begin{algorithm}[tb]
    \caption{$\textsc{PruneLexSet}$}
    \label{alg:PruneLexSet}
    \KwIn{$R$,  $\vec{\epsilon}$}
    \For{$i = 1$ \KwTo $N-1$} {
        $\epsMin \leftarrow \{C(x) \ \forall x \in R\}$ \;
        $\hat{c}_i^* \leftarrow \min(c_i \in \epsMin)$\;
        \ForEach{$x \in R$} {
            \If{$c_i(x) > \hat{c}_i^* + \epsilon_i$}{
                Remove $x$ from $R$ and $V$\;
            }
        }
    }
    \Return $R$
\end{algorithm}

\begin{algorithm}[tbp]
    \caption{$\textsc{PruneDominated}$}
    \label{alg:PruneDominated}
    \KwIn{$R$, $\vec\varepsilon$}
    $\ballNodes \leftarrow \emptyset; \quad R^{\varepsilon} \gets \emptyset$ \;
    \ForEach{$x \in R$} {
        \If{not $\exists \, x' \in R$ such that $C(x') \succ C(x)$}{
            Add $x$ to $\ballNodes$\;
        }
    }
    \ForEach{$x \in \ballNodes$} {
        \If{not $\exists \, x' \in \ballNodes$ s.t. $\vert c_i(x') - c_i(x)\vert < \varepsilon_i \ \forall i$}{
            Add $x$ to $R^{\varepsilon}$\;
        }
    }
    $G.V \gets G.V \setminus (R \setminus R^{\varepsilon})$\;
    \Return $R^{\varepsilon}$
\end{algorithm}

\begin{algorithm}[tbp]
    \caption{$\textsc{AddNewNode}$}
    \label{alg:AddNewNode}
    \KwIn{$x_{new}$, $x_{selected}$, $G$, $\Sols$}
    $G.V \leftarrow G.V \cup \{x_{new}\}$\;
    $G.E\leftarrow G.E\cup\{x_{selected}, x_{new}\}$\;
    \If{$x_{new} \in \X_G$}
    {
        $\Sols \leftarrow \textsc{PruneLexSet}(\Sols \cup \{x_{new}\})$\;
        $\Sols \leftarrow \textsc{PruneDominated}(\Sols)$\;
    }
\end{algorithm}
% The main contribution in \lexSST is maintaining an $\epsilon$-minimal set $\repSet$ for each neighborhood, rather than a single ``best" representative node as in \textsc{SST}.
% % The set $\repSet$ effectively represents the last filtered set $\Pi^\epsilon_N$ as described in Subsection~\ref{subsec:e-eqivalence}. 
% Accordingly, each witness is redefined as:
% % 
% \begin{equation*}
%     \wit = (s, R_s),
% \end{equation*}
% % 
% where $s \in \X$ is the center of the neighborhood and $R \subset \X$ is the representative set. For \textsc{lexSST}, $R = \repSet$ and contains the endpoints of all sub-trajectories in the current motion tree $G$ entering the ball $\mathcal{B}_{\delta_s}(s)$ that satisfy:
% % 
% \[
% \repSet = \{x \in G \mid x\in \mathcal{B}_{\delta_s}(S.s) \wedge \overline{x_0\rightarrow x} \in \Pi^\epsilon_{N-1} \}.
% \]
% % 
% \ml{what does $\overline{x_0\rightarrow x}$ mean? The trajectory from $x_0$ to $x$?  If so, use $\overline{x_0x}$ instead?}
% where $\Pi^\epsilon_{N-1}$ is the filtered set all the trajectories ending in $\mathcal{B}_{\delta_s}(s)$ according to the procedure described in Subsection~\ref{subsec:e-eqivalence}.

Since the number of non-dominated nodes, and thus active representatives, can grow without bound, we introduce a cost-space sparsity metric $\vec{\varepsilon}$. This metric ensures that each set $\repSet$ retains only nodes that are sufficiently different in cost, i.e.,
\[\repSet^\varepsilon = \{x \in \repSet \mid \forall x' \in \repSet, \, \exists i \text{ s.t. } |c_i(x) - c_i(x')| > \varepsilon_i \}.\]
As the algorithm progresses, the state-space sparsity metric $\delta_s$ (inherited from \SST) ensures a finite number of neighborhoods $\ball(s)$ are reached, while $\vec{\varepsilon}$ ensures a finite number of representatives within each neighborhood. 
In Section~\ref{sec:analysis}, we show that $\vec{\varepsilon}$ directly affects the near-optimality of our proposed algorithms,
% via the sub-optimality bound in \eqref{eq:bounds}, 
and in Sec.~\ref{sec:eval}, we empirically demonstrate the trade-off between sparsity and solution quality by varying $\vec{\varepsilon}$.
% We study the effects of $\vec{\varepsilon}$ on near-optimality in Section~\ref{sec:analysis}.

The algorithm continues for a fixed time budget, incrementally growing the motion tree $G$ and maintaining a set of candidate solution trajectories $\Sols$. Upon termination, it returns the solution that minimizes the final cost component $\lexApprox = \arg\min_{\pi \in \Sols} c_N(\pi)$.

\section{Constrained-Optimal Motion Planning}
\label{sec:conSST}

Here, we introduce our proposed algorithm, \textit{constrained-optimal} SST (\conSST), for solving Problem~\ref{prob:conOptMP}, with the goal of returning the trajectory that minimizes cost $c_N$ subject to $N-1$ constraints.
Similar to \lexSST, we define a representative set $\repCon$ for each witness, but use a different pruning procedure to ensure efficient exploration while respecting the user-defined constraints $\bar{c}_i$.

Effectively, this replaces the shifting window of \lexSST, retaining nodes with $c_i < \hat{c}_i^* + \epsilon_i$, with a static window keeping nodes with $c_i < \bar{c}_i + \bar \epsilon_i$. The buffer term $\bar\epsilon_i$ is introduced to ensure that a $\delta$-similar trajectory to the true solution $\conMin$ is not pruned prematurely, as discussed in Section~\ref{subsec:analyzeConSST}. Formally, we define the representative set as
\[\repCon = \{x \in \ballNodes \mid c_i(x) \leq \bar{c}_i + \bar \epsilon_i \quad \forall i \in \{1,\ldots,N-1\} \}.\]
As with \lexSST, $\repCon$ only contains non-dominated nodes.

% \ml{combine these into one $P_C$ by integrating $P_C^-$ conditions with $P_C$? Is the notation $\textbf{C}$ defined? What's $s$ in $\mathcal{B}_{\delta_s}(s)$ of $P_C^-$?}

Intuitively, each witness neighborhood maintains a slice of its local Pareto front, defined by the intersection of half-spaces $c_i \leq \bar{c}_i + \bar \epsilon_i$ for cost dimensions $i \in \{1, \dots, N-1\}$. As the algorithm progresses, nodes are selected from $\repCon$ for extension, ensuring all viable subtrajectories are maintained. Since costs are monotonic, we can safely prune nodes that exceed this threshold and focus our search on more promising ones. Further, if an admissible cost-to-go heuristic $g: \X \times 2^{\X} \rightarrow \R_{\geq 0}$ is available, this pruning can be made more aggressive by retaining nodes that satisfy $c_i(x) \leq \bar{c}_i + \bar\epsilon_i - g(x, \X_G)$.

The algorithm for \textsc{coSST} follows that of Algorithm~\ref{alg:lexSST}, except that we replace \textsc{PruneLexSet} in Line 10 with \textsc{PruneCoSet} (Algorithm~\ref{alg:PruneConSet}). Upon termination, \conSST returns the solution that minimizes the unconstrained cost, i.e., $\conApprox = \arg\min_{\pi \in \Sols} c_N(\pi)$.
\begin{algorithm}[tb]
    \caption{$\textsc{PruneCoSet}$}
    \label{alg:PruneConSet}
    \KwIn{$R$,  $\{\bar{c}_i\}_{i=0}^{N-1}$}
    \For{$i = 1$ \KwTo $N-1$} {
        \ForEach{$x \in R$} {
            \If{$c_i(x) > \bar{c}_i + \bar\epsilon_i$}{
                Remove $x$ from $R$\;
            }
        }
    }
    \Return $R$
\end{algorithm}
\subsection{Limitations of SST for Constrained Optimization}
\label{subsec:SSTlimitaion}
The solution to Problem~\ref{prob:conOptMP} involves minimizing a single cost while satisfying the remaining $N-1$ constraints, yielding an optimal trajectory $\conMin$. Unlike lexicographic optimization, no $\epsilon$-equivalence is required for tie-breaking, which might suggest that vanilla \textsc{SST} is sufficient. However, we show that because \textsc{SST} maintains only a single representative per witness, it cannot account for downstream constraint satisfaction and is therefore \emph{incomplete} for Problem~\ref{prob:conOptMP}.

Consider the flytrap example in Fig.~\ref{fig:conSSTcase}, where the agent must navigate from the start (yellow star) to the goal (red region) via a narrow passage. The goal is to minimize path length while satisfying a constraint on the line integral over a cost function shown in greyscale, i.e.,
% . Formally:
% 
% \begin{align*}
%     \textsc{CostFunction}: \quad c_1(\pi) &= \int_\pi \mathcal{G}(x) \, dx, \\
%     \textsc{PathLength}: \quad c_2(\pi) &= \int_\pi dx.
% \end{align*}
%
$c_1(\pi) = \int_\pi \mathcal{N}(y) dy$, 
where $\mathcal{N}(y)$ is a Gaussian function
% \ml{what's the co-variance? Gaussian functions are defined by two parameters: mean and (co)-variance}
evaluated at state $y$, and 
$c_2(\pi) = \int_\pi dy$.

The narrow passage can accommodate only a few witnesses through which all candidate solutions must pass. 
Now, consider three subtrajectories $\pi_a, \pi_b, $ and $ \pi_c$ that arrive at witness $s$. 
Naively applying \SST to optimize for path length ($c_2$) while pruning $c_1 \leq \bar{c}_1$ would select $x_b$ as the representative of $s$, even if it nearly exceeds $\bar{c}_1$. Then, nodes $x_a$ and $x_c$ are pruned since they do not improve upon $c_2$. However, the minimum cost-to-go from $x_b$ may exceed the cost threshold (i.e., $c_1(\pi_B) > \bar{c}_1$), rendering SST incomplete for Problem~\ref{prob:conOptMP}. Nearby witnesses would exhibit similar behavior, as the shortest path to any witness in the passage would pass through the high-cost region. In Section~\ref{sec:cs4}, we demonstrate this limitation empirically.

% Notably, the narrow passage only allows for one witness neighborhood to be defined, thereby forcing all potential solution trajectories to pass through it. Now consider a constraint of $c_1 \leq \textbf{c}_1$ is applied to the objective \textsc{CostFunction}. Then, the desired trajectory $\pi^*_C$ would be a motion plan that minimizes \textsc{PathLength} while respecting the cost constraint. Now consider three subtrajectories, $\pi_a, \pi_b,$ and $ \pi_c$ that may be found at any iteration of the algorithm. Naively, we could use \textsc{SST} to optimize for \textsc{PathLength} while applying the constraint $c_1 \leq \textbf{c}_1$ to each witness neighborhood. Doing so, however, could admit a node $x_b$ that approaches but does not exceed $\textbf{c}_1$, consequently omitting nodes $x_a$ and $x_c$ from addition to the search tree. Then, the incurred cost of moving from $x_b$ to the goal may exceed the set cost threshold, leading the algorithm to fail to find a path. In short, if the node $x_b$ becomes the representative of witness $S$, then no path will exist to the goal that respects the constraints, rendering  \textsc{SST} incomplete for Problem~\ref{prob:conOptMP}. 

% 
\begin{figure}[tb]
    \centering
    \includegraphics[width=0.48\textwidth]{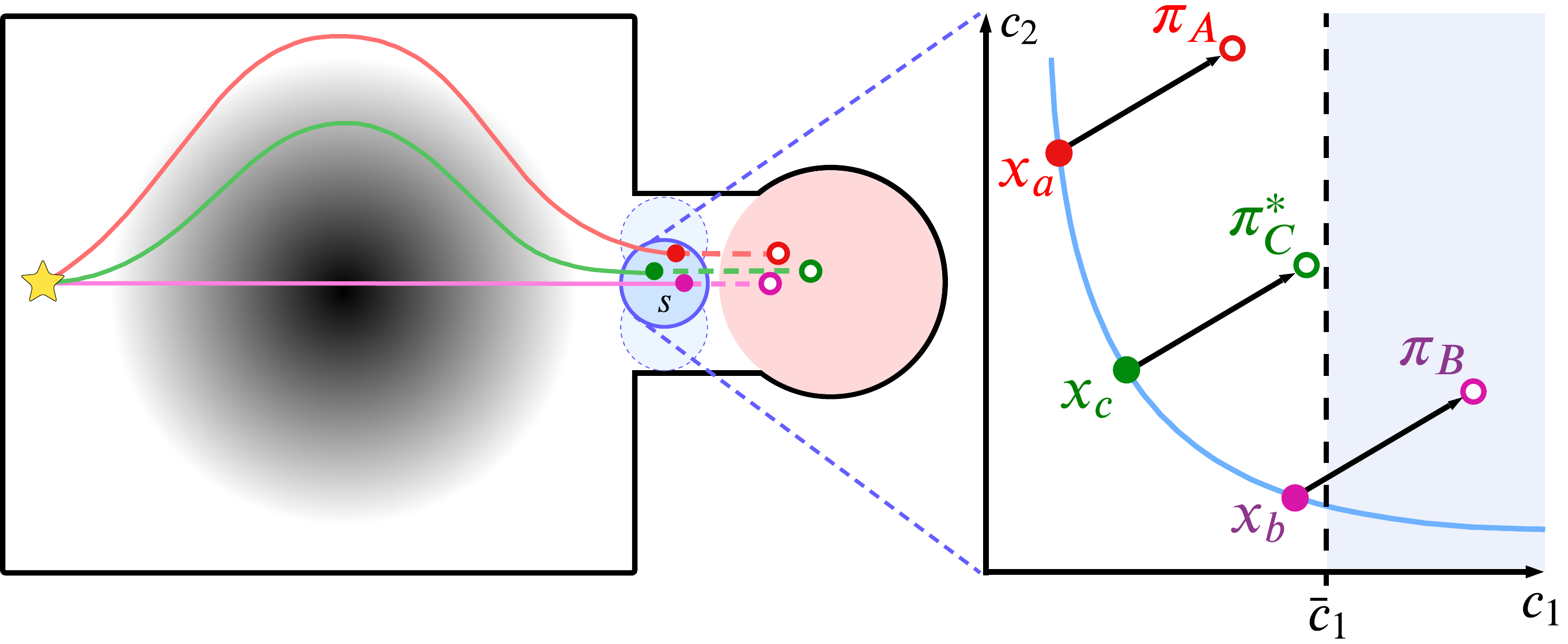} 
    \caption{Flytrap example showing the limitation of \textsc{SST} in minimizing \PL given a constraint on \CF.}
    \label{fig:conSSTcase}
\end{figure}

In contrast, \textsc{conSST} maintains all nodes that respect the cost constraints in each neighborhood. Therefore, $x_a, x_b,$ and $x_c$ are maintained in $\repCon$, allowing the algorithm to select and extend each node with positive probability until a near-optimal motion plan is found (i.e., trajectory $\pi^*_C$). In Section~\ref{sec:analysis}, we prove that \textsc{conSST} is asymptotically near-optimal.

\section{Pareto-Optimal Motion Planning}
\label{sec:poSST}

Lastly, we present Pareto-optimal \SST (\poSST), a kinodynamic sampling-based motion planner that approximates the entire Pareto front of solution trajectories with arbitrarily tight error bounds, thereby solving Problem~\ref{prob:paretoOptMP}. This algorithm closely mirrors the structure of \lexSST and \conSST but employs fewer pruning operations, thereby exploring a broader set of cost trade-offs.

% Unlike \lexSST, which restricts the active set in each neighborhood using the \textsc{PruneLexSet} routine, \poSST retains \emph{all non-dominated sub-trajectories} entering each witness neighborhood. This results in a local Pareto front per neighborhood, which improves monotonically over time as better solutions are discovered.

Specifically, \poSST retains Pareto-optimal candidates per neighborhood (every node when $\vec\varepsilon = 0$, otherwise a sparse subset). This can be interpreted as initializing \conSST with constraints $\bar c_i = \infty$ for all $i$, allowing the algorithm to explore the full extent of the Pareto front. Consequently, the representative set for \poSST is simply $\repPO := \ballNodes$, as no additional filters are applied. Upon termination, \poSST returns a set of trajectories $\Sols$ rather than a single solution.

% \ml{not sure if I understand notation $[\infty]$.}
% , or \conSST with $\mathbf{C} = \mathbf{\infty}$
% \ml{is $\mathbf{C}$ defined?}
% \ml{combine the two notations in one def?}
% \[
% S.P^\epsilon = \{x \in S.P \mid \nexists \, x' \in S.P \text{ s.t. } C(\overline{x_0 \rightarrow x'}) \prec C(\overline{x_0 \rightarrow x})\}.
% \]

Due to this broader search scope, \poSST converges more slowly than \lexSST and \conSST, especially in high-dimensional cost spaces. However, it guarantees that all trade-offs among objectives are represented in the output, making it suitable for applications where post-hoc preference elicitation is necessary.

% Algorithmically, all other components of \poSST remain unchanged from \lexSST. Node extension, witness association, and pruning via \textsc{PruneDominated} are executed identically. The difference lies solely in the replacement of $\lexSet$ with a locally maintained Pareto front $P$ at each witness, without applying any tolerance thresholds.

% We analyze the theoretical guarantees of $\lexSST$ and $\poSST$ in the next section. Empirical evaluations comparing their performance across a range of domains are provided in Section~\ref{sec:eval}.

\section{Analysis}
\label{sec:analysis}

Here, we analyze the theoretical properties of our proposed algorithms, focusing on probabilistic completeness and near-optimality guarantees. For clarity, we first prove these properties for \poSST, then show how the additional pruning procedures employed by \lexSST and \conSST preserve them within their respective problem classes.  All the proofs are provided in the appendix.

% \subsection{Overview}
% \label{subsec:overview}

We first extend the notions of probabilistic $\delta$-robust completeness (Def.~\ref{def:delta-robust-completeness}) and asymptotic $\delta$-robust near-optimality 
(Def.~\ref{def:asymp_delta_near_opt})
from~\cite{SST} to the multi-objective setting. When targeting a single solution, as in \lexSST and \conSST, Defs.~\ref{def:delta-robust-completeness}-\ref{def:asymp_delta_near_opt} apply with slight alterations. For \poSST, which must approximate the entire Pareto front, we introduce analogous definitions that generalize these requirements to hold simultaneously across all Pareto-optimal solutions. 

\begin{definition}[Probabilistic $\delta$-Robust Pareto-Completeness]
\label{def:pareto-completeness}
Let $\Sols$ denote the set of trajectories computed by algorithm \textsc{alg} at iteration $n$. Algorithm \textsc{alg} is probabilistically $\delta$-robustly Pareto-complete if, for every motion planning problem admitting a non-empty set of $\delta$-robust Pareto-optimal solutions $\poDSol$, the following holds for each independent run:
\begin{multline*}
    \liminf_{n \to \infty} \; \mathbb{P}\!\left(\forall\pi^*\in \poDSol, \; \exists \pi \in \Sols \right. \\
    \left. \text{ s.t. } \pi \text{ and } \pi^* \text{ are } \delta\text{-similar} \right) = 1.
\end{multline*}
\end{definition}

\noindent Given an algorithm with this property and the Lipschitz continuity of the cost functions (Assumption~\ref{assumption:CostFunction}), we can ensure that any Pareto-optimal trajectory in $\poDSol$ can be approximated, yielding the following notion of near-optimality.

\begin{definition}[Asymptotic $\delta$-Robust Near-Pareto-Optimality]
\label{def:near-pareto-optimality}
Consider a motion-planning problem that admits at least one $\delta$-robust motion plan. Let $\C^*_\delta = \{C(\pi^*) \mid \pi^* \in \poDSol\}$ denote the set of cost vectors of all $\delta$-robust Pareto-optimal trajectories, and let $\algRVs$ denote a random variable representing the set of cost vectors $\C = \{C(\pi) \mid \pi \in \Sols\}$ returned by algorithm \textsc{alg} after iteration $n$. Algorithm \textsc{alg} is \emph{asymptotically $\delta$-robustly near-Pareto-optimal} if for each independent run:
\begin{multline*}
    \mathbb{P}\big(\limsup_{n \to \infty} \; \forall C^* \in \C^*_\delta, \; \exists Y \in \algRVs \\
    \text{ s.t. } Y \preceq H(C^*, \witEps, \delta) \big) = 1,
\end{multline*}
where $H: \mathbb{R}^n_{\geq 0} \times \mathbb{R}^n_{\geq 0} \times \mathbb{R}_{\geq 0} \rightarrow \mathbb{R}^n_{\geq 0}$ bounds the sub-optimality of \textsc{alg} as a function of an optimum cost $C^*$ and the user-defined sparsity parameters. Note that $H(C^*, \witEps, \delta) \to C^*$ as $\delta \to 0$ and $\witEps \to 0$.
\end{definition}

% Subsection~\ref{subsec:analyzePoSST} proves that \poSST satisfies the definitions above. Subsection~\ref{subsec:analyzeConSST} establishes the completeness of \conSST under the stipulation that at least one $\delta$-robust solution satisfying the cost constraints exists, and its near-optimality with respect to $\conMin$. Finally, Subsection~\ref{subsection:analyzeLexSST} proves completeness and near-optimality of \lexSST with respect to $\lexMin$, restricted to the case of two cost objectives.

\subsection{Analysis of \poSST}
\label{subsec:analyzePoSST}

Here, we show that the pruning and selection procedures of $\poSST$ preserve the ability to generate a $\delta$-similar trajectory to any $\delta$-robust, Pareto-optimal solution $\pi^* \in \poDSol$, from which completeness and near-optimality follow.
We begin by recalling three results from~\cite{SST}. First, we relate the user-defined parameters $\dps$ and $\delta_s$ to the dynamic clearance $\delta$.

\begin{proposition}[\!{\cite[Prop.~13]{SST}}]
    \label{prop:deltas}
    The parameters $\dps$ and $\delta_s$ must satisfy the following relationship with $\delta$ for a $\delta$-robust feasible motion planning problem to be solved:
    $$\dps + 2 \delta_s < \delta.$$
\end{proposition}

Next, we characterize $\delta$-similarity by defining a \textit{covering ball sequence}, adapted from \cite{SST}, over each $\pi^* \in \poDSol$.

\begin{definition}[Covering Balls]
\label{def:coveringBalls}
    Given a trajectory $\pi^*$, a dynamic clearance $\delta > 0$, and a cost increment $\Delta c_1 \in \R_{>0}$, the covering ball sequence $\mathbb{B}(\pi^*, \delta, \Delta c_1)$ is a set of $M+1$ hyper-balls
    \[
    \{\mathcal{B}_\delta (x^*_0), \mathcal{B}_\delta (x^*_1), \dots, \mathcal{B}_\delta (x^*_M)\},
    \]
    where each $x^*_i$ lies along $\pi^*$ and satisfies $ \Delta c_1 =  c_1(\overline{x^*_j \rightarrow x^*_{j+1}})$ for all $j \in  \{0, 1, \dots, M-1\}$. 
\end{definition}

\noindent Without loss of generality, we partition a reference trajectory $\pi^*$ according to its primary cost $c_1$, resulting in $k = \frac{c_1(\pi^*)}{\Delta c_1}$ segments. 
A trajectory is $\delta$-similar to $\pi^*$ if and only if it remains entirely within this sequence. 

Lastly, because \textsc{MonteCarloProp} (Algorithm~\ref{alg:MCProp}) samples controls with full support over $\U$, it can produce a control arbitrarily close to that of $\pi^*$, ensuring a positive probability of transitioning from one ball to the next~\cite[Theorem 17]{SST}. 
% Since $\poSST$ is initialized at $x_0 = \pi^*(0) \in \mathcal{B}_\delta (x^*_0)$ for all $\pi^* \in \poDSol$, a $\delta$-similar trajectory to any reference can therefore be generated, provided nodes within the covering ball sequence have a positive probability of selection at each iteration. 
Since $\poSST$ is initialized at $x_0 = \pi^*(0) \in \mathcal{B}_\delta(x^*_0)$ for all $\pi^* \in \poDSol$ (i.e., all solutions begin at $x_0$), a $\delta$-similar trajectory to any Pareto-optimal reference can therefore be generated, provided that nodes within its covering ball sequence have a positive probability of selection at each iteration.
It thus remains to show that the subroutines \textsc{ParetoSelect} and \textsc{PruneDominated} preserve this selection property.
% , which we establish in the following lemmas.

\subsubsection{Properties of \textsc{ParetoSelect}}
We first analyze \textsc{ParetoSelect} (Algorithm~\ref{alg:ParetoSelect}), which biases node selection toward locally non-dominated solutions, thereby accelerating convergence relative to simple nearest-neighbor selection. 

At each iteration, a random state $x_{\text{rand}} \sim \mathcal{U}(\X)$ is sampled from a uniform distribution over $\X$. Then, all active nodes within a ball $\mathcal{B}_{\dps}(x_{\text{rand}})$ are retrieved. The subset of these nodes that are not dominated by any other node in the set (with respect to all cost objectives) forms a local Pareto front. Finally, a random node $x_{\text{selected}}$ is chosen uniformly from this set for propagation.

\begin{lemma}
\label{lem:ParetoSelect}
    Assume $x_{\text{rand}} \sim \mathcal{U}(\X)$. If there exists a node $x \in \mathcal{B}_{\dps}(x^*_i)$ at some iteration $n$, where $x^*_i$ corresponds to a state along a $\delta$-robust reference trajectory $\pi^*$, then the probability that \textsc{ParetoSelect} chooses a node $x' \in \mathcal{B}_\delta(x^*_i)$ for propagation is lower bounded by a positive constant $\gamma > 0$ for all iterations $n' > n$.
\end{lemma}
\noindent The proof for Lemma~\ref{lem:ParetoSelect} can be found in Appendix~\ref{proof:ParetoSelect}.

% While \textsc{ParetoSelect} ensures that locally non-dominated nodes are considered for propagation with nonzero probability, this mechanism alone does not guarantee efficient exploration of Pareto-optimal solutions. In general, the local Pareto front may be arbitrarily dense, and uniform sampling from this set may not sufficiently promote diversity among generated trajectories. Consequently, additional mechanisms are required to mitigate redundancy and enhance solution diversity, which are addressed through the pruning procedure discussed next.
While \textsc{ParetoSelect} ensures that locally non-dominated nodes are considered for propagation with positive probability, this alone does not guarantee efficient exploration of the Pareto-optimal set. The local Pareto front may be arbitrarily dense, and uniform sampling from it may not promote diversity among the generated trajectories. We address this redundancy by actively managing the search tree using the pruning procedure discussed next.

\subsubsection{Properties of \textsc{PruneDominated}}

The \textsc{PruneDominated} procedure (Algorithm~\ref{alg:PruneDominated}) is used by $\poSST$ to retain only non-dominated candidate nodes at every witness neighborhood. Further, this procedure involves a sparsity parameter $\witEps$ that ensures only one node is retained within a region of the objective space. Then, we bound the worst-case deviation from a reference trajectory $\pi^*$ that can result from applying \textsc{PruneDominated}.

\begin{lemma}
\label{lem:pruneDom}
Let $\delta_c = \dps - 2\delta_s$. If a state $x \in V$ is generated at iteration $n$ such that $x \in \mathcal{B}_{\delta_c}(x^*_i)$, then for every iteration $n' > n$, there exists a state $x' \in V$ such that $x' \in \mathcal{B}_{\delta - \dps}(x^*_i)$ and $C(x') \succeq C(x) + \witEps$.
\end{lemma}

\begin{lemma}
\label{lem:poSST}
Assuming uniform sampling of the state space in \textsc{ParetoSelect}, if there exists a node $x \in V$ such that $x \in \mathcal{B}_{\delta_c}(x^*_i)$ at iteration $n$, then the probability that \textsc{ParetoSelect} chooses a node $x' \in \mathcal{B}_\delta(x^*_i)$ is lower bounded by a positive constant $\gamma_{\text{PS}} > 0$ for all iterations $n' > n$.
\end{lemma}

Proofs are included in Appendices~\ref{proof:pruneDom} and~\ref{proof:lemmaPoSST}. From this result, the following is immediate.

\begin{theorem}[Completeness of \poSST]
\label{thm:poSST-complete}
    $\poSST$ is probabilistically $\delta$-robustly Pareto-complete.
\end{theorem}
% \ml{probabilistically $\delta$-robustly Pareto-complete is not defined.}

\subsubsection{Near-optimality of \poSST}

We now derive a formal bound on the worst-case performance of $\poSST$  in approximating any Pareto-optimal trajectory. Building on the result that $\poSST$  can generate a $\delta$-similar trajectory to any reference trajectory $\pi^* \in \poDSol$, we analyze how this deviation in the state space translates into sub-optimality in the multi-objective cost space.

% Let $\mathbf{K}_c = (K_{c_0}, K_{c_1}, \dots, K_{c_N}) \in \mathbb{R}^N_{\geq 0}$ be the vector of Lipschitz constants associated with the $N$ individual cost components. Intuitively, for any $\delta$-similar subtrajectories, the deviation in their cost vectors can be upper bounded $|C(\pi^*) - C(\pi)| \leq \delta \cdot \mathbf{K}_c$. We formalize this in the following theorem.

Let $\mathbf{K}_c = (K_{c,1}, \dots, K_{c,N}) \in \mathbb{R}^N_{\geq 0}$ denote the Lipschitz constants of the $N$ cost components. Intuitively, for any $\delta$-similar subtrajectories $\pi, \pi^*$, the cost deviation is bounded component-wise by $|c_i(\pi^*) - c_i(\pi)| \leq \delta \cdot K_{c,i}
$ for all $i$. We formalize this in the following theorem.

\begin{theorem}[Asymptotic near-optimality of \poSST]
\label{thm:poSST-optimal}
    $\poSST$  is asymptotically $\delta$-robustly near-Pareto-optimal.
\end{theorem}

The proof, included in Appendix~\ref{proof:poSST-optimal}, also provides an upper bound on the near-optimality of $\poSST$, reproduced below:
\begin{equation}
    \label{eq:bounds}
    c_i(\pi) \leq \left(1 + \frac{\delta \cdot K_{c, i} + \varepsilon_i}{\Delta c_i} \right) \cdot c_i(\pi^*) \quad \forall i.
\end{equation}
\noindent where $\Delta c_i = \min_{j \in \{0, \dots, M-1\}} c_i(\overline{x^*_j \rightarrow x^*_{j+1}})$.

Figure~\ref{fig:2DUpperBound} illustrates this upper bound in a 2D objective space. Finally, since the probability of generating a $\delta$-similar trajectory segment is strictly positive, $\poSST$ almost surely generates a $\delta$-similar motion plan to any point along the true Pareto front. Thus, in the worst case, it returns a solution within this upper bound as the approximate Pareto front, i.e., for every $\pi^* \in \poDSol$, $\exists \pi \in \Sols$ s.t. \eqref{eq:bounds} holds.

\begin{figure}[tbp]
    \centering
    \includegraphics[width=0.9\columnwidth]{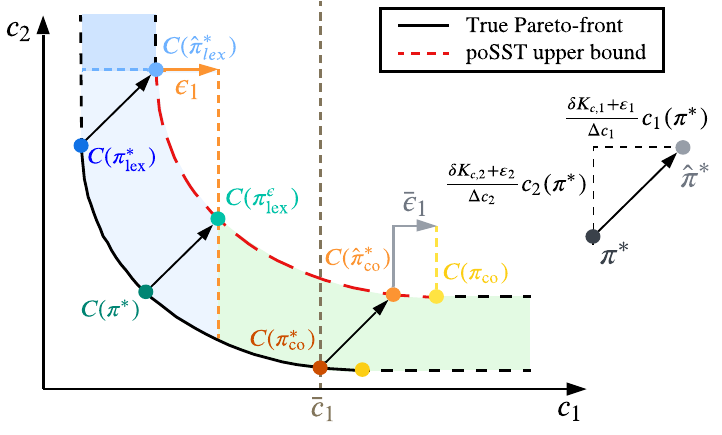} 
    \caption{Upper bound on approximation error of the Pareto-front returned by $\poSST$  as $n \rightarrow \infty$. Worst-case, near-optimal solutions of \lexSST ($\lexApprox$, green) and \conSST ($\conApprox$, yellow) shown relative to true minimums $\lexMin$ and $\conMin$, respectively.
    % The admissible solution region of \lexSST, given by \eqref{eq:lex-optimality}, is shown in blue in Figure~\ref{fig:2DUpperBound}, with the corresponding worst-case cost shown in light purple.
    }
    \label{fig:2DUpperBound}
\end{figure}

\subsection{Analysis of $\lexSST$}
\label{subsection:analyzeLexSST}

We now analyze the impact of \textsc{PruneLexSet} on the \poSST framework, arguing that its integration preserves the ability to generate a $\delta$-similar trajectory to the lexicographical minimum $\lexMin \in \poDSol$ for bi-objective problems. Recall that rather than exploring the entire Pareto front, \lexSST restricts the search to nodes within an $\epsilon$-tolerance of the current best-known primary cost, effectively considering a slice of width $\epsilon$ near the lexicographical optimum (shown in blue in Fig.~\ref{fig:2DUpperBound}). 
% Figure~\ref{fig:2d_lex} illustrates this region, defined by $c_1(x) \leq c_1(\bar{x}^*) + \epsilon_1$.

We first show that \textsc{PruneLexSet} does not remove nodes along a $\delta$-similar trajectory to $\lexMin$.

\begin{lemma}
\label{lemma:lexSST}
Let $\epsilon \geq \frac{\delta \cdot K_{c_1} + \varepsilon_1}{\Delta c_1} \cdot c_1(\lexMin)$, and let $\repLex$ be the representative set of a witness
% \ml{just one witness?}
centered at $s \in \mathcal{B}_{\delta_c + \delta_s}(\lexBall)$. If a state $x \in \repLex$ is generated at iteration $n$ such that $x \in \mathcal{B}_{\delta_c}(\lexBall)$ and
\[
c_1(\overline{x_0 \rightarrow x}) \leq \frac{\delta \cdot K_{c_1} + \varepsilon_1}{\Delta c_1} \cdot c_1(\lexMin),
\]
then for every iteration $n' > n$, there exists a state $x' \in \repLex$ such that $x' \in \mathcal{B}_{\delta - \dps}(\lexBall)$ and $C(x') \succeq C(x)$.
\end{lemma}

\begin{theorem}[Completeness of \lexSST]
\label{thrm:lexComplete}
\lexSST is probabilistically $\delta$-robustly complete for Problem~\ref{prob:lexicoMP} with bi-objective requirements.
\end{theorem}

\begin{theorem}[Asymptotic near-optimality of \lexSST]
\label{thrm:lexOptimal}
\lexSST is asymptotically $\delta$-robustly near-optimal for Problem~\ref{prob:lexicoMP} with bi-objective requirements.
\end{theorem}

\noindent Proofs are included in Appendices~\ref{proof:lemmaLexSST}, \ref{proof:lexComplete}, and \ref{proof:lexOptimal}, respectively.

In practice, $c_1(\lexMin)$ is unknown a priori, so a sufficiently large, fixed $\epsilon$ cannot be set to satisfy the condition in Lemma~\ref{lemma:lexSST}. However, a surrogate $\tilde\epsilon$ can be computed dynamically from the current best solution:
\[
\tilde\epsilon = \frac{\delta \cdot K_{c_1} + \varepsilon_1}{\Delta c_1} \cdot c_1(\hat\pi^*), \quad \text{where } \hat\pi^* = \arg\min_{\pi \in \algSol} c_1(\pi).
\]
Since $c_1(\hat\pi^*) \geq c_1(\lexMin)$, we maintain a conservative tolerance satisfying Lemma~\ref{lemma:lexSST}. As the algorithm progresses, $c_1(\hat\pi^*)$ decreases monotonically, $\tilde\epsilon$ shrinks accordingly, and
\[
\lim_{n \to \infty} c_1(\hat\pi^*) \to c_1(\lexMin), \qquad \lim_{n \to \infty} \tilde\epsilon \to \frac{\delta \cdot K_{c_1} + \varepsilon_1}{\Delta c_1} \cdot c_1(\lexMin).
\]
Since $\delta$, $\delta_s$, and $\vec\varepsilon$ are user-defined, the sub-optimality bound can be driven to zero, and by Eq.~\eqref{eq:lex-optimality} in the proof:
\[
C(\lexApprox) \to C(\lexMin) \quad \text{as} \quad \dps, \delta_s \to 0 \text{ and } \vec\varepsilon \to 0.
\]

An asymptotically optimal variant $\lexSST^*$ can be obtained by applying a shrinking schedule to $\dps$, $\delta_s$, and $\vec\varepsilon$, analogous to that of $\textsc{SST}^*$~\cite{SST}, thereby progressively densifying the search tree in both the state and objective spaces.

% In Fig.~\ref{fig:2DUpperBound}, the admissible solution region of \lexSST (given by \eqref{eq:lex-optimality}) is shown in blue.
% , with the corresponding worst-case cost shown in light purple.

% \ml{either remove this discussion, or modify it by saying that designing an alg. for $N>2$ remains and open research question.}
% We conclude with a remark on the fundamental difficulty of lexicographic minimization in continuous cost spaces. As discussed in Section~\ref{subsec:e-eqivalence}, the $\epsilon$-equivalence approach does not extend to $N > 2$ objectives: even if the true Pareto front is approximated with arbitrary accuracy, there is no guarantee that a node representing the lexicographic minimum is retained throughout the filtering procedure.

% \begin{remark}
% \ml{remove remark and discuss in the conclusion section}
% Let $\C = \{C_1, C_2, \dots, C_N\}$ be a set of vectors $C \in \mathbb{R}^N$, and let $\hat{L} \in \C$ deviate from the true lexicographic minimum $L$ by a known bound $\vec\epsilon \in (0, \epsilon_i]^N$ (i.e., $\hat{L} = L + \vec\epsilon$). Then there \emph{does not exist} an algorithm that can identify $\hat{L}$ with bounded error
% \ml{this is a huge claim.  Can you prove it?}\yr{proof: I thought hard about it and couldn't find a way.}
% .
% \end{remark}

\subsection{Analysis of $\conSST$}
\label{subsec:analyzeConSST}

The analysis of \conSST mirrors that of \lexSST, with the key difference that \textsc{PruneConSet} enforces a fixed constraint window $\bar{c}_i + \bar\epsilon_i$ rather than the shifting $\epsilon$-window of \textsc{PruneLexSet}. We show this substitution preserves the ability to generate a $\delta$-similar trajectory to $\conMin$, the motion plan minimizing $c_N$ subject to all $N-1$ constraints. For compactness, we write $i \neq N$ to mean $i \in \{1, \dots, N-1\}$.

\begin{lemma}
\label{lemma:conSST}
Let $\bar\epsilon_i \geq \frac{\delta \cdot K_{c,i} + \varepsilon_i}{\Delta c_i} \cdot c_i(\conMin)$ for all $i \neq N$, and let $\repCon$ be the representative set of a witness 
% \ml{one witness?}
centered at $s \in \mathcal{B}_{\delta_c + \delta_s}(\conBall)$. If a state $x \in \repCon$ is generated at iteration $n$ such that $x \in \mathcal{B}_{\delta_c}(\conBall)$ and
\[
c_i(\overline{x_0 \rightarrow x}) \leq \frac{\delta \cdot K_{c,i} + \varepsilon_i}{\Delta c_i} \cdot c_i(\conMin) \quad \forall i \neq N,
\]
then for every iteration $n' > n$, there exists a state $x' \in \repCon$ such that $x' \in \mathcal{B}_{\delta - \dps}(\conBall)$ and $C(x') \succeq C(x)$.
\end{lemma}

\begin{theorem}[Completeness of \conSST]
\label{thrm:conComplete}
\conSST is probabilistically $\delta$-robustly complete for Problem~\ref{prob:conOptMP}.
\end{theorem}

\begin{theorem}[Asymptotic near-optimality of \conSST]
\label{thrm:conOptimal}
\conSST is asymptotically $\delta$-robustly near-optimal for Problem~\ref{prob:conOptMP}.
\end{theorem}

\noindent Proofs are included in Appendices~\ref{proof:conSST}, \ref{proof:conComplete}, and \ref{proof:conOptimal}, respectively.

Since $c_i(\conMin)$ is unknown a priori, $\bar\epsilon_i$ can be conservatively set using the user-defined bounds $\bar{c}_i$:
\[
\tilde\epsilon_i = \frac{\delta \cdot K_{c,i} + \varepsilon_i}{\Delta c_i} \cdot \bar{c}_i \quad \forall i \neq N,
\]
which satisfies Lemma~\ref{lemma:conSST} since $c_i(\conMin) \leq \bar{c}_i$ by the problem statement. As with \lexSST, the sub-optimality bound decreases monotonically as $\dps, \delta_s \to 0$ and $\vec\varepsilon \to 0$, and an asymptotically optimal variant, $\conSST^*$, follows by applying a shrinking schedule analogous to that of $\textsc{SST}^*$~\cite{SST}.

\section{Evaluations}
\label{sec:eval}

We present a series of case studies designed to validate the theoretical contributions of this work and demonstrate the practical advantages of our proposed algorithms. We first introduce the experimental setup, then present our evaluations of \lexSST, \conSST, and \poSST in the subsequent subsections. 
\subsection{Overview}
\subsubsection*{Workspaces}
Experiments are conducted across four workspaces shown in Fig.~\ref{fig:Workspaces}: a simple environment (\ws), a two-homotopy-class environment (\wss), a three-homotopy-class environment (\wsss), and a cluttered environment with many homotopic solution classes (\wssss). 
% The workspaces are shown in Fig.~\ref{fig:Workspaces}.

\subsubsection*{Dynamics Models}

We consider two dynamic models: (i) a 2D double integrator (\DI) with state $x = [p_x, p_y, v_x, v_y]^\top$ comprising position and velocity, and control input $u = [a_x, a_y]^\top$ representing acceleration, and 
% The discrete-time dynamics are given by $x_{k+1} = Ax_k + Bu_k$, where
% %
% \[
% A = \begin{bmatrix} 1 & 0 & 1 & 0 \\ 0 & 1 & 0 & 1 \\ 0 & 0 & 1 & 0 \\ 0 & 0 & 0 & 1 \end{bmatrix}, \quad
% B = \begin{bmatrix} 0 & 0 \\ 0 & 0 \\ 1 & 0 \\ 0 & 1 \end{bmatrix},
% \]
% %
% with state and control limits that vary by workspace.
% 
% The second model is 
(ii) a 4D bicycle model (\BI) 
% representing a car-like robot 
with 
% state $x = [p_x, p_y, \theta, \lambda]^\top$, 
dynamics:
\[
\dot{p}_x = u_v \cos\theta, \quad \dot{p}_y = u_v\sin\theta, \quad \dot\theta = \frac{u_v}{L}\tan\lambda, \quad \dot\lambda = u_{\dot\lambda},
\]
where $(p_x, p_y)$ is the rear-axle position, $\theta$ is the heading angle, and $\lambda$ is the steering angle. The control input 
% $u = [u_v, u_{\dot\lambda}]^\top$ 
consists of the longitudinal velocity $u_v$ and the steering rate $u_{\dot\lambda}$.
% The continuous-time dynamics are given by:
% 
% where $L$ is the wheelbase length. Forward propagation is performed using the Runge-Kutta (RK4) integration method.

\subsubsection*{Objective functions}
% We define several objective functions to evaluate the quality of solution trajectories. 
% The cost functions corresponding to each objective are as follows:
%
We consider three cost functions:
\begin{align*}
\PL: \quad c(\pi) &= \int_\pi dy, \\
\CF: \quad c(\pi) &= \int_\pi \mathcal{N}(y) dy, \\
\CL: \quad c(\pi) &= 100 - \min_{x \in \pi} \mathcal{\check{D}}(x, \X_O),
\end{align*}
where $\mathcal{\check{D}}(x, \X)$ is the shortest distance from state $x$ to the set $\X_O$. Each experiment is given two of the above objectives, with the goal of minimizing the associated cost functions.

For example, a subset of Pareto-optimal trajectories for objectives \PL and \CL is shown in Figs.~\ref{fig:WS1} and \ref{fig:WS2} for workspaces \ws and \wss, respectively. Due to the simplicity of these environments and the interdependence of the objectives, the Pareto front can be computed analytically. 
% \ml{shown in Figs. ???}\yr{added in  parenthesis above.}
% \ml{what are the system dynamics considered?} \yr{added below, maybe too detailed.}

\begin{figure}[htbp]
    \centering
    \begin{subfigure}{0.24\textwidth}
        \centering
        \includegraphics[width=\textwidth]{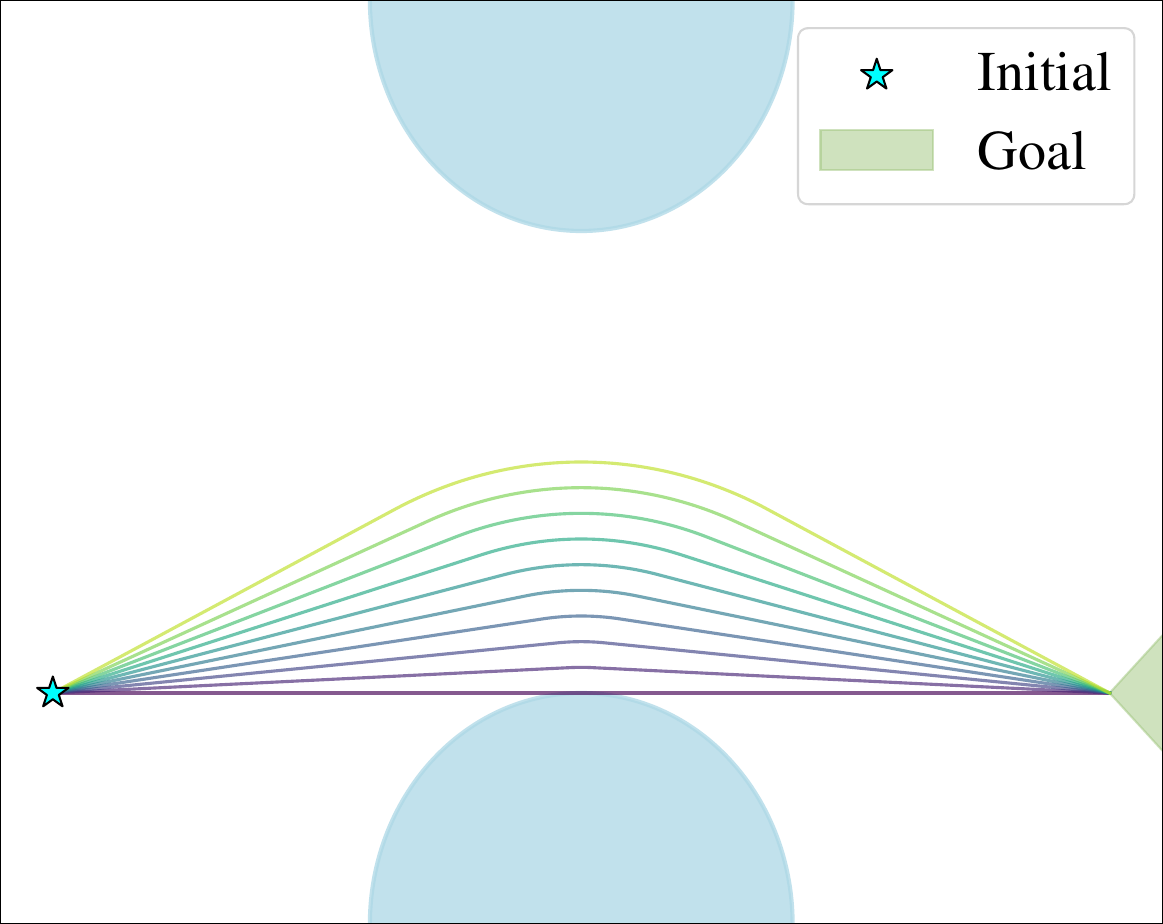}
        \caption{\ws.} 
        \label{fig:WS1}
    \end{subfigure}
    \begin{subfigure}{0.24\textwidth}
        \centering
        \includegraphics[width=\textwidth]{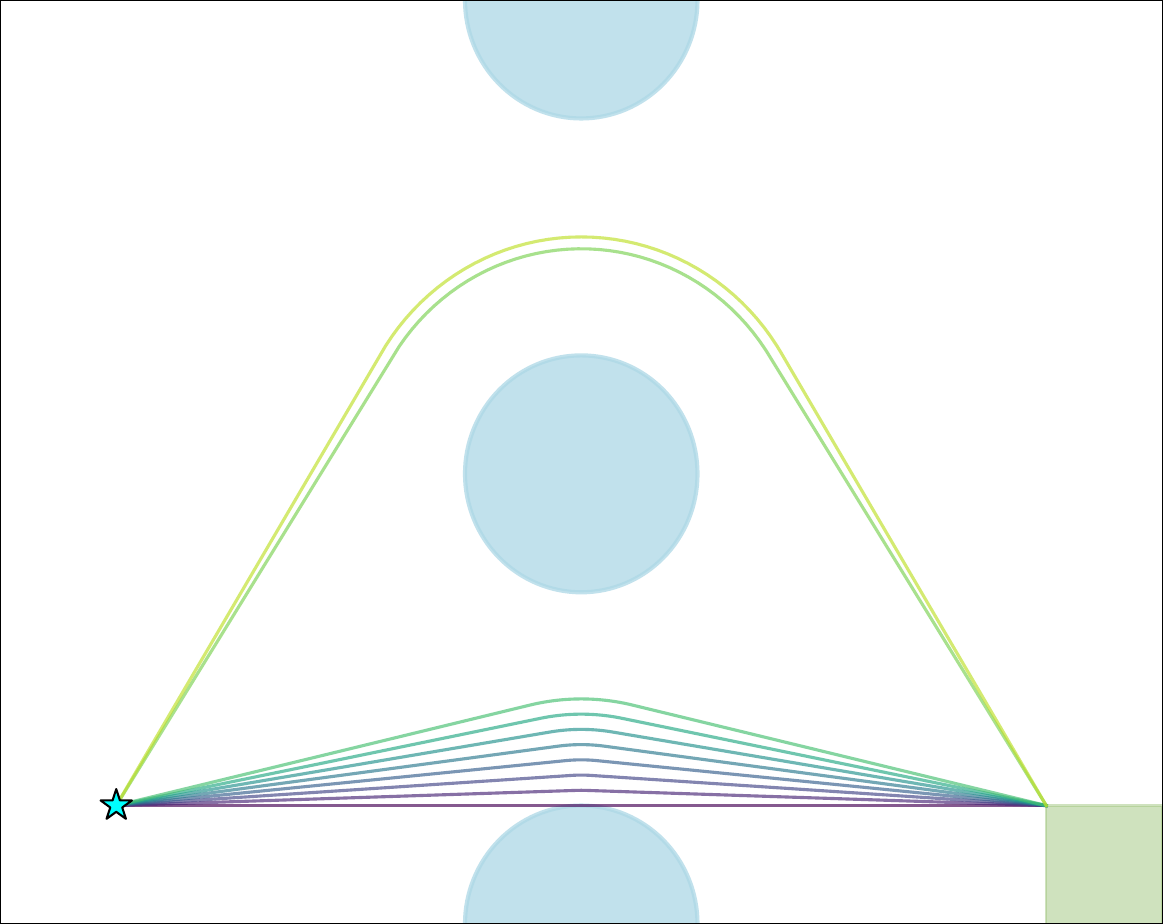}
        \caption{\wss.} 
        \label{fig:WS2}
    \end{subfigure} 
    \begin{subfigure}{0.24\textwidth}
        \centering
        \includegraphics[width=\textwidth]{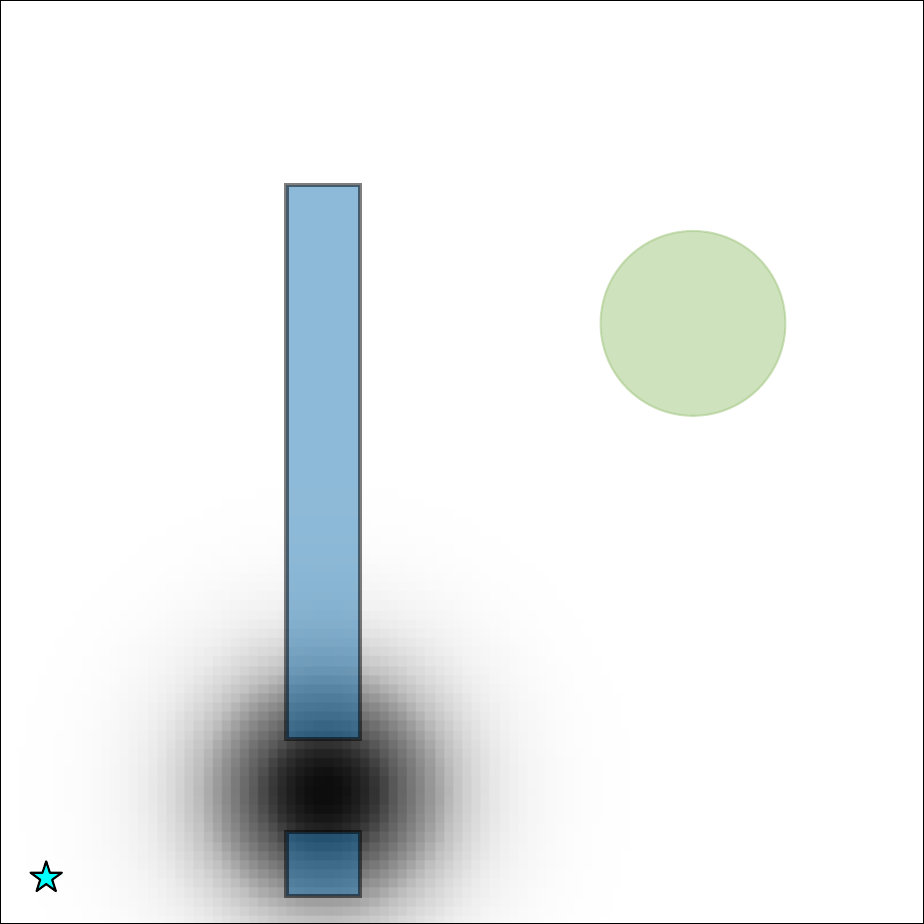}
        \caption{\wsss.} 
        \label{fig:WS3}
    \end{subfigure} 
    \begin{subfigure}{0.24\textwidth}
        \centering
        \includegraphics[width=\textwidth]{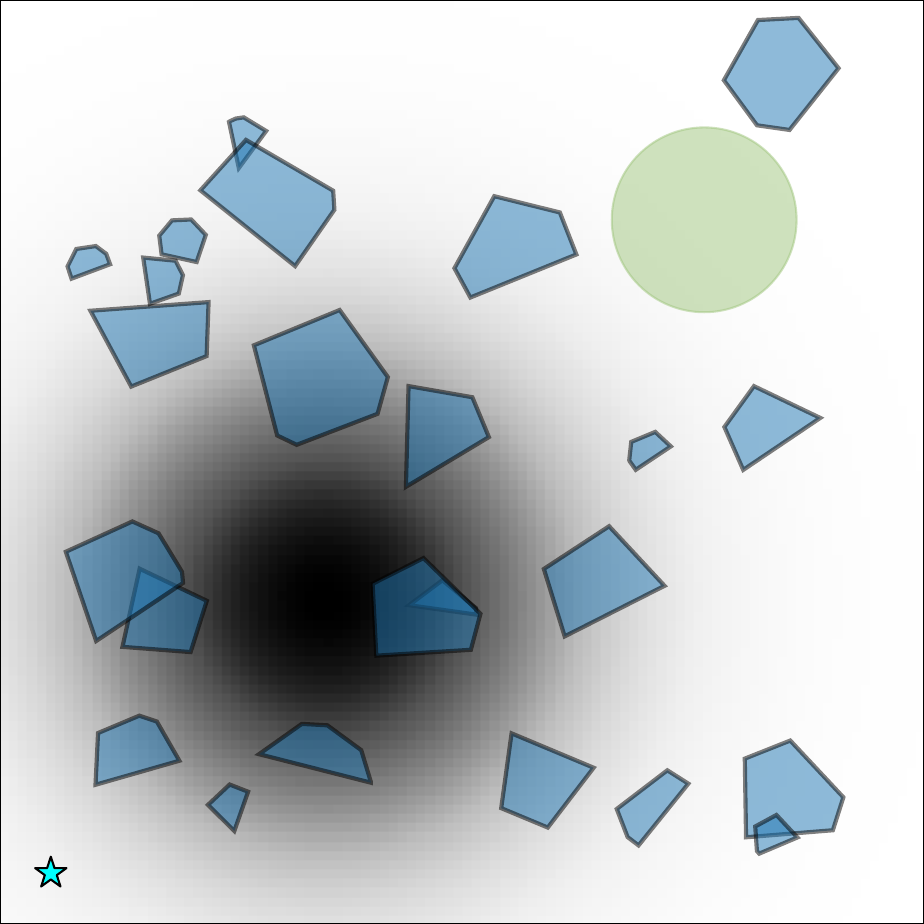}
        \caption{\wssss.} 
        \label{fig:WS4}
    \end{subfigure}    
    \caption{Considered workspaces. 
    % used for the evaluation of \lexSST, \conSST, and \poSST.
    Shades of gray in (c) and (d) correspond to the values of \CF.
    }
    \label{fig:Workspaces}
\end{figure}

\subsubsection*{Setup}
Results are reported across 100 independent runs and compared against weighted-sum SST (\wSST) as the primary baseline. Each case study is designed to isolate a specific aspect of the proposed approach, progressing from single-solution problems to full Pareto front approximation.

% \wSST is obtained by replacing the single cost of \SST with a weighted sum of the cost objectives. \lexSST, \conSST, and \poSST each extend the base \SST structure by replacing the single representative per witness with their respective representative set definitions. 
All algorithms were implemented in C++ as extensions to the \SST implementation provided by the Open Motion Planning Library (OMPL)~\cite{OMPL}. 
No parallelization was employed; all simulations were run sequentially on a machine equipped with an AMD Ryzen\texttrademark\ 5 8645HS processor (4.3\,GHz base clock) and 16\,GB of RAM. The implementation will be made publicly available upon the acceptance of the article.
% \ml{make sure, in every case study, the dynamics and planning times are clearly stated.}
\subsection{Lexicographic Minimization}

\subsubsection*{Case Study 1 -- \lexSST vs \SST}
\label{subsec:cs1}

\begin{figure*}[t]
    \centering
    \begin{subfigure}{0.21\textwidth}
        \centering
        \includegraphics[width=\textwidth]{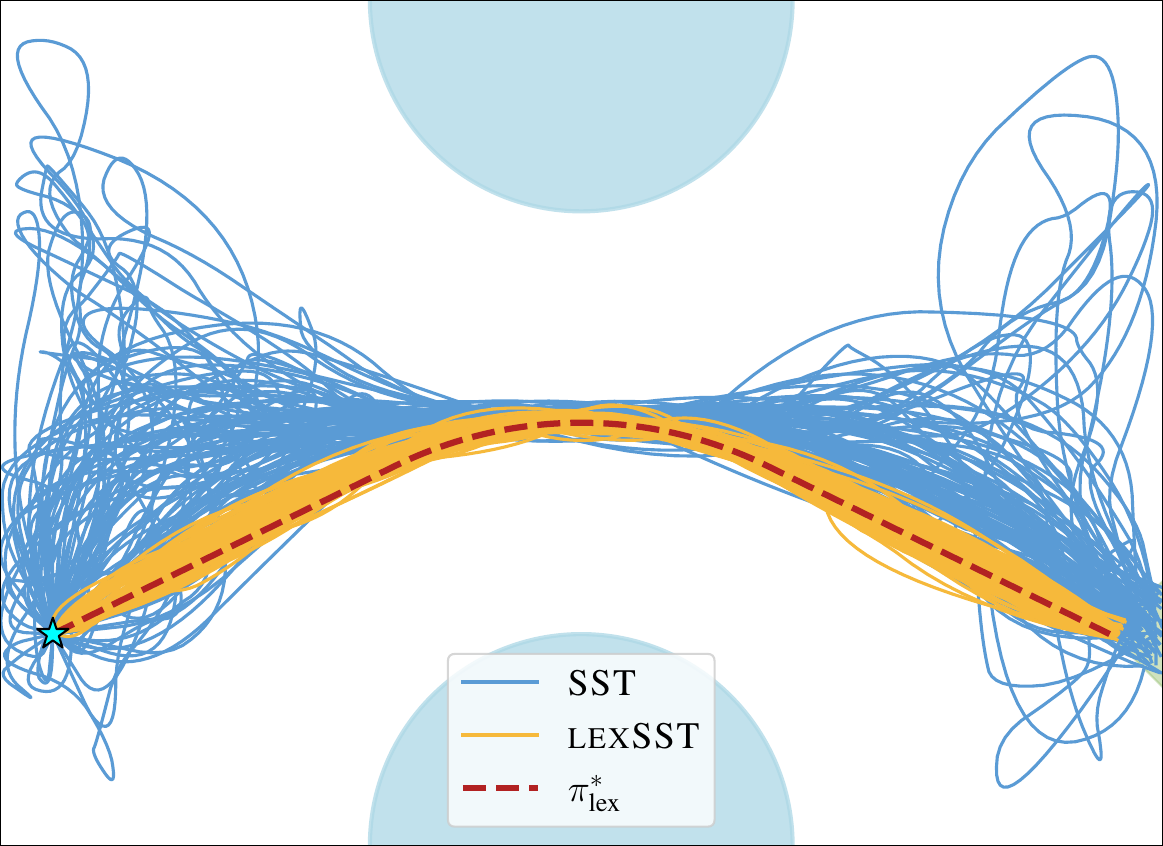}
        \vspace{0mm}
        \caption{Solution trajs. in \ws.}
        \label{fig:WS1lex traj}
    \end{subfigure}
    \begin{subfigure}{0.25\textwidth}
        \centering
        \includegraphics[width=\textwidth]{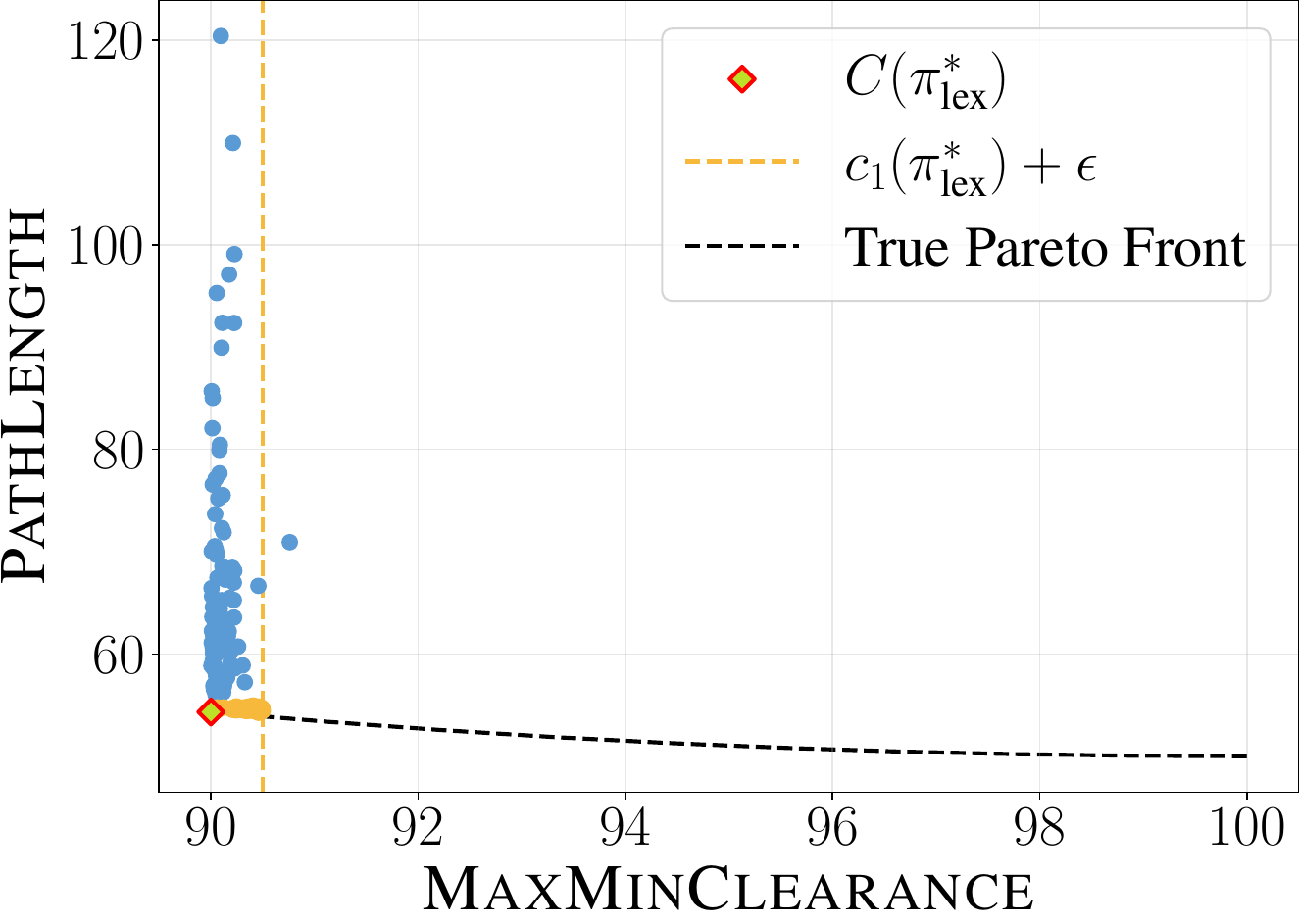}
        \caption{Solutions in objective space.}
        \label{fig:WS1lex costs}
    \end{subfigure} 
    \begin{subfigure}{0.25\textwidth}
        \centering
        \includegraphics[width=\textwidth]{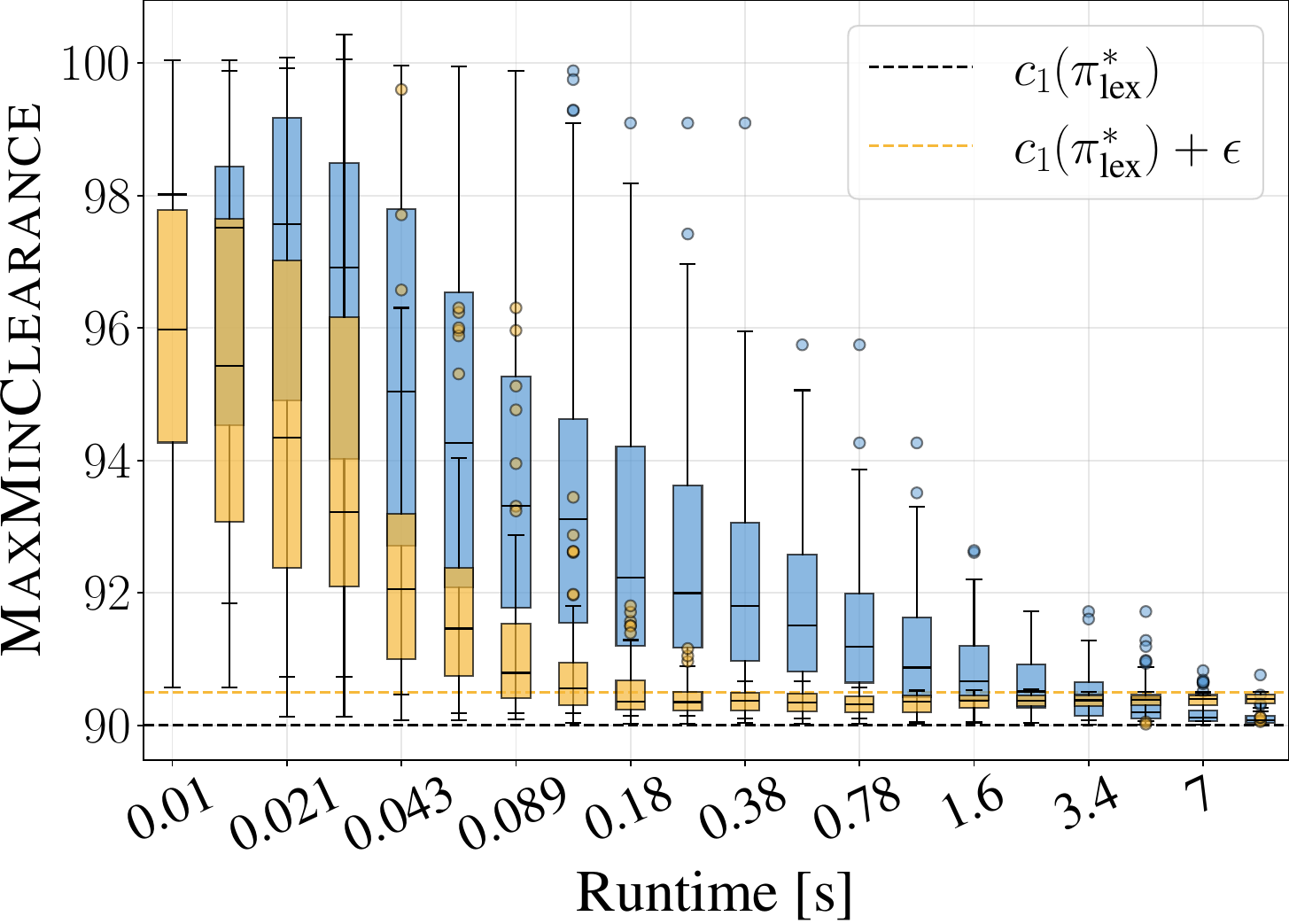}
        \caption{Time evolution of $c_1(\algSol)$.}
        \label{fig:WS1lex boxplot evol c1}
    \end{subfigure} 
    \begin{subfigure}{0.25\textwidth}
        \centering
        \includegraphics[width=\textwidth]{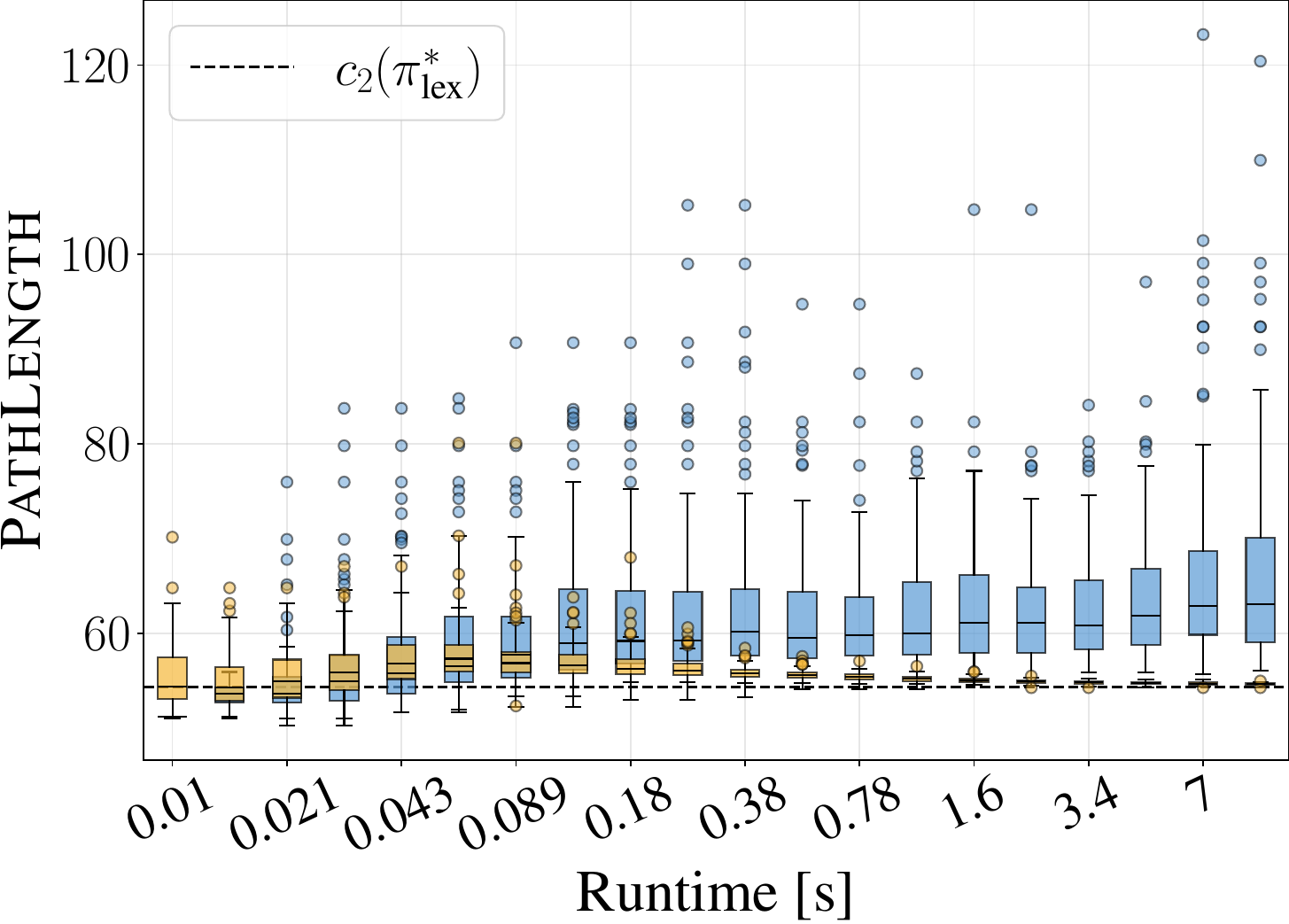}
        \caption{Time evolution of $c_2(\algSol)$.}
        \label{fig:WS1lex boxplot evol c2}
    \end{subfigure}    
    \caption{Case Study 1 (\lexSST vs. \SST). Lexicographic objectives: \CL (1st) and \PL (2nd).}
    \label{fig:WS1lex}
\end{figure*}
\begin{figure*}[t]
  \centering
  \begin{subfigure}{0.24\textwidth}
    \centering
    \includegraphics[width=\textwidth]{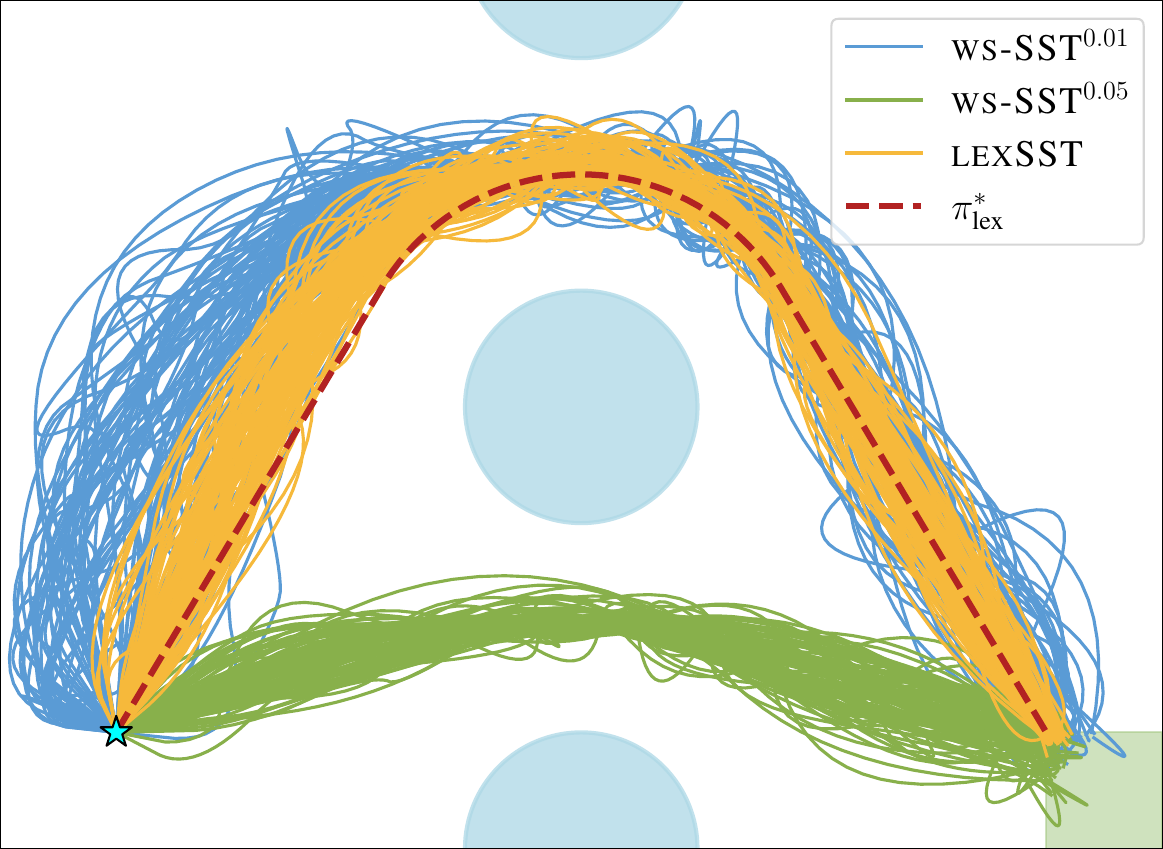}
    \caption{Solution trajectories in \wss.}
    \label{fig:WS2lex traj}
  \end{subfigure}
  \begin{subfigure}{0.24\textwidth}
    \centering
    \includegraphics[width=\textwidth]{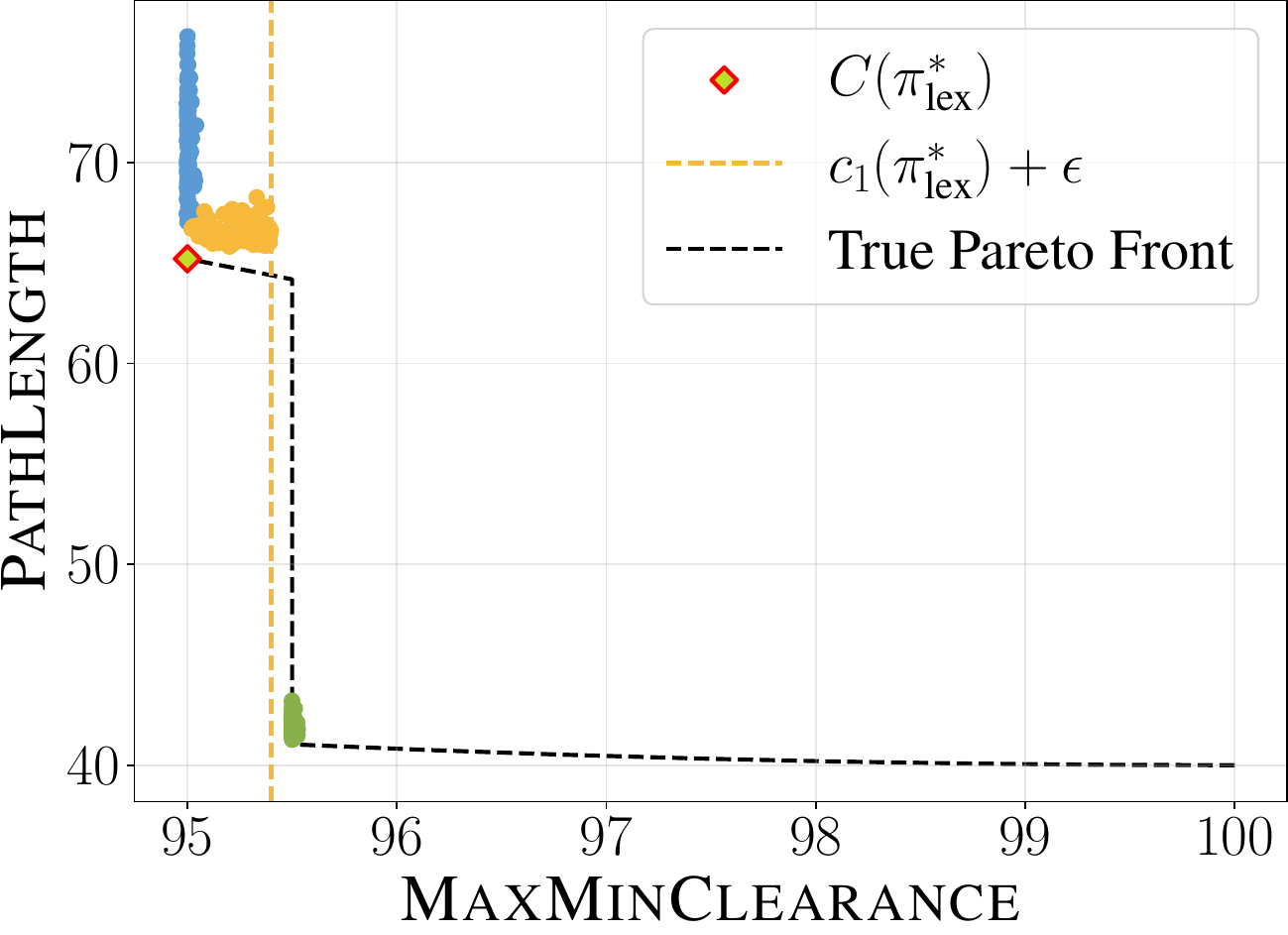}
    \caption{Solutions in objective space.}
    \label{fig:WS2lex cost}
  \end{subfigure}
  \begin{subfigure}{0.24\textwidth}
    \centering
    \includegraphics[width=\textwidth]{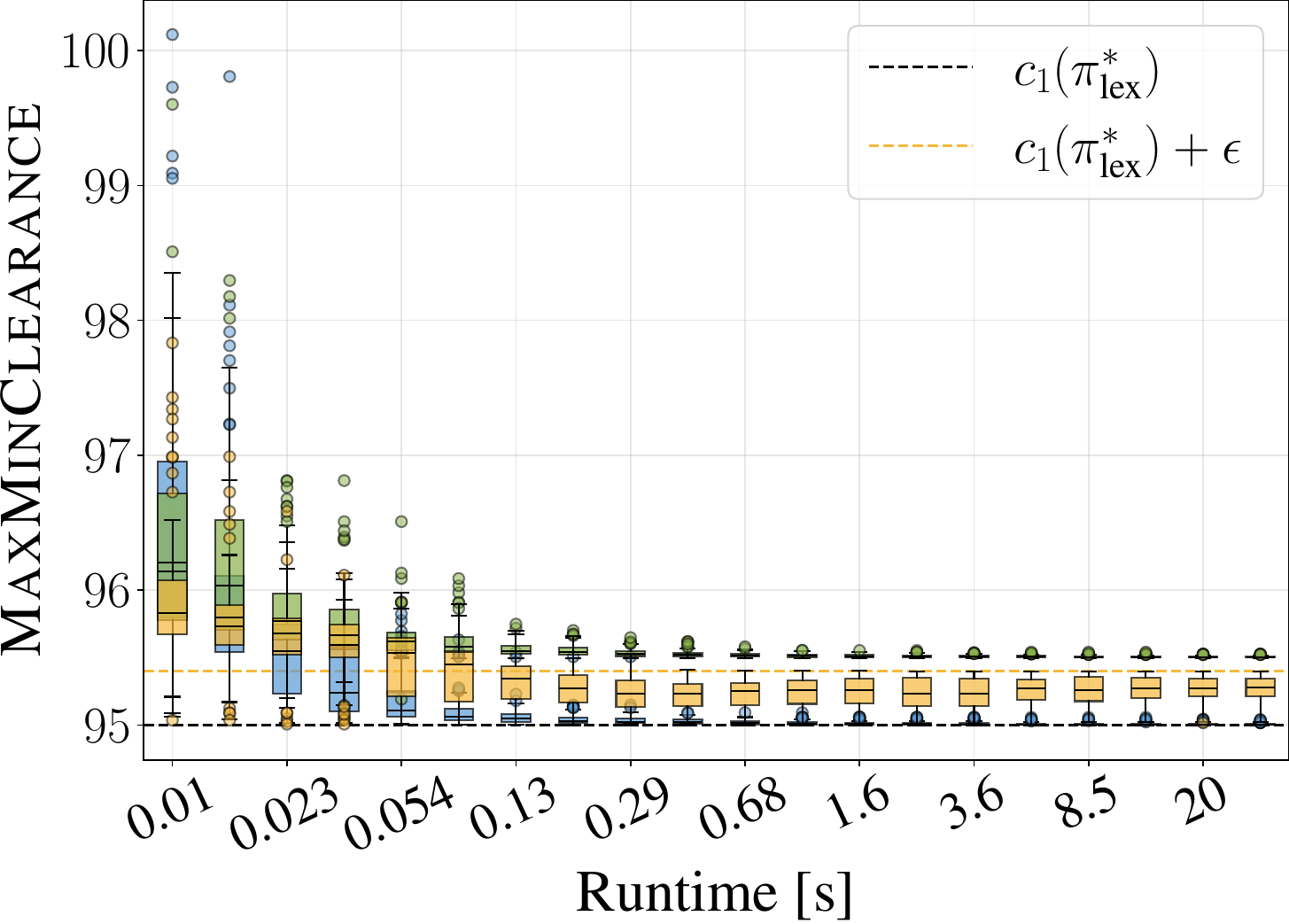}
    \caption{Time evolution of $c_1(\algSol)$.}
    \label{fig:WS2lex evol c1}
  \end{subfigure}
  \begin{subfigure}{0.24\textwidth}
    \centering
    \includegraphics[width=\textwidth]{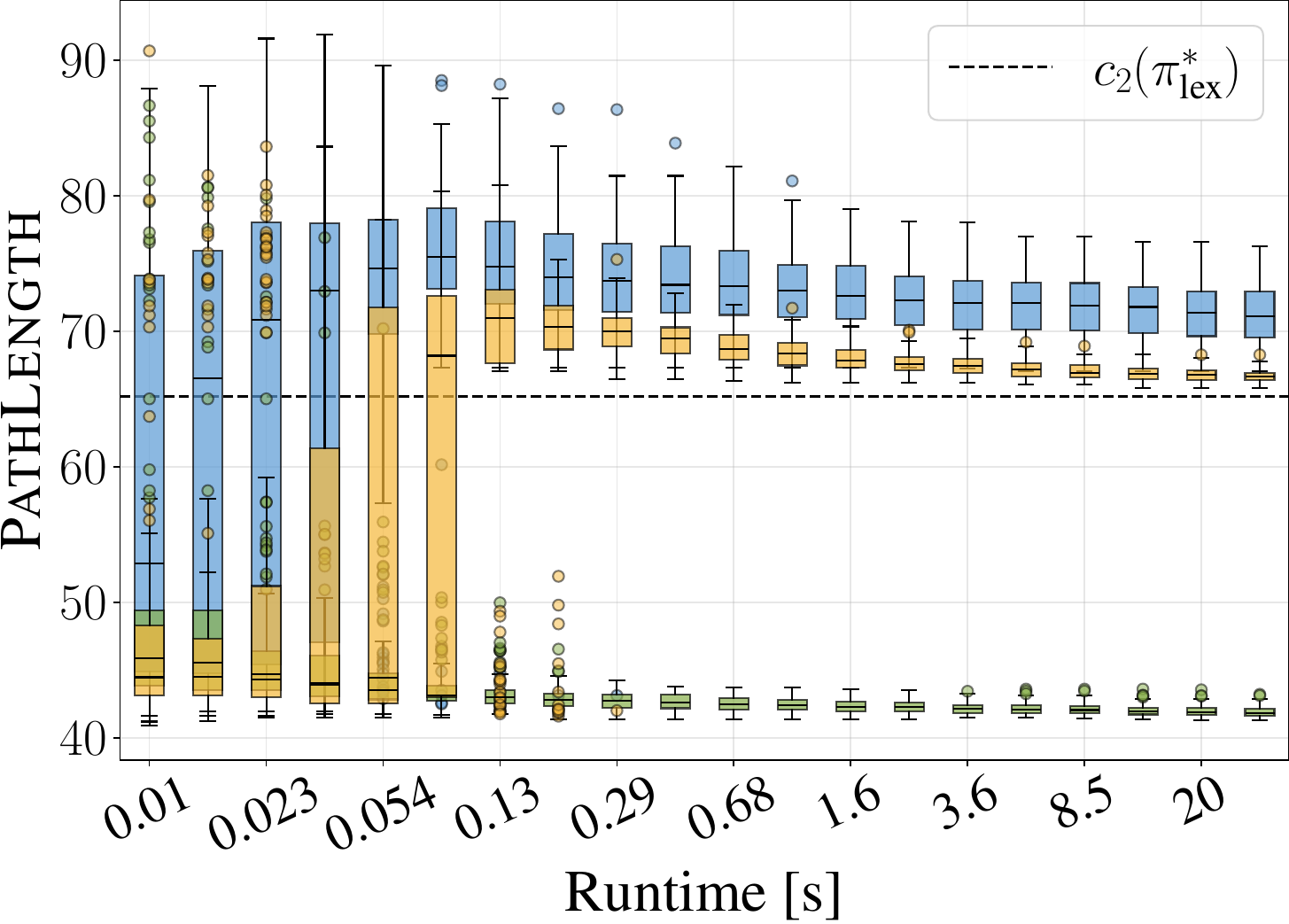}
    \caption{Time evolution of $c_2(\algSol)$.}
    \label{fig:WS2lex evol c2}
  \end{subfigure}
  \caption{Case Study 2 (\lexSST vs. \wSST). Lexicographic objectives: \CL (1st) and \PL (2nd).}
  \label{fig:WS2lex}
\end{figure*}

We begin by validating $\lexSST$'s ability to approximate the true lexicographic minimum, as established in Theorems~\ref{thrm:lexComplete} and~\ref{thrm:lexOptimal}. The experiment considers workspace \ws with the \DI model, where the optimization objective is lexicographic: minimize \CL first, then \PL. A 10-second planning budget was given per instance. Note that the primary optimum $c_1(\lexMin)$ is achieved by infinitely many trajectories, since any path traversing the widest passage without approaching an obstacle attains the same optimal clearance.

The results are shown in Fig.~\ref{fig:WS1lex}. As seen in Fig.~\ref{fig:WS1lex traj}, \SST (blue) reliably finds trajectories that maximize clearance but produces a wide spread of path lengths, since it lacks a mechanism to refine secondary objectives. In contrast, $\lexSST$ (orange) correctly prioritizes \CL and simultaneously refines \PL, yielding a tight cluster of solutions near the true lexicographic minimum (red dashed line). This behavior is further evident in the objective space (Fig.~\ref{fig:WS1lex costs}), where \SST solutions are broadly dispersed along the \PL axis, while $\lexSST$ solutions concentrate around the lexicographic minimum (red diamond).

The time series in Figs.~\ref{fig:WS1lex boxplot evol c1}--\ref{fig:WS1lex boxplot evol c2} shows the convergence behavior of both planners. Both \SST\ and $\lexSST$ converge to $c_1(\lexMin)$ at comparable rates, and $\lexSST$'s final solutions remain within the user-defined tolerance $\epsilon$, consistent with Eq.~\eqref{eq:lex-optimality}. However, while \SST exhibits a large spread in \PL, $\lexSST$ steadily converges toward $c_2(\lexMin)$, demonstrating its ability to refine the secondary objective while respecting the tolerance on the primary one.

\subsubsection*{Case Study 2 -- Robustness of $\lexSST$ to Weight Selection}
\label{subsec:cs2}
% \ml{add results for another dynamical system}

Case Study 1 shows that \SST\ produces wide variance in $c_2$ when optimizing $c_1$ alone. A natural remedy is to add a small weight to $c_2$, giving preference to solutions that also improve upon \PL. Here, we demonstrate the practical failure of this approach in a setting where weight selection is consequential. Workspace \wss\ contains two homotopy classes: an upper corridor with slightly greater clearance and a lower corridor with shorter path length. Using the same dynamics model and objectives as Case Study 1, we compare $\lexSST$ against \wSST\ with two weight vectors: $\mathbf{w}^1 = [1, 0.01]$ (blue) and $\mathbf{w}^2 = [1, 0.05]$ (green), shown in Fig.~\ref{fig:WS2lex}. Here, a 
30-second planning budget was used. 

Despite both weight vectors appearing to heavily prioritize $c_1$, their behavior differs. In Fig.~\ref{fig:WS2lex traj}, $\mathbf{w}^1$ assigns insufficient weight to \PL, resulting in solutions with high variance in $c_2$, similar to vanilla \SST. Conversely, $\mathbf{w}^2$ overweights \PL, biasing the planner toward the lower corridor at the expense of $ c_1$-optimality. In contrast, $\lexSST$ consistently selects the upper corridor, forming a tighter bundle around $\lexMin$. Fig.~\ref{fig:WS2lex cost} shows these results in the objective space: \wSST\ solutions exhibit either large $c_2$ variance ($\mathbf{w}^1$) or unbounded $c_1$ suboptimality ($\mathbf{w}^2$), while $\lexSST$ clusters about $\lexMin$ within the user-defined tolerance $\epsilon$.

The time series in Figs.~\ref{fig:WS2lex evol c1}--\ref{fig:WS2lex evol c2} are consistent with Case Study 1: all planners converge in $c_1$ at comparable rates, but only $\lexSST$ steadily improves $c_2$ toward $c_2(\lexMin)$; whereas the variance under $\mathbf{w}^1$ persists throughout the planning horizon. This underscores a fundamental limitation of scalarization: no fixed weight vector reliably guarantees near-optimality of the lexicographic solution, whereas $\lexSST$ achieves this without any weight tuning.

\subsection{Constrained Optimization}
\label{subsec:cs3}

We now evaluate $\conSST$ on Problem~\ref{prob:conOptMP}, examining its ability to restrict the search tree to constraint-admissible regions and its completeness advantage over \SST. In the following studies, we use the \DI model and a 10-second time budget per planning instance. 

\subsubsection*{Case Study 3 -- Constraint-Guided Search}

\begin{figure}[b]
  \centering
  \begin{subfigure}{0.24\textwidth}
    \centering
    \includegraphics[width=\textwidth]{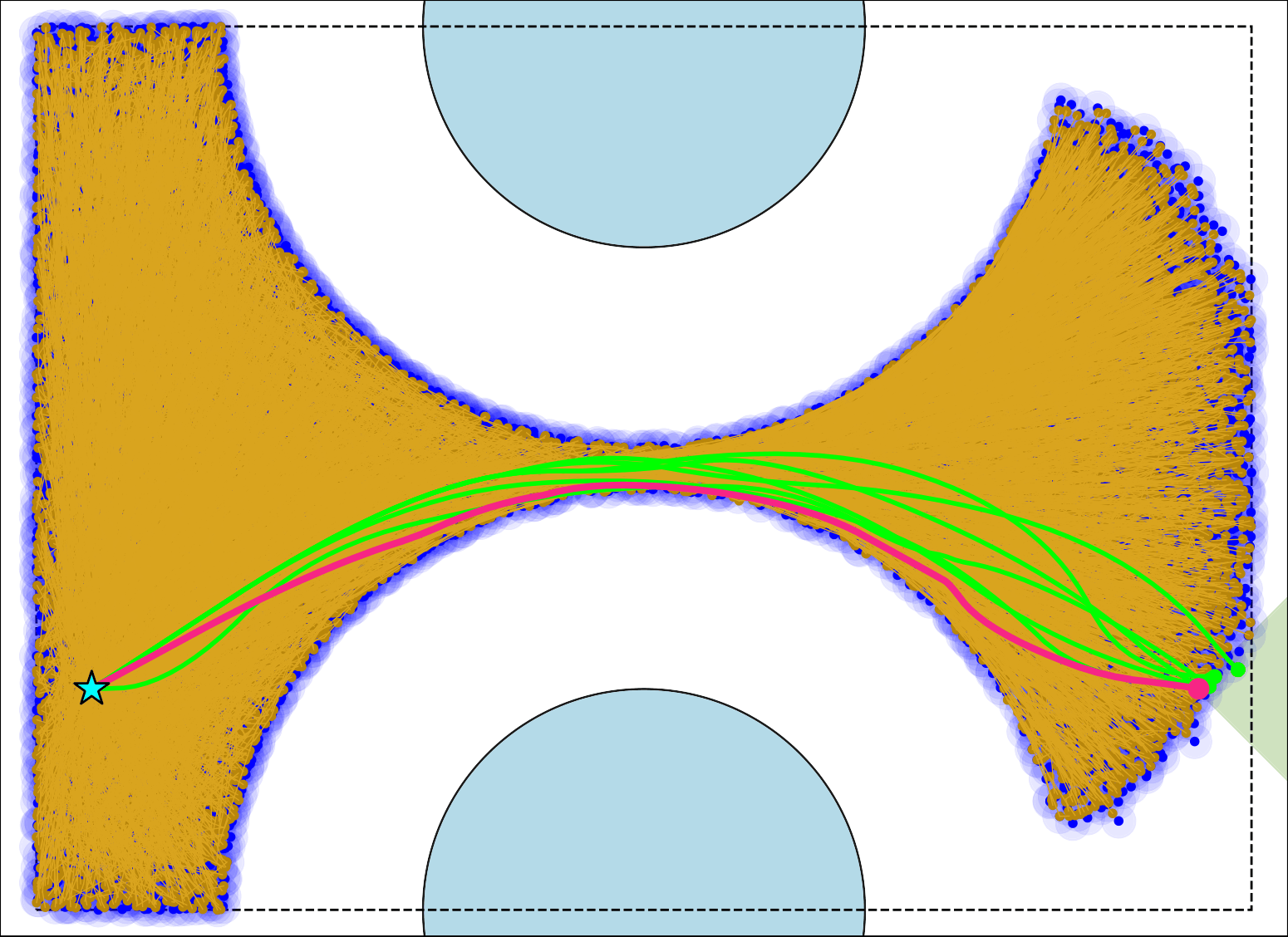}
    \caption{$\conSST$ search tree with $\algSols$ dominance pruning.}
    \label{fig:conSST-solSet}
  \end{subfigure}
  \hfill
  \begin{subfigure}{0.24\textwidth}
    \centering
    \includegraphics[width=\textwidth]{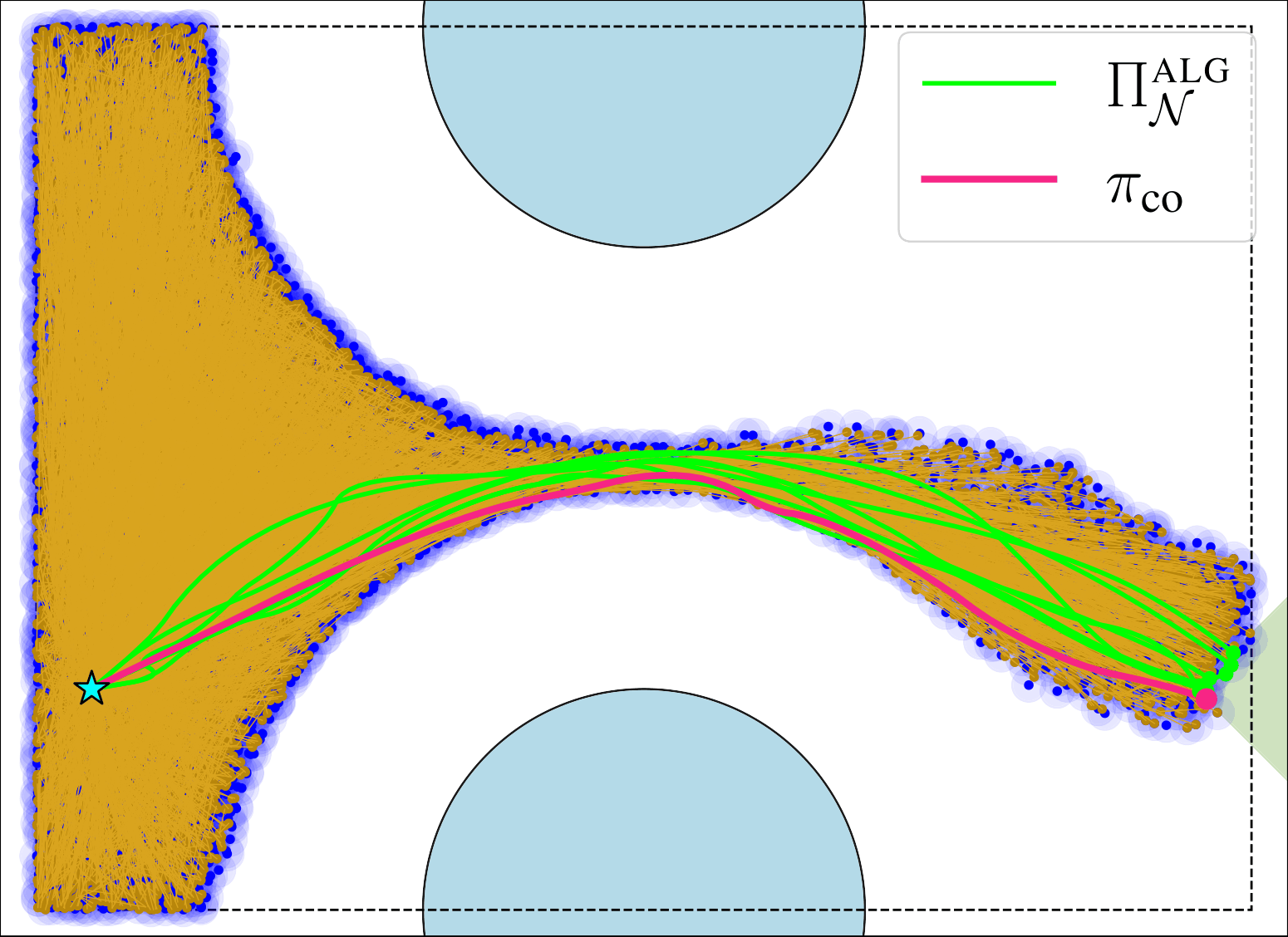}
    \caption{Search tree with cost-to-go pruning for \PL.}
    \label{fig:conSST-ctg}
  \end{subfigure}
  \caption{Case Study 3 (Constraint-guided search via \conSST). Constrained search of $\conMin$ for objectives \CL and \PL.}
  \label{fig:conSST-tree}
\end{figure}

We apply a constraint $\bar{c}_1 = 91$ to \CL in \ws, requiring all trajectories to maintain a clearance of at least 9 units from obstacles. Fig.~\ref{fig:conSST-tree} shows the resulting search trees after a 10-second planning window, with nodes in yellow and witness neighborhoods in blue. The tree is visibly confined to the admissible region, returning multiple non-dominated solutions (in green), with the one that minimizes \PL (in red).

\begin{figure*}[t]
  \centering
  \begin{subfigure}{0.229\textwidth}
    \centering
    \includegraphics[width=\textwidth]{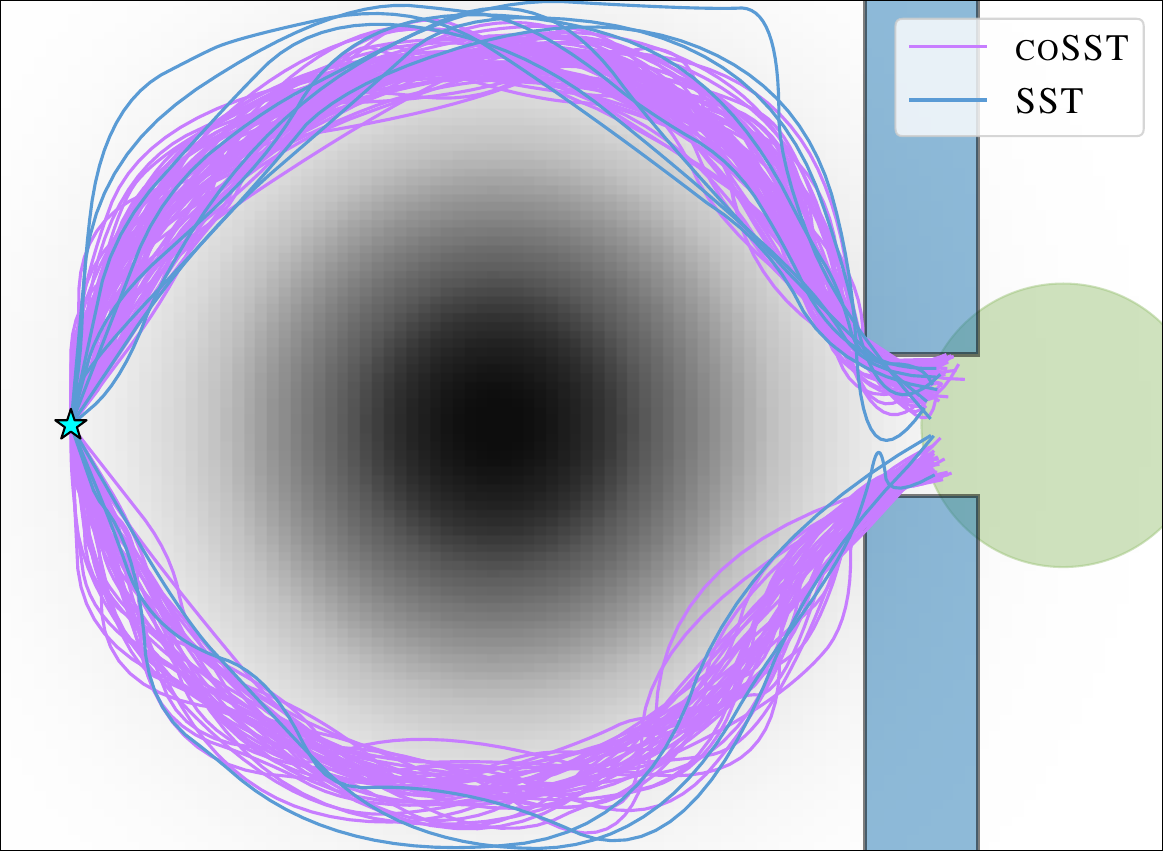}
    \caption{Solution trajectories.}
    \label{fig:bugtrap traj}
  \end{subfigure}
  \begin{subfigure}{0.141\textwidth}
    \centering
    \includegraphics[width=\textwidth]{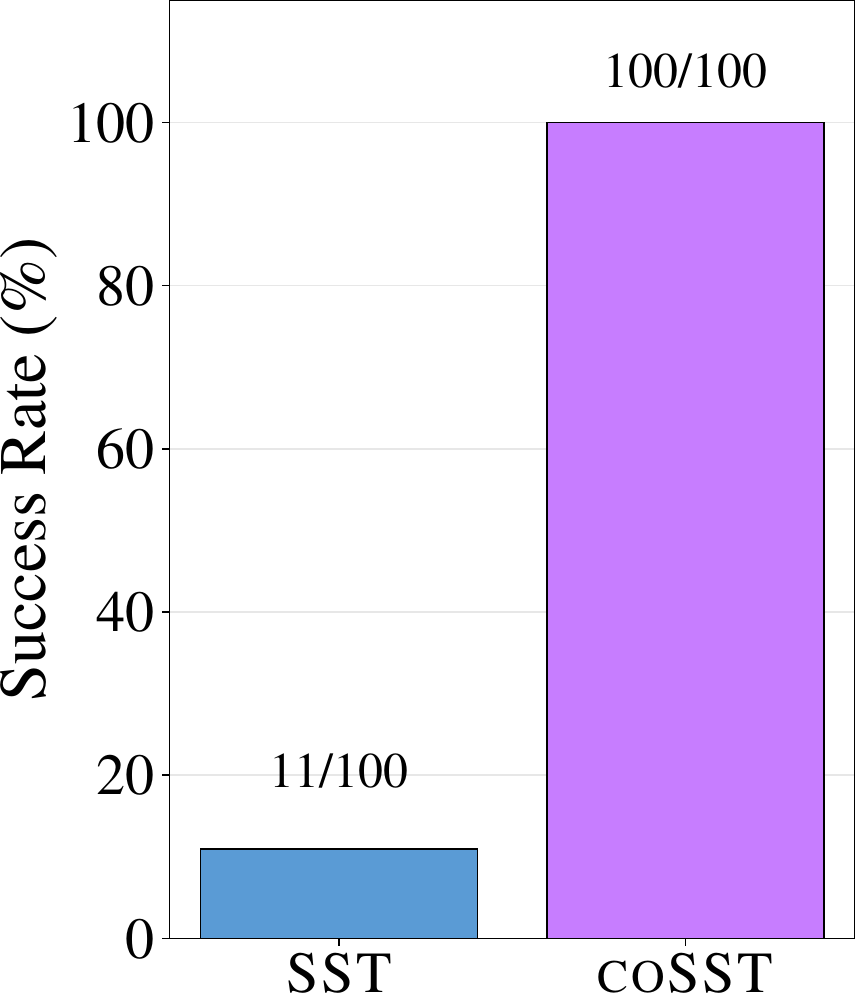}
    \caption{Success rate.}
    \label{fig:bugtrap succ}
  \end{subfigure}
  \hspace{0.1pt}
  \begin{subfigure}{0.229\textwidth}
    \centering
    \includegraphics[width=\textwidth]{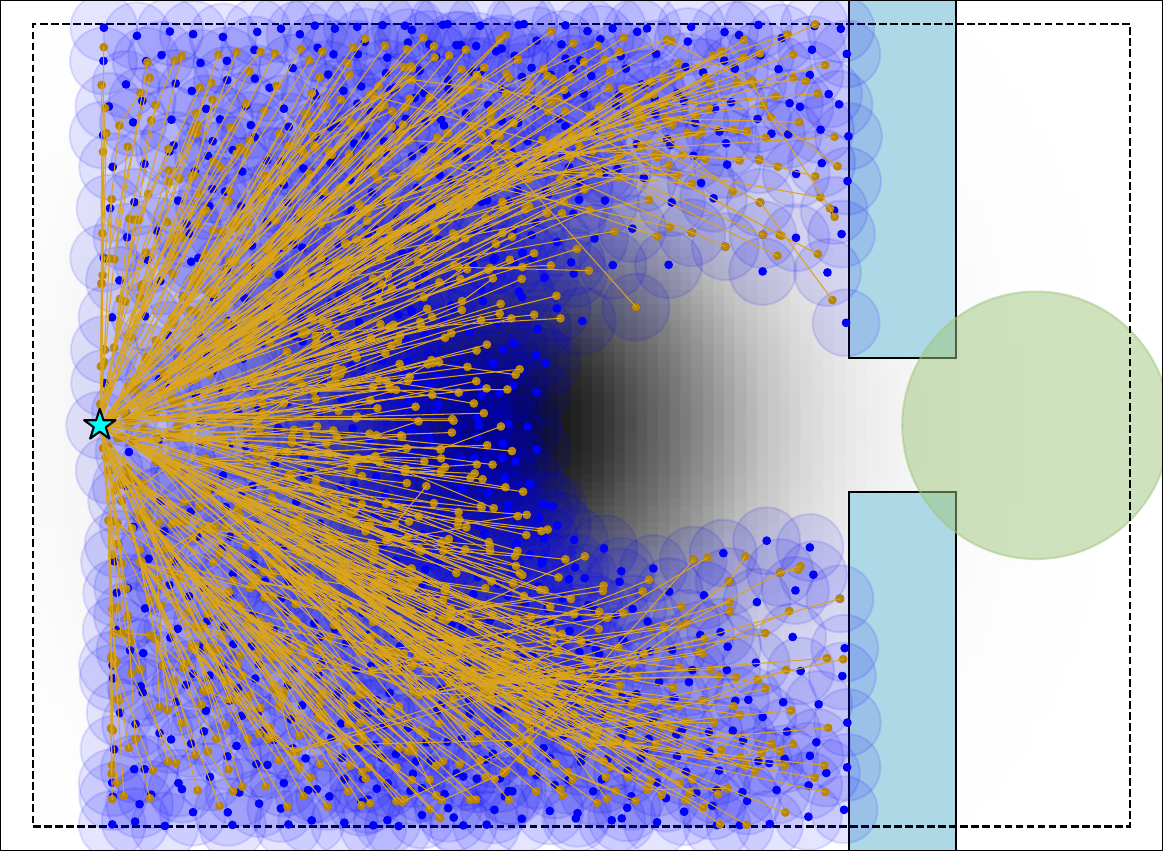}
    \caption{\SST\ search tree.}
    \label{fig:SST tree}
  \end{subfigure}
  \begin{subfigure}{0.229\textwidth}
    \centering
    \includegraphics[width=\textwidth]{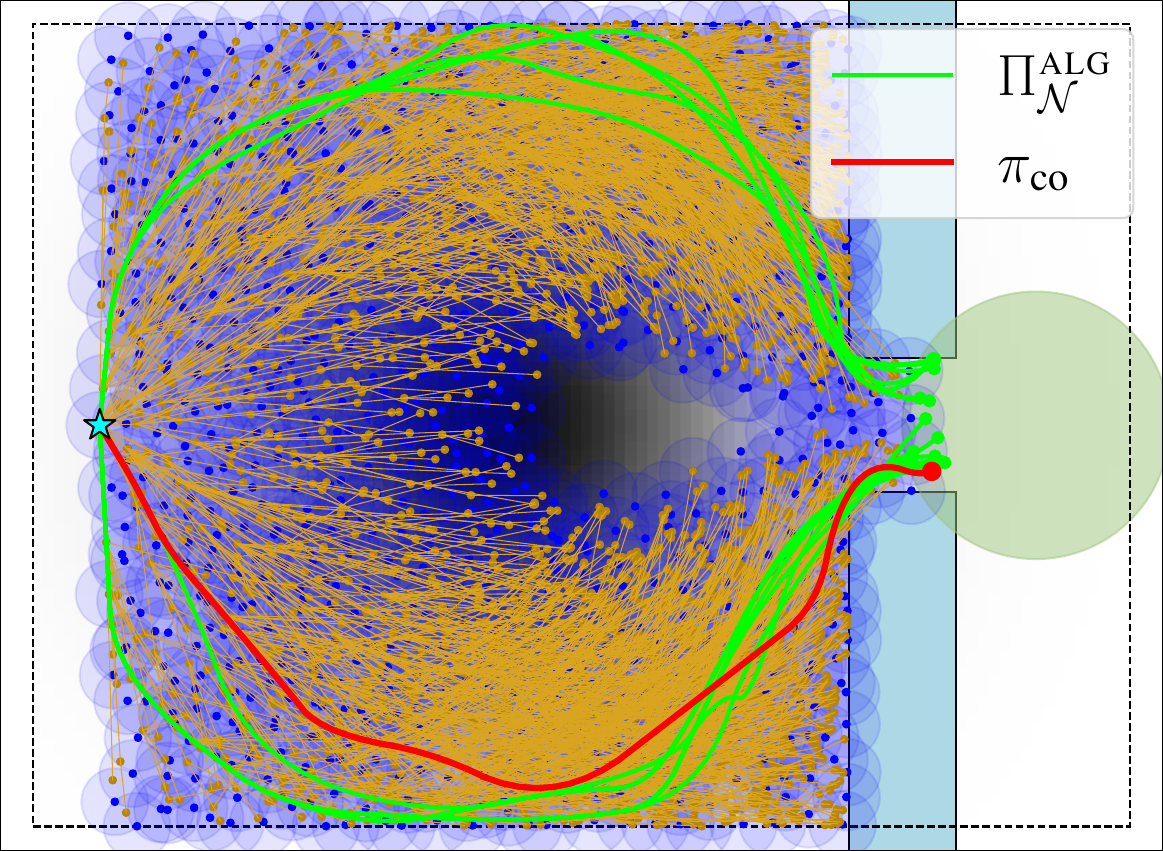}
    \caption{$\conSST$ search tree.}
    \label{fig:conSST tree}
  \end{subfigure}
  \begin{subfigure}{0.141\textwidth}
    \centering
    \includegraphics[width=\textwidth]{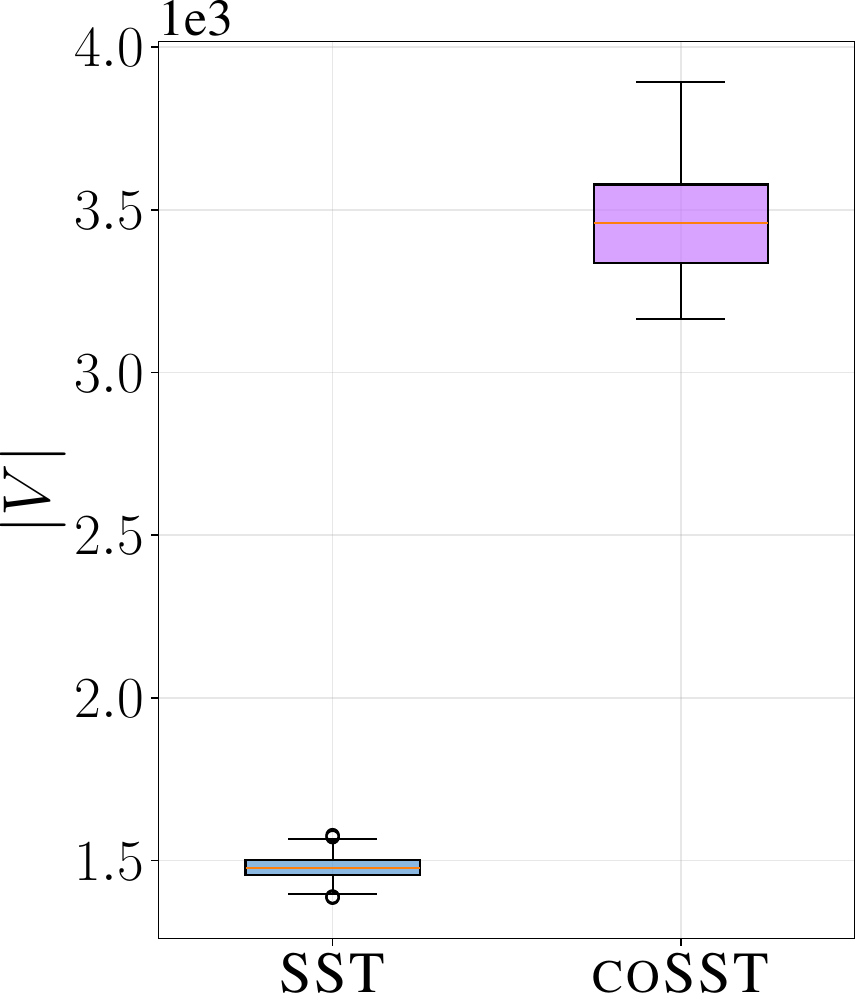}
    \caption{Tree size.}
    \label{fig:bugtrap tree size}
  \end{subfigure}
  \caption{Case Study 4 (incompleteness of \SST for constrained optimization). \SST\ fails to find feasible solutions in most runs, whereas $\conSST$ maintains completeness by retaining multiple constraint-admissible subtrajectories for each witness.}
  \label{fig:bugtrap}
\end{figure*}

Two additional pruning mechanisms further focus the search. First, we utilize a solution-set dominance check to prune nodes that are strictly dominated by existing solutions in $\algSols$. Since costs increase monotonically, there is no risk of prematurely pruning such nodes. As shown in Fig.~\ref{fig:conSST-solSet}, this prevents exploration of the upper-right and lower-right regions of the workspace, since shorter solutions already exist. Second, we leverage an admissible cost-to-go heuristic which inflates each candidate node's \PL cost by the shortest remaining path to the goal. As shown in Fig.~\ref{fig:conSST-ctg}, this disqualifies a substantially larger region, since reaching the goal within the remaining \PL budget becomes infeasible from those states.

\subsubsection*{Case Study 4 -- Incompleteness of \SST}
\label{sec:cs4}

We empirically confirm the incompleteness of \SST\ for constrained optimization, as argued in Section~\ref{subsec:SSTlimitaion}. The task is to minimize \PL\ subject to a constraint on \CF in the narrow-passage environment in Fig.~\ref{fig:bugtrap}, with sparsity parameters chosen to guarantee the existence of a $\delta$-robust solution. We conduct 100 independent runs per planner, each with a 30-second time budget, using the \DI model.

As shown in Figs.~\ref{fig:bugtrap traj}--\ref{fig:bugtrap succ}, \SST (blue) finds a valid motion plan in only a small fraction (0.11) of runs, whereas $\conSST$ succeeds across all runs. This failure is structural, not a consequence of limited planning time. Figs.~\ref{fig:SST tree}--\ref{fig:conSST tree} show the search trees after a 300-second budget to ensure convergence. Because \SST\ greedily optimizes for \PL, its single representative at each witness accumulates a high \CF, leaving no feasible extensions through the narrow passage. In contrast, $\conSST$ retains multiple constraint-admissible subtrajectories per witness, preserving the ability to extend through the passage. The cost of this completeness guarantee is a larger search tree, as shown in Fig.~\ref{fig:bugtrap tree size}, since multiple nodes are maintained per witness.

\subsection{Pareto Front Coverage}
Having validated the single-solution algorithms, we now evaluate $\poSST$ on Problem~\ref{prob:paretoOptMP}. We first demonstrate \poSST's Pareto-completeness and near-optimality, and then show its computational advantages relative to scalarization methods. 

\subsubsection*{Case Study 5 -- Convex Pareto Front}
\label{subsec:cs5}
We again consider a \DI model in \ws\ under the objectives \CL and \PL, but without imposing any priority among the objectives, aiming to recover the entire Pareto front.
Fig.~\ref{fig:WS1poSST traj} shows an example run as $\poSST$ discovers a diverse set of non-dominated trajectories that span the full Pareto front. Fig.~\ref{fig:WS1poSST pareto evo} shows how the approximate front computed by \poSST progressively approaches the theoretical optimum over time, with lighter hues indicating later times. 

Over 100 runs with a 60-second planning budget, $\poSST$ identified, on average, 50 non-dominated solutions per run, providing broad coverage of the objective space as shown in Fig.~\ref{fig:WS1poSST fronts}.
To quantify approximation quality, we define a coverage metric according to the ratio of the region dominated by the approximate front to the area enclosed by the true Pareto front and the axis-aligned boundary (red dashed lines in Fig.~\ref{fig:WS1poSST pareto evo}). As shown in Fig.~\ref{fig:WS1poSST coverage}, $\poSST$ dominated, on average, roughly 80\% of the reference area over the 60-second planning horizon, demonstrating its near-Pareto-optimality.

\begin{figure}[tbp]
  \centering
  \begin{subfigure}{0.24\textwidth}
    \centering
    \includegraphics[width=\textwidth]{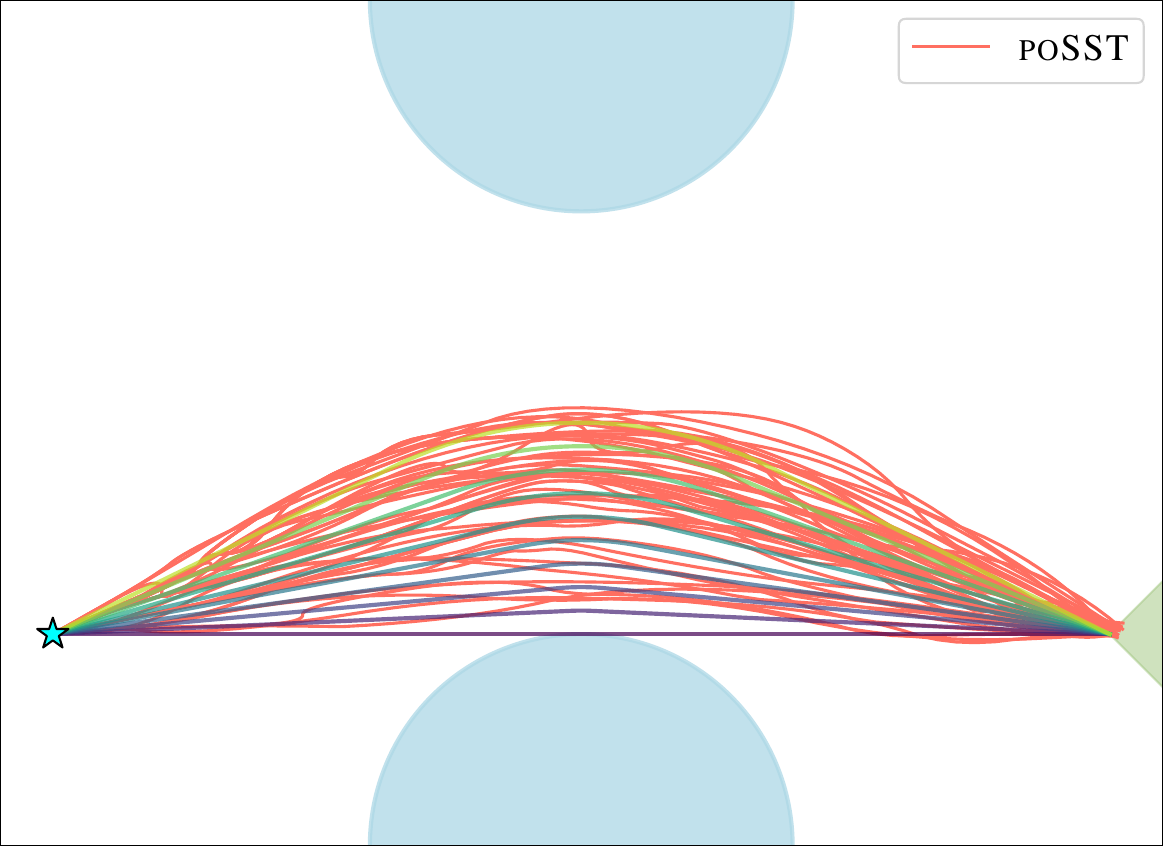}
    \caption{Pareto-optimal trajectories.}
    \label{fig:WS1poSST traj}
  \end{subfigure}
  \hfill
  \begin{subfigure}{0.24\textwidth}
    \centering
    \includegraphics[width=\textwidth]{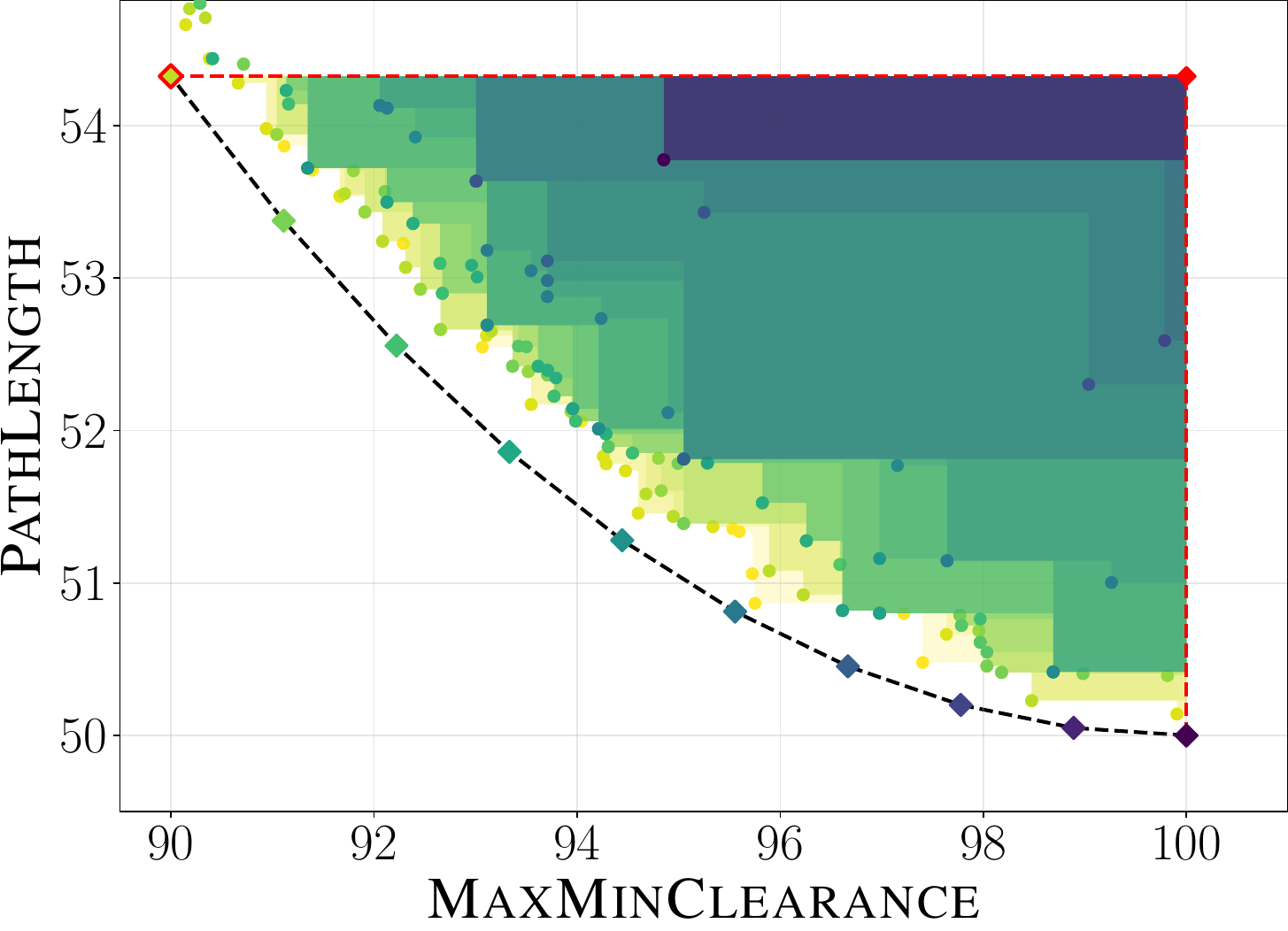}
    \caption{Evolution of Pareto front.}
    \label{fig:WS1poSST pareto evo}
  \end{subfigure}
  \begin{subfigure}{0.24\textwidth}
    \centering
    \includegraphics[width=\textwidth]{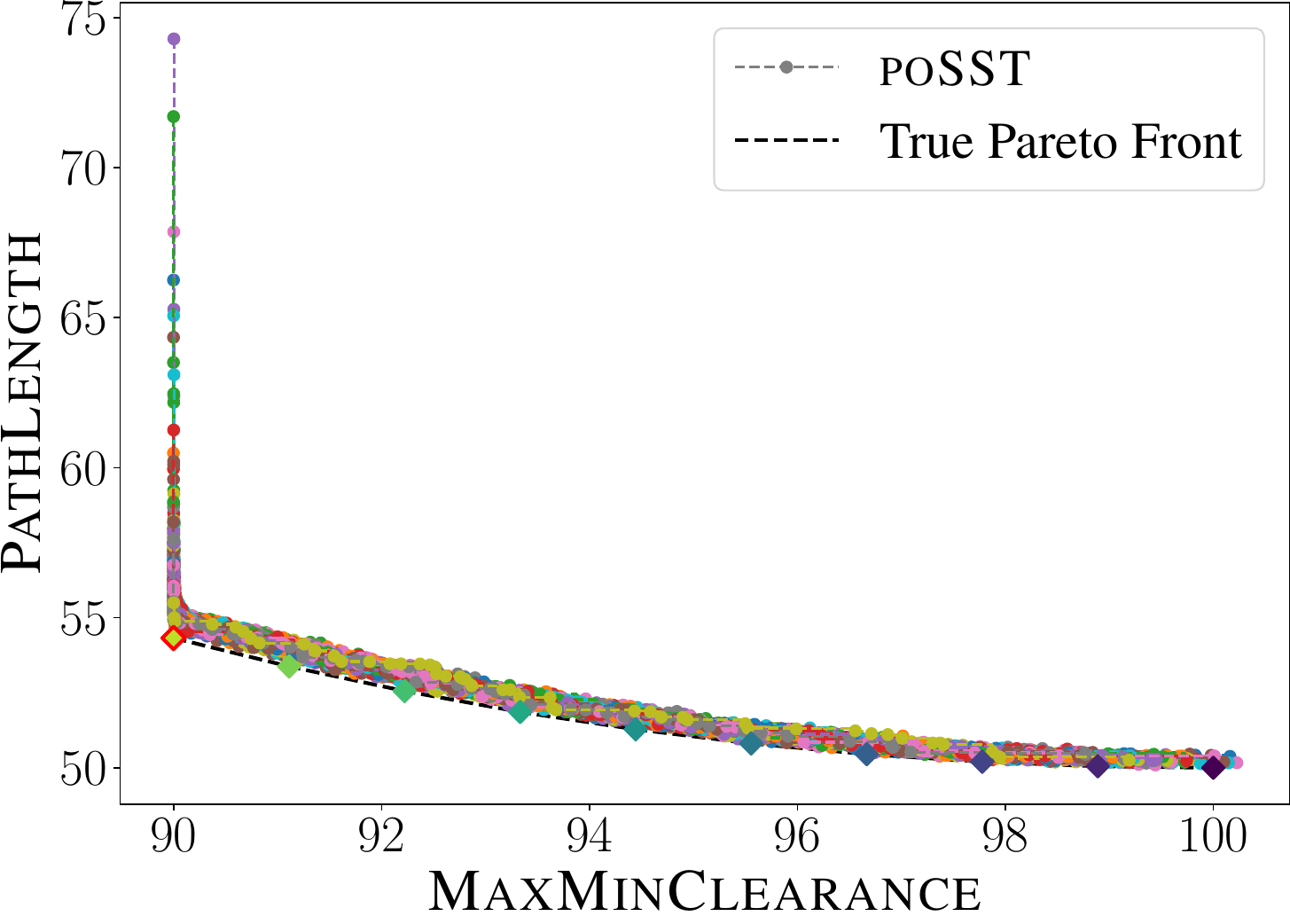}
    \caption{Pareto fronts (all runs).}
    \label{fig:WS1poSST fronts}
  \end{subfigure}
  \hfill
  \begin{subfigure}{0.24\textwidth}
    \centering
    \includegraphics[width=\textwidth]{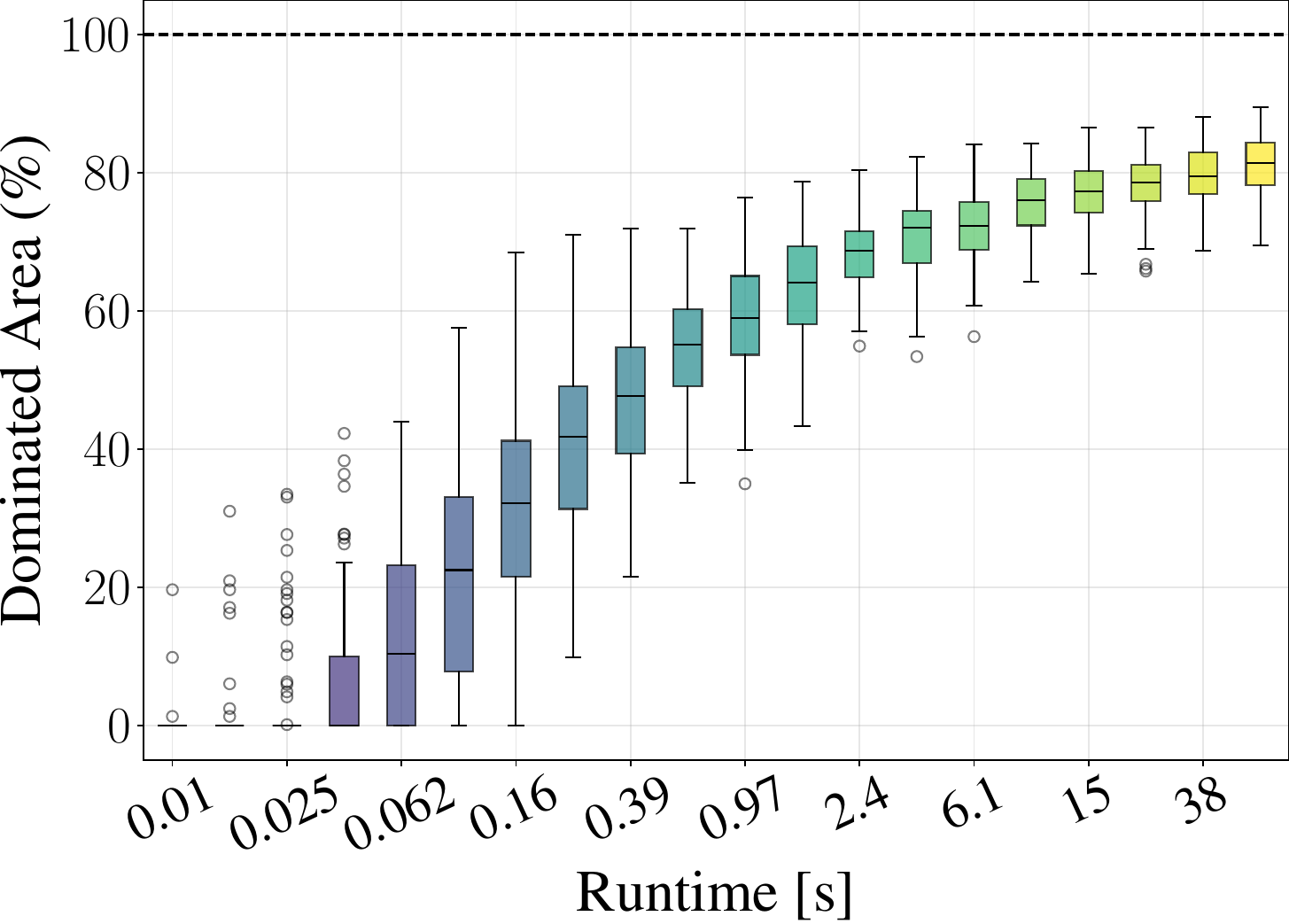}
    \caption{Coverage metric of $\algSols$.}
    \label{fig:WS1poSST coverage}
  \end{subfigure}
  \caption{Case Study 5 (efficacy of $\poSST$ in approximating $\Pi^*_\text{sol}$ with objectives \CL and \PL). (a) and (b) show a single representative run. (c) and (d) aggregate over 100 runs.}
  \label{fig:WS1po}
\end{figure}

\subsubsection*{Case Study 6 -- Non-Convex Pareto Fronts}
\label{subsec:cs6}

A key motivation for this work is the structural inability of weighted-sum scalarization to recover non-convex Pareto fronts, as established in Section~\ref{sec:prelim}. We evaluate $\poSST$ with objectives \CF and \PL across two environments: \wsss with three homotopic solution classes and a cluttered workspace, \wssss. We compare \poSST against \wSST initialized with 101 different weight vectors, spanning a normalized sweep of unique ratios to promote diverse solutions: 
\[\mathbf{w^\alpha} = [\alpha, 1-\alpha] \text{ for all } \alpha \in \{0, 0.01, 0.02, \dots, 1\}.\]
Each instance of \wSST is given a planning budget of 30-seconds (3030 seconds total), while $\poSST$ is given a single 30-second planning instance. This comparison is intentionally demanding, as the baseline is allocated $101\times$ the computational budget of $\poSST$.

\textbf{Non-convex front recovery} (Figs.~\ref{fig:WS3paths}, \ref{fig:WS3fronts}): Workspace \wsss, using the \DI model, isolates the structural limitation of WS scalarization (\wSST) on a highly non-convex Pareto front, where many optimal trade-offs do not lie on the convex hull. Across the 101 weight selections, \wSST collapses to two solution clusters: one that prioritizes \PL by passing through the gap, and one that prioritizes \CF by traveling over the wall. The intermediate trade-offs that balance both objectives are missed entirely. $\poSST$, by contrast, discovers a diverse set of solutions spanning all three homotopic classes, empirically demonstrating the Pareto-completeness established in Theorem~\ref{thm:poSST-complete}.

\begin{figure}[tbp]
  \centering
  \begin{subfigure}{0.24\textwidth}
    \centering
    \includegraphics[width=\textwidth]{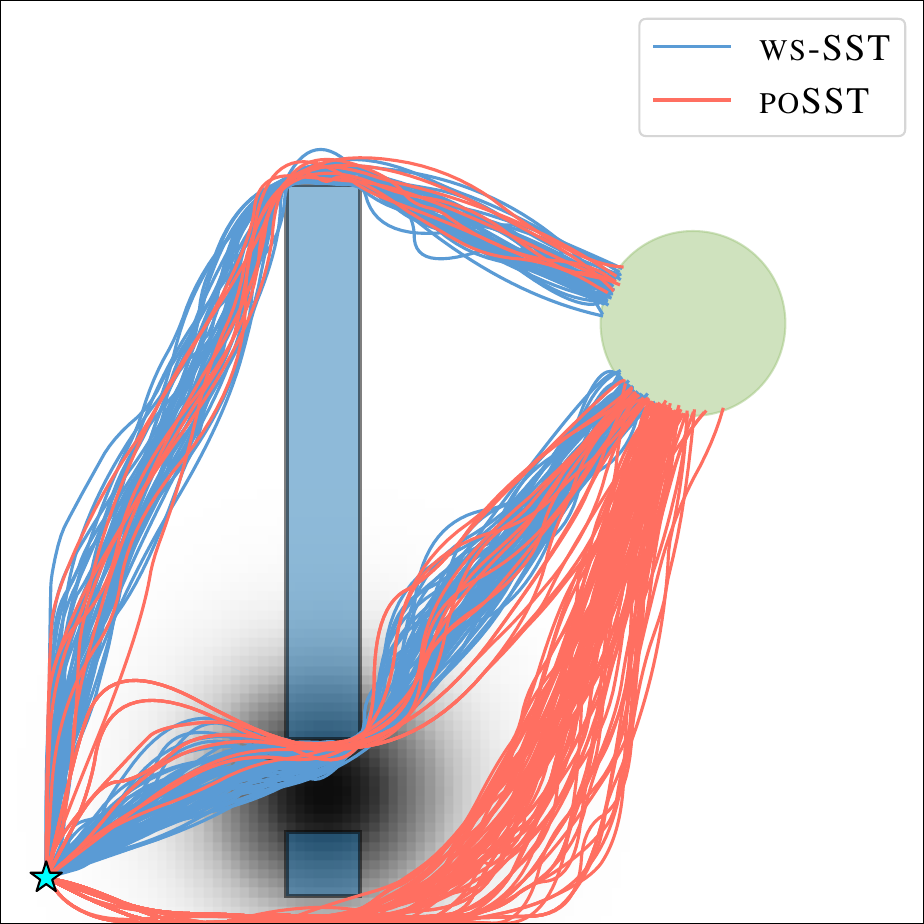}
    \caption{\wsss solution trajectories.}
    \label{fig:WS3paths}
  \end{subfigure}
  \hfill
  \begin{subfigure}{0.24\textwidth}
    \centering
    \includegraphics[width=\textwidth]{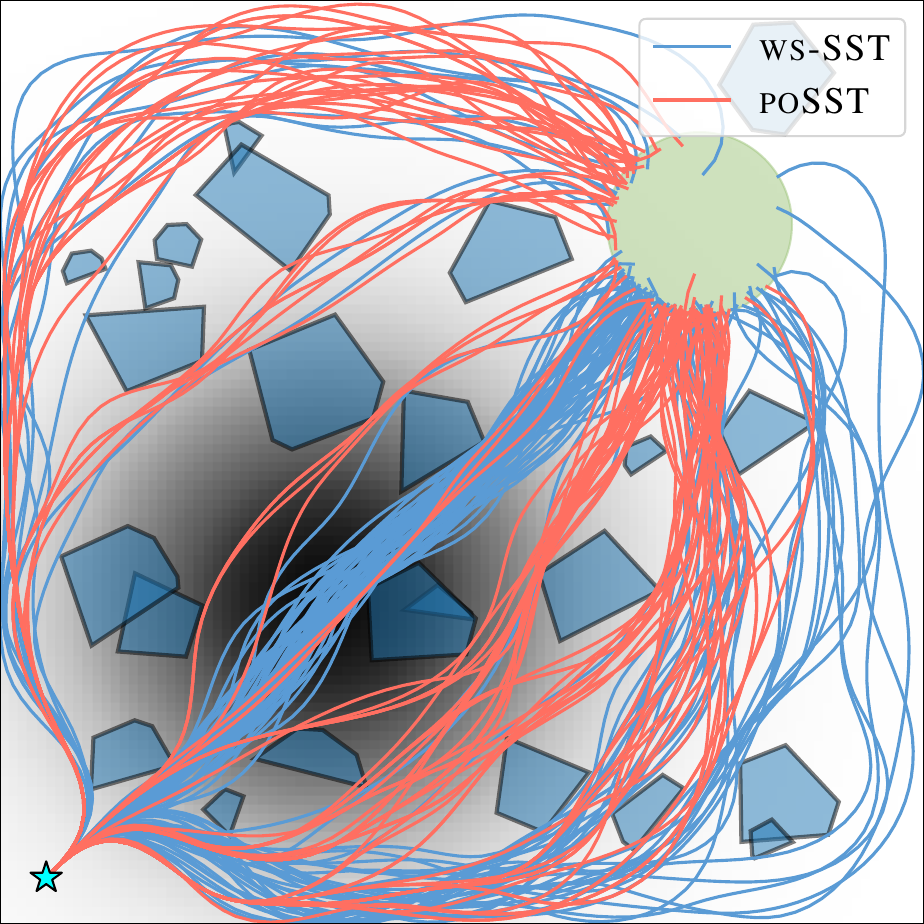}
    \caption{\wssss solution trajectories.}
    \label{fig:WS4paths}
  \end{subfigure}
  \begin{subfigure}{0.24\textwidth}
    \centering
    \includegraphics[width=\textwidth]{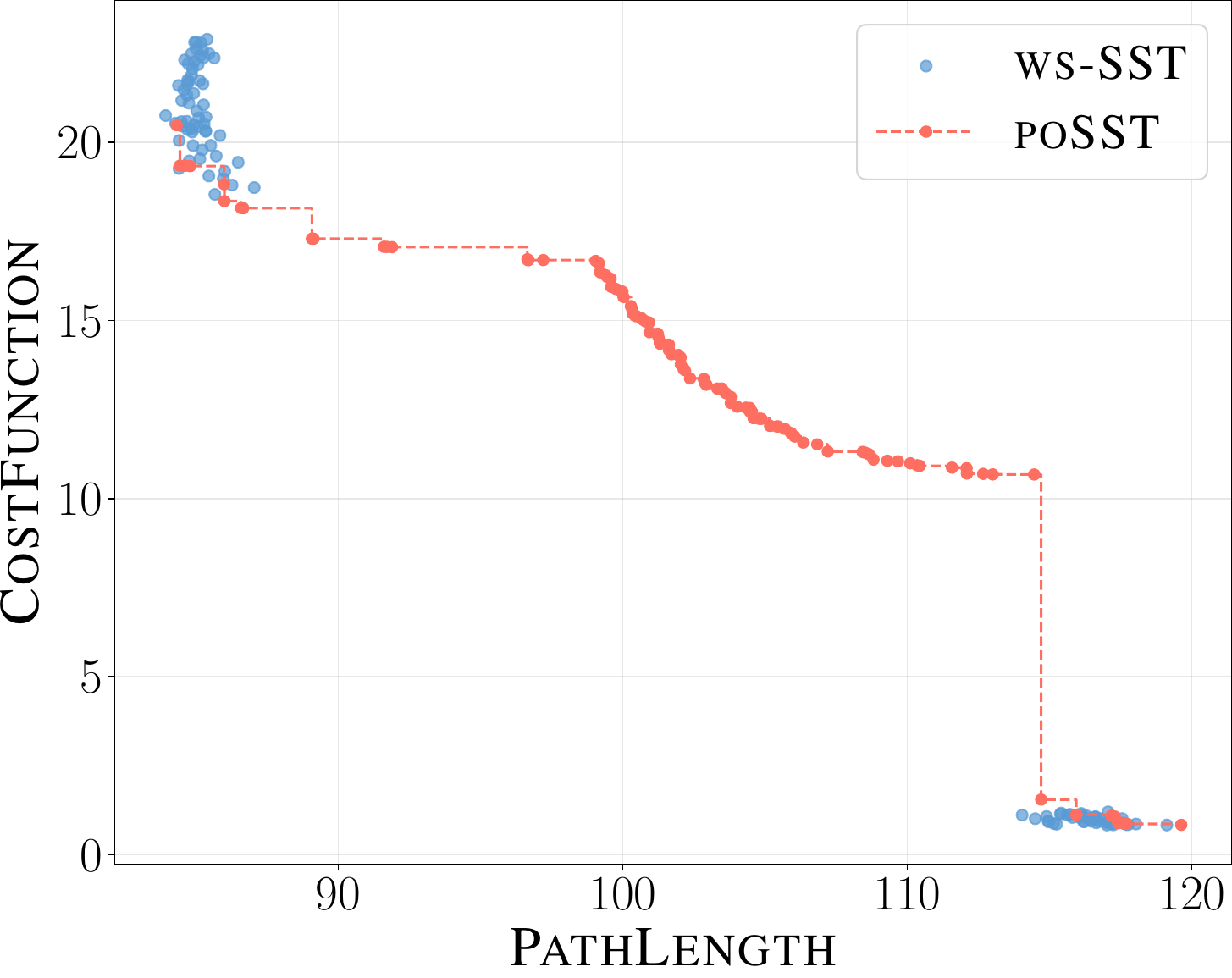}
    \caption{\wsss Pareto front.}
    \label{fig:WS3fronts}
  \end{subfigure}
  \hfill
  \begin{subfigure}{0.24\textwidth}
    \centering
    \includegraphics[width=\textwidth]{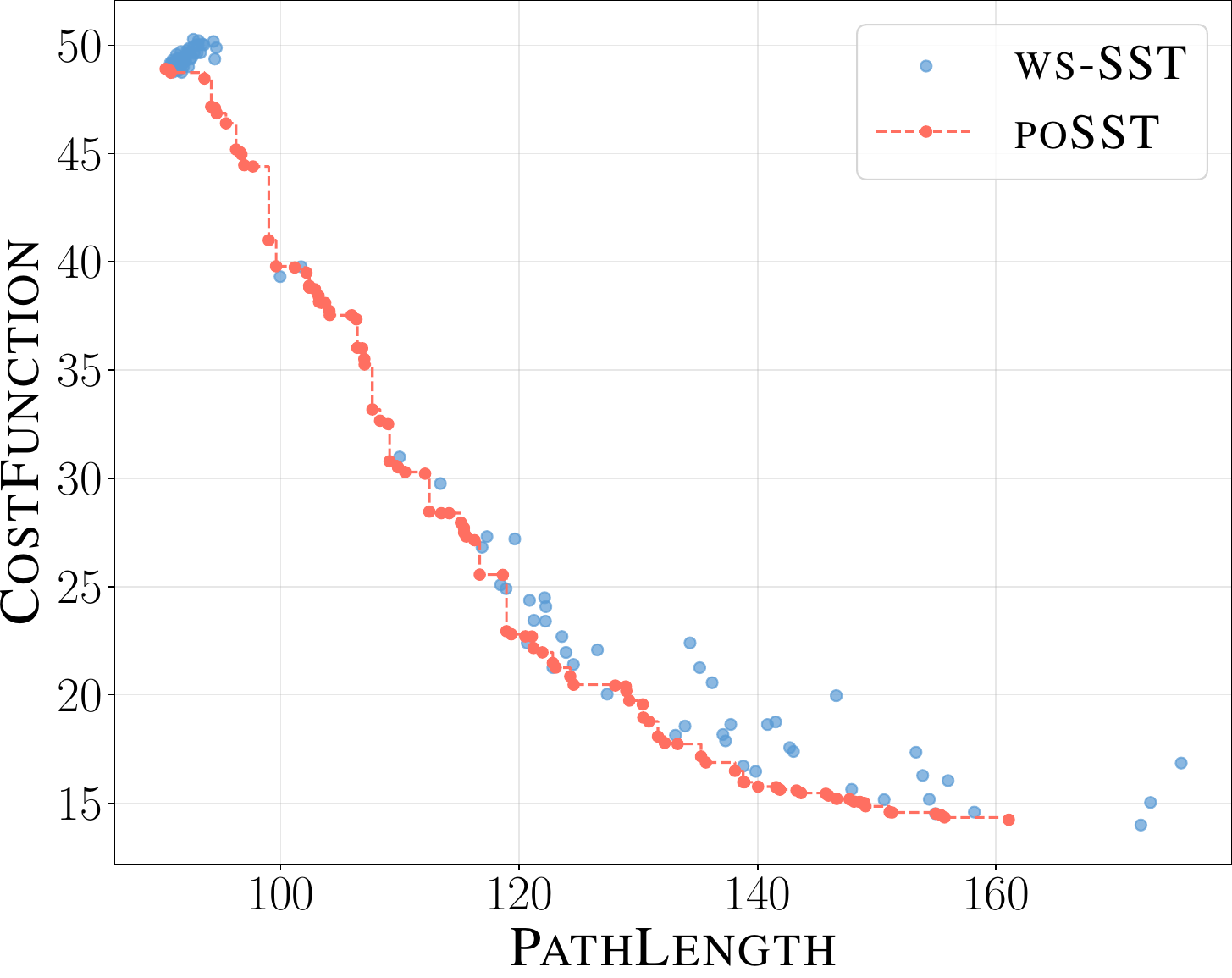}
    \caption{\wssss Pareto front.}
    \label{fig:WS4fronts}
  \end{subfigure}
  \caption{Case Study 6 (\poSST vs. \wSST on non-convex Pareto fronts). In \wsss, \wSST recovers only extreme solutions, while $\poSST$ discovers trade-offs across all homotopic classes. In \wssss, a single run of $\poSST$ achieves better coverage than 101 scalarized planning instances combined.}
  \label{fig:pareto front results}
\end{figure}

% \textbf{Computational efficiency} (Figs.~\ref{fig:WS4paths}, \ref{fig:WS4fronts}): In \wssss using the \BI model, $\poSST$ recovers 124 unique Pareto-optimal solutions within a single 30-second planning instance, achieving better coverage than \wSST across all 101 weight selections. 
% As seen in Fig.~\ref{fig:WS4fronts}, the cluster of points about the optimal \PL solutions underscores the inherent clumsiness of summing weights across objectives that operate over different ranges and units of measure. 
% In this setup, since \PL operates at a higher range, any weight ratio greater than $\frac{1}{2}$ yields the same solution, rendering many weight selections redundant. However, this cannot be known a priori. In contrast, \poSST performs element-wise comparisons across the various objectives, thereby completely circumventing the need for a unified metric. 
% This demonstrates that maintaining a single, non-dominated search tree is substantially more efficient than repeatedly scalarizing and re-planning to recover a diverse set of tradeoffs.

\textbf{Computational efficiency} (Figs.~\ref{fig:WS4paths}, \ref{fig:WS4fronts}): In \wssss using the \BI model, $\poSST$ recovers 124 unique Pareto-optimal solutions within a single 30-second planning instance, achieving better coverage than \wSST across all 101 weight selections. 
As seen in Fig.~\ref{fig:WS4fronts}, the cluster of solutions near the \PL optimum reveals a key pitfall of WS scalarization: since \PL operates over a larger range than \CL, any weight ratio greater than about 0.5 collapses to the same solution, rendering most weight selections redundant. This imbalance is not known a priori, making it difficult to choose weights that yield a diverse spread. $\poSST$ avoids this issue entirely, as its element-wise cost comparisons require no unified (scalar) cost metric across objectives. 
Further, this case study highlights that maintaining a single, non-dominated search tree is substantially more efficient than repeatedly scalarizing and re-planning to recover a diverse set of trade-offs.

\subsubsection*{Case Study 7 -- Effect of $\vec\varepsilon$ on Pareto Convergence}
\label{subsec:cs7}

Finally, we examine the effect of the objective-space sparsity parameter $\vec\varepsilon$ introduced in this work, which governs the trade-off between near-optimality and tree size (sparsity). Using workspace \ws and the \DI model, we sweep $\vec\varepsilon$ over a range of values and run 100 instances of $\poSST$ for each configuration over a 300-second planning horizon.

Fig.~\ref{fig:WS1poBench} presents results as box plots. As expected, larger $\vec\varepsilon$ values admit fewer nodes per witness neighborhood (Fig.~\ref{fig:eps-node}), reducing tree size but coarsening the approximation of the Pareto front (Fig.~\ref{fig:eps-sols}). Conversely, smaller $\vec\varepsilon$ yields denser trees with tighter approximations at the cost of slower convergence rates (Fig.~\ref{fig:eps-coverage}). This experiment provides practical guidance for tuning $\vec\varepsilon$: in time-constrained settings, a moderate $\vec\varepsilon$ efficiently achieves broad Pareto coverage, whereas $\vec\varepsilon \to \mathbf{0}$ is appropriate when tight bounds are required.

\begin{figure}[tbp]
  \centering
  \begin{subfigure}{0.24\textwidth}
    \centering
    \includegraphics[width=\textwidth]{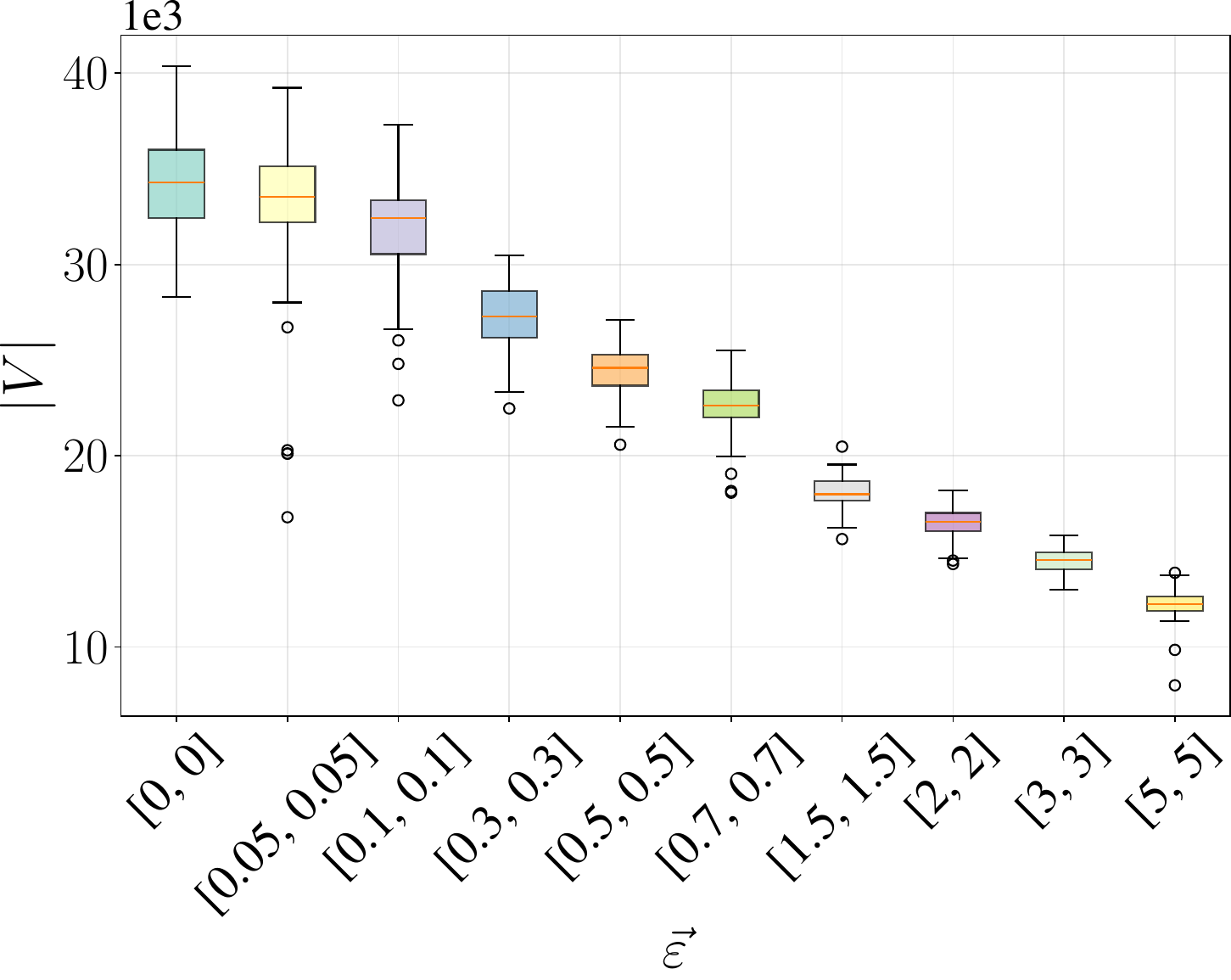}
    \caption{Size of search tree.}
    \label{fig:eps-node}
  \end{subfigure}
  \hfill
  \begin{subfigure}{0.24\textwidth}
    \centering
    \includegraphics[width=\textwidth]{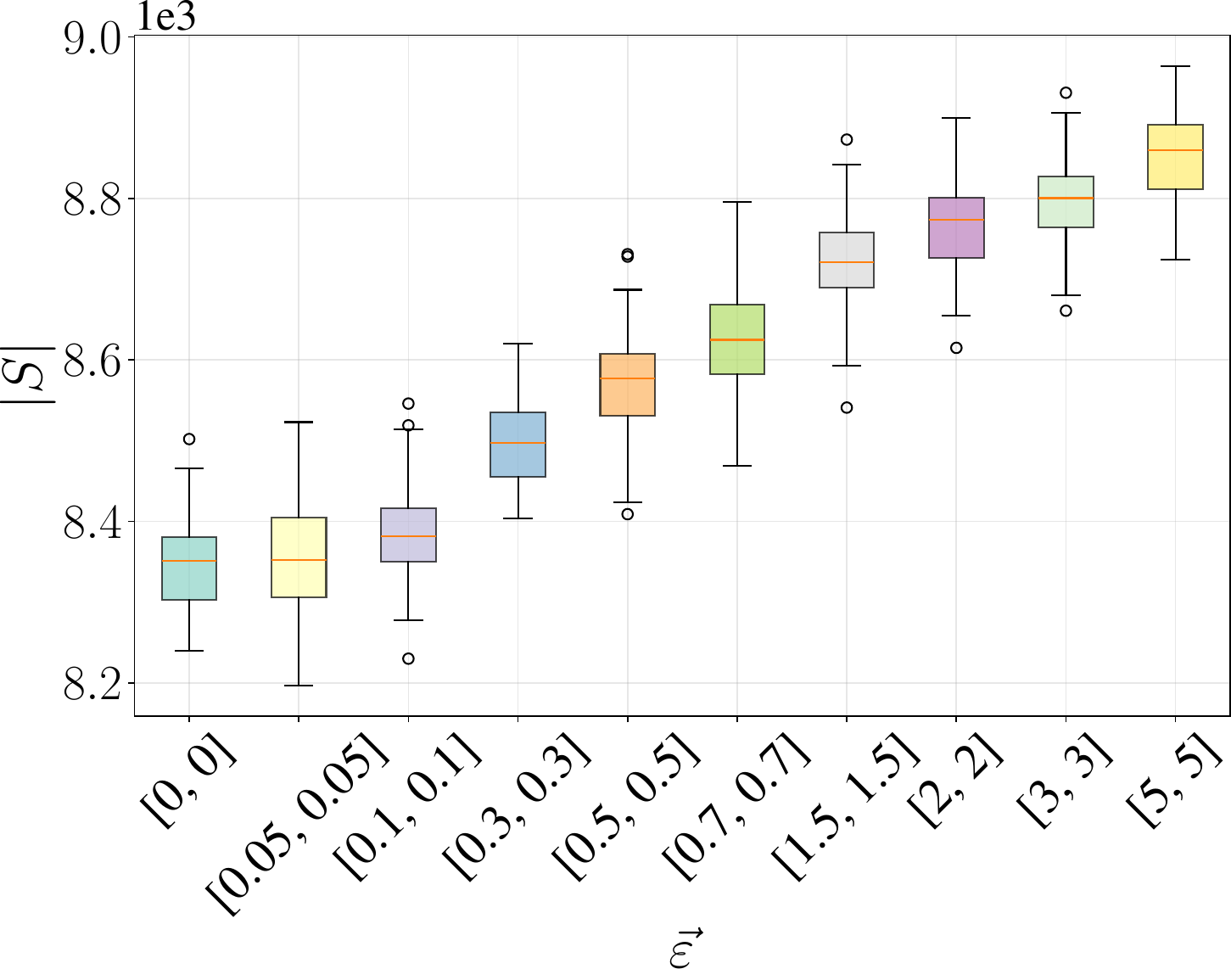}
    \caption{Number of witness nodes.}
    \label{fig:eps-witness}
  \end{subfigure}
  \begin{subfigure}{0.24\textwidth}
    \centering
    \includegraphics[width=\textwidth]{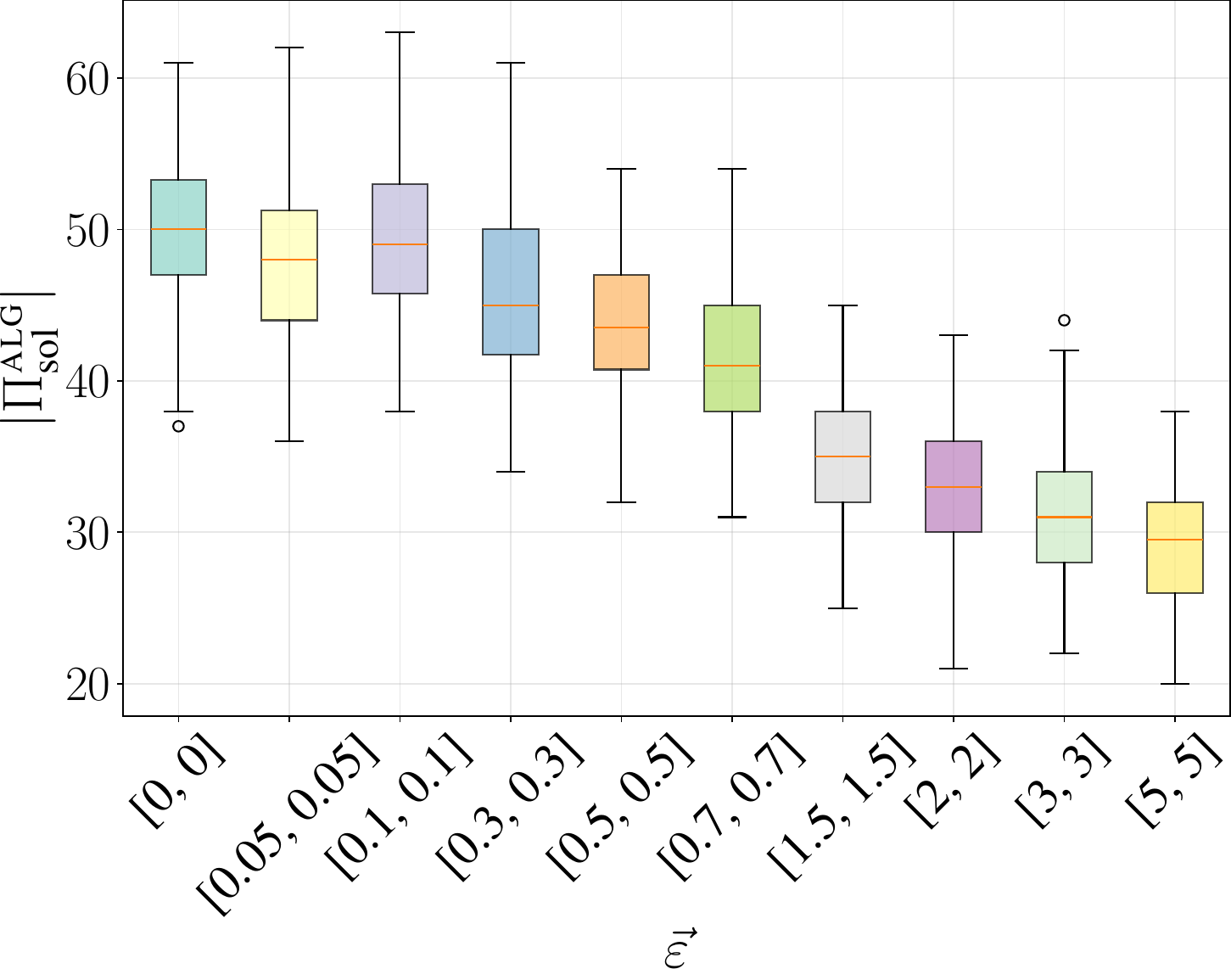}
    \caption{Size of solution set.}
    \label{fig:eps-sols}
  \end{subfigure}
  \hfill
  \begin{subfigure}{0.24\textwidth}
    \centering
    \includegraphics[width=\textwidth]{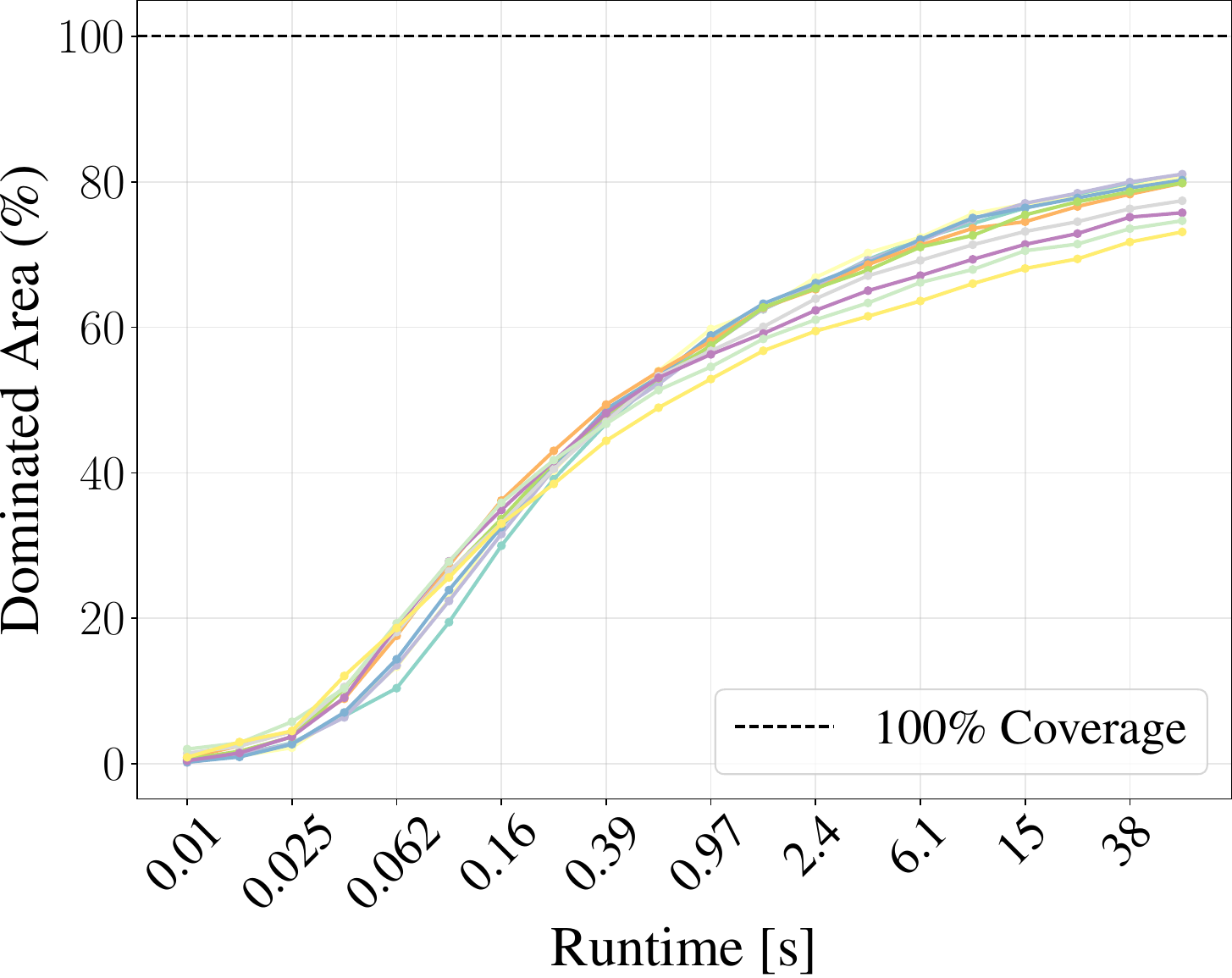}
    \caption{Coverage metric of $\algSols$.}
    \label{fig:eps-coverage}
  \end{subfigure}
  \caption{Case Study 7. Effect of $\vec\varepsilon$ on key metrics of \poSST with objectives \CL and \PL.}
  \label{fig:WS1poBench}
\end{figure}

\section{Conclusion}
\label{sec:conclusion}

In this paper, we introduced a unified framework for multi-objective kinodynamic motion planning built on the Stable Sparse-RRT algorithm. Our key insight is the extension from a single representative per witness neighborhood to a \emph{representative set}, enabling the simultaneous maintenance of locally Pareto-optimal subtrajectories within a single search tree. This structure underlies three algorithms: \lexSST for lexicographic minimization, \conSST for constrained optimization, and \poSST for Pareto-front approximation. 

For \lexSST, we proved that no scalar utility function can represent lexicographic dominance over continuous cost spaces, and,
% , introduced $\epsilon$-equivalence to enable secondary-objective refinement. 
% and established probabilistic $\delta$-robust completeness and asymptotic near-optimality for bi-objective problems. 
for \conSST, we showed that vanilla \SST is structurally incomplete under cost constraints. Instead, our algorithms are probabilistically complete and asymptotically near-optimal for both problems.
% 
% — confirmed empirically by the flytrap experiment (7\% vs.\ 100\% success) — 
% and proved completeness and near-optimality of our approach. 
Finally, \poSST approximates the entire Pareto front in a single planning instance with formal completeness and bounded sub-optimality guarantees, recovering diverse trade-offs  
% including non-convex Pareto regions inaccessible to weighted-sum scalarization, 
at a fraction of the computation time of repeated scalarized re-planning.

Several directions for future work remain. The $\epsilon$-equivalence approach for lexicographic search does not extend beyond two objectives; hence, how to apply \lexSST with three or more cost functions remains an open question.
% ; relatedly, while \conSST and \poSST are formulated for arbitrary $N$, empirical validation with three or more objectives is a natural next step. 
Also, since the cost of maintaining representative sets grows with the number of non-dominated nodes per witness, leveraging GPU-accelerated~\cite{KinoPAX2025, KinoPAXPlus2025} and vectorized~\cite{VAMP2024} sampling-based planning could bring these algorithms closer to real-time deployment. 
Adaptation of these algorithms and their formal analysis for multi-objective problems are left to future work.

% Finally, asymptotically optimal variants $\textsc{lexSST}^*$, $\textsc{coSST}^*$, and $\textsc{poSST}^*$ can be obtained by applying shrinking schedules to the sparsity parameters, analogous to $\textsc{SST}^*$~\cite{SST}; their formal analysis and validation are left to future work.

% \ml{points to make:\\
% - GPU parallelization, \\
% - existence of an Alg for $\epsilon$-optimal ordering of objectives in continuous domain.
% }

\appendix
\label{sec:appendix}
\subsection{Proof for Theorem \ref{thrm:lexOrder}}
\label{proof:lexOrder}

\begin{proof}
We prove the result by contradiction in the $N = 2$ case, and then argue that it extends to any $N > 2$.

Assume
%, for the sake of contradiction, 
that there exists a function $u: \mathbb{R}^2_{\geq0} \rightarrow \mathbb{R}$ such that $C \succ_{\text{lex}} C'$ iff $u(C) > u(C')$. Fix any $c_1 \in \mathbb{R}_{\geq0}$, then by definition of lexicographic preference, we have:
$(c_1, 0) \succ_{\text{lex}} (c_1, 1).$
So it must follow that:
$u((c_1, 0)) > u((c_1, 1)),$
and an interval between $c_1$-equivalent vectors can be defined:
$I(c_1) := [u((c_1, 1)), u((c_1, 0))] \subset \mathbb{R}.$

We claim that for any two distinct $c_1, c'_1 \in \mathbb{R}_+$, the intervals $I(c_1)$ and $I(c_1')$ are disjoint. W.L.O.G., assume $c_1 < c_1'$. Then:
$$(c_1, 1) \succ_{\text{lex}} (c_1', 0)
\ \Rightarrow \
u((c_1, 1)) > u((c_1', 0)).$$
By the construction of each interval, we have:
\[
 u((c_1', 1)) < u((c_1', 0)) < u((c_1, 1)) < u((c_1, 0)).
\]
Hence, 
% the intervals $I(c_1) = [u((c_1, 0)), u((c_1, 1))]$ and $I(c_1') = [u((c_1', 0)), u((c_1', 1))]$ are disjoint.
$I(c_1)  \cap I(c_1') = \emptyset$.
Now define the mapping:
\[
\phi : \mathbb{R}_+ \rightarrow \mathbb{I},  \quad \ \phi(c_1) = I(c_1),
\]
where $\mathbb{I}$ is the collection of disjoint intervals in $\mathbb{R}$. Since each $I(c_1)$ is distinct and disjoint, $\phi$ is injective.
Next, since the rationals $\mathbb{Q}$ are dense in $\mathbb{R}$, each interval $I(c_1)$ contains at least one rational number. Thus, we can define:
\[
\tau: \mathbb{I} \rightarrow \mathbb{Q}, \quad \tau(I(c_1)) = r_{c_1},
\]
where $r_{c_1}$ is any rational number contained in $I(c_1)$. Since each $r_{c_1}$ is drawn from a distinct and disjoint interval, they are unique. Thus, $\tau$ is injective and the composition $\tau \circ \phi : \mathbb{R}_{\geq0} \rightarrow \mathbb{Q}$ is likewise injective, implying:
$|\mathbb{R}_{\geq0}| \leq |\mathbb{Q}|.$
This is a contradiction, since $\mathbb{R}_{\geq0}$ is uncountable and $\mathbb{Q}$ is countable. 
Hence, the assumption is incorrect, and no such $u$ exists.

The contradiction above holds for vectors in $\mathbb{R}^2_{\geq0}$, and generalizes naturally to $\mathbb{R}^n_{\geq0}$ for $n > 2$. Since lexicographic ordering compares objectives sequentially, the utility function will necessarily compare objectives $c_i$ and $c_{i+1}$, resulting in the same conclusion.
\end{proof}

\subsection{Proof for Lemma~\ref{lem:ParetoSelect}}
\label{proof:ParetoSelect}

\begin{proof}[Proof]
Consider a node $x \in \mathcal{B}_{\dps}(x^*_i)$ added to the tree at iteration $n$. 
Suppose a sample $x_{\text{rand}}$ is drawn such that 
$x_{\text{rand}} \in \mathcal{B}_{\dps}(x) \cap B_\theta(x^*_i),$
where $\theta = \delta - \dps > 0$ by Proposition~\ref{prop:deltas}. 
In this case, node $x$ lies within the region considered by \textsc{ParetoSelect}, and $\gamma_x > 0$ unless $x$ is strictly dominated by another node $x'$ (i.e., $C(x') \succ C(x)$). 
In that case, $x'$ is the better representative of $\pi^*$.  
% 
% It is important to note that neither $\poSST$  nor \textsc{SST} guarantees that a $\delta$-similar trajectory to $\pi^*$ will be returned. Rather, they guarantee only that such a trajectory \emph{can} be found, and will be excluded only if a strictly better subtrajectory exists within the covering ball sequence. From that ball forward, the selected node $x'$ can continue to follow the reference with $\delta$-similarity, thereby preserving the cost bounds.
% 
This leads to a positive lower bound on the probability of selecting a node within $\mathcal{B}_\delta(x^*_i)$ that can generate a $\delta$-similar trajectory to $\pi^*$:
\[
\gamma = \frac{\mu(B_\theta(x^*_i) \cap \mathcal{B}_{\dps}(x))}{|\mathcal{B}_{\dps}(x))| \cdot \mu(\X_f)} > 0,
\]
where $\mu(\cdot)$ denotes the Lebesgue measure and $|\cdot |$ denotes the cardinality of a set.
% over the corresponding subset of the state space.
\end{proof}

\subsection{Proof for Lemma~\ref{lem:pruneDom}}
\label{proof:pruneDom}
\begin{proof}
Suppose a node $x$ is generated by $\poSST$  at iteration $n$ such that $x \in \mathcal{B}_{\delta_c}(x^*_i)$. Then, $x$ is either used to establish a new witness set or is added to an existing witness located at most $\delta_s$ away. In the latter case, $x$ can lie at most a distance $\delta_c$ from $x^*_i$ while being associated with a witness at a distance $\delta_c + \delta_s$ from $x^*_i$, as shown by the blue ball in Figure~\ref{fig:radiiSchematic}.
% \ml{what do you mean by `the worst case'? There are only two cases as you state previously.}\yr{maybe addressed.}
% 
\begin{figure}[htbp]
    \centering
    \includegraphics[width=0.8\columnwidth]{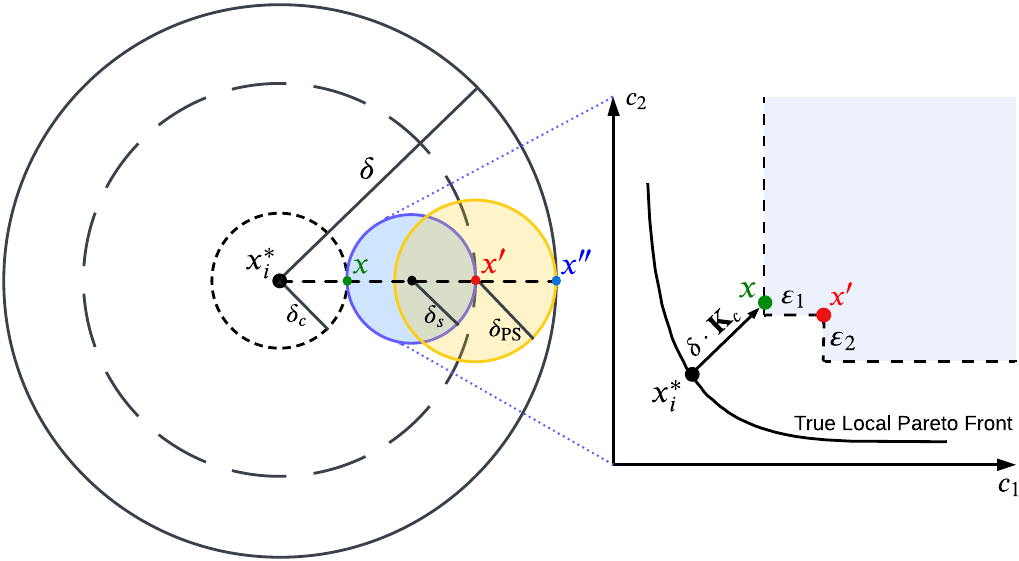} 
    \caption{The witness and selection radii used in \poSST.}
    \label{fig:radiiSchematic}
\end{figure}

The \textsc{PruneDominated} procedure ensures that a node $x \in \mathcal{B}_{\delta_c}(x_i^*)$ is pruned only if there exists another node $x' \in \mathcal{B}_{\delta_c + 2\delta_s}(x_i^*)$ such that $C(x') \succeq C(x) + \witEps$. The presence of such a node $x'$ may lead to the pruning of any node it $\varepsilon$-dominates, as indicated by the gray region in Figure~\ref{fig:radiiSchematic} (right).

Consequently, for any reference trajectory $\pi^* \in \poDSol$, there always exists a node $x' \in V$ that deviates from $x$ by at most $\varepsilon_i$ in some (but not all) cost components and remains within a distance of $\delta - \dps$ from $x_i^*$ for all iterations $n' > n$.
\end{proof} 

\subsection{Proof for Lemma~\ref{lem:poSST}}
\label{proof:lemmaPoSST}

\begin{proof}
   Refer to Fig.~\ref{fig:radiiSchematic} for visualization. The \textsc{ParetoSelect} procedure considers all non-dominated nodes in the neighborhood $\mathcal{B}_{\dps}(x_{\text{rand}})$. For the algorithm to consider only nodes within $\mathcal{B}_\delta(x^*_i)$, the sampled state $x_{\text{rand}}$ must lie in $\mathcal{B}_{\delta - \dps}(x^*_i)$. 
    In the worst case, a node $x \in \mathcal{B}_{\delta_c}(x^*_i)$ is replaced by a node $x' \in \mathcal{B}_{\delta - \dps}(x^*_i)$, lying on the outer edge of that ball. For $x'$ to be considered, the sample $x_{\text{rand}}$ must fall within the intersection of $\mathcal{B}_{\delta - \dps}(x^*_i)$ and $\mathcal{B}_{\dps}(x')$. Once $x'$ is considered, then it will have a non-zero probability of selection unless it is strictly dominated by another node $x'' \in \mathcal{B}_{\dps}(x')$, which, in the worst case, lies on the outer edge of $\mathcal{B}_{\delta}(x_i^*)$. Therefore, the probability of selecting such a node is lower-bounded by:
    $\gamma_{\text{PS}} = \frac{\mu \left( \mathcal{B}_{\delta - \dps}(x^*_i) \cap \mathcal{B}_{\dps}(x') \right)}{|\mathcal{B}_{\dps}(x')| \cdot \mu(\X)} > 0.$
    \noindent
\end{proof}

\subsection{Proof for Theorem~\ref{thm:poSST-complete}}
\label{proof:poSST-complete}
\begin{proof}
Given that the initial state satisfies $x_0 = \pi^*(0) \in \mathcal{B}_{\delta_c}(\pi^*(0))$ for all $\pi^* \in \Pi^*$, and that $V$ is initialized as $\{x_0\}$, we can inductively establish the continued progress of $\poSST$ in generating a $\delta$-similar trajectory to any reference trajectory $\pi^*$. At each iteration, the following two conditions are guaranteed: (1) with probability $\rho_{\delta \rightarrow \delta_c}>0$, the propagation procedure \textsc{MonteCarloProp} successfully transitions from $\mathcal{B}_{\delta}(x^*_i)$ to the next ball $\mathcal{B}_{\delta_c}(x^*_{i+1})$ along the reference trajectory by \cite[Theorem 17]{SST}; and (2) with probability $\gamma_{\text{PS}} > 0$, a node within $\mathcal{B}_{\delta }(x^*_i)$ is selected for extension by Lemma~\ref{lem:poSST}. 
Together, these properties guarantee that at every iteration, $\poSST$ maintains a non-zero probability of generating a node within each ball $\mathcal{B}_{\delta_c}(x^*_i)$ and selecting a node in $\mathcal{B}_{\delta}(x^*_i)$ for all $i = 0, 1, \dots, M - 1$, thereby ensuring probabilistic visitation of every ball in a covering sequence about any reference path.
\end{proof}

\subsection{Proof for Theorem~\ref{thm:poSST-optimal}}
\label{proof:poSST-optimal}
\begin{proof}

Consider a $\delta$-similar trajectory segment $\overline{x_j \rightarrow x_{j+1}}$ generated by \poSST, where $x_j \in \mathcal{B}_\delta(x_j^*)$ and $x_{j+1} \in \mathcal{B}_\delta(x_{j+1}^*)$ for some optimal trajectory $\pi^*$. Then, due to the Lipschitz continuity of the cost functions, each cost component $c_i$ satisfies:
% 
% \[
$c_i(\overline{x_j \rightarrow x_{j+1}}) \leq c_i(\overline{x_j^* \rightarrow x_{j+1}^*}) + \delta \cdot K_{c,i}.$
% \]
Furthermore, from Lemma~\ref{lem:pruneDom}, we know that if a node $x_j \in \mathcal{B}_{\delta - \dps}(x_j^*)$ exists at iteration $n$, then there exists at least one node $x_j' \in V$ that is either $\varepsilon$-dominant or equal to $x_j$ in the same ball $\mathcal{B}_{\delta - \dps}(x_j^*)$. Lemma~\ref{lem:ParetoSelect} ensures that $x_j'$ will be selected with probability at least $\gamma_{\text{PS}}>0$, or else a dominant node $x_j'' \in \mathcal{B}_\delta(x_j^*)$ will be selected. 
% Note that $x_j$, $x_j'$, and $x_j''$ may refer to the same node. 
Therefore, we have
\[
C(x_j'') \succeq C(x_j') \succeq C(x_j) + \witEps
\]
for all future iterations of the algorithm.

Now, consider the first segment $\overline{x_0 \rightarrow x_1}$ of a trajectory generated by $\poSST$  that approximates a Pareto-optimal trajectory $\pi^*$. Then, the cost of a representative segment at an iteration $n'>n$ can be bounded as:
$$\quad C(\overline{x_0 \rightarrow x_1'}) \succeq C(\overline{x_0 \rightarrow x_1}) \succeq C(\overline{x_0 \rightarrow x_1^*}) + \delta \cdot \mathbf{K}_c + \vec\varepsilon.$$
% 
% Assuming this bound holds for every segment $\overline{x_i \rightarrow x_{i+1}}$, 
We can extend the bound to a $k$-segment trajectory from $x_0$:
\begin{align}
C(\overline{x_0 \rightarrow x_k'}) &\succeq C(\overline{x_0 \rightarrow x_k}) = \sum_{j=0}^{k} C(\overline{x_j \rightarrow x_{j+1}}) \notag \\
&\succeq \sum_{j=0}^{k} \left[C(\overline{x_j \rightarrow x_{j+1}^*}) + \delta \cdot \mathbf{K}_c + \vec\varepsilon \right] \notag \\ 
&= C(\overline{x_0 \rightarrow x_k^*}) + k  (\delta \cdot \mathbf{K}_c + \vec\varepsilon)
\end{align}
Let $\Delta c_i$ denote the minimum cost incurred in objective $i$ among all segments of $\pi^*$ in $\mathbb{B}(\pi^*, \delta, \Delta c_1)$. Then, the maximum number of segments $k$ is bounded by $k \leq \frac{c_i(\pi^*)}{\Delta c_i}$ for each $i$, yielding the final sub-optimality bound:
% 
% \[
% C(\pi') \preceq C(\pi^*) + \frac{C(\pi^*)}{C_\Delta}(\delta \cdot \mathbf{K}_c + \vec\varepsilon) = \left(1 + \frac{\delta \cdot \mathbf{K}_c + \vec\varepsilon}{C_\Delta} \right) \cdot C(\pi^*)
% \]
\begin{equation}
    \label{eq:boundsAppx}
    c_i(\pi) \leq \left(1 + \frac{\delta \cdot K_{c, i} + \varepsilon_i}{\Delta c_i} \right) \cdot c_i(\pi^*) \quad \forall i.
\end{equation}
\end{proof}

\subsection{Proof for Lemma~\ref{lemma:lexSST}}
\label{proof:lemmaLexSST}
\begin{proof}
By Theorems~\ref{thm:poSST-complete} and~\ref{thm:poSST-optimal}, the \textsc{ParetoSelect} and \textsc{PruneDominated} procedures do not hinder the generation of a $\delta$-similar trajectory to any Pareto-optimal solution. Since $\lexMin$ is Pareto-optimal by definition, a node $x_j \in \mathcal{B}_\delta(\lexBall)$ that lies within the covering ball sequence of $\lexMin$ will have a primary cost that satisfies \eqref{eq:bounds} after the $j$-th segment:
$c_1(x_j) \leq c_1(\lexBall) + j\cdot (\delta \cdot K_{c,1} + \varepsilon_1).$
\textsc{PruneLexSet} removes all nodes exceeding $\bar c_1^* + \epsilon$, where $\bar c_1^* = \min_{x \in \repLex} c_1(x)$. Since $\epsilon \geq k(\delta \cdot K_{c,1} + \varepsilon_1)$, $ \bar c_1^* \geq c_1(\lexBall)$, and $k \geq j$, we get:
\[
c_1(x_j) \leq c_1(\lexBall) + j(\delta \cdot K_{c,1} + \varepsilon_1) \leq \bar c_1^* + \epsilon,
\]
and $x_j$ remains in $\repLex$.
\end{proof}

\subsection{Proof for Theorem~\ref{thrm:lexComplete}}
\label{proof:lexComplete}
\begin{proof}
% The proof closely follows that of Theorem~\ref{thm:poSST-complete}. 
The search tree is initialized with $x_0 \in \mathcal{B}_{\delta_c}(\lexMin(0))$. At every iteration: Lemma~\ref{lem:ParetoSelect} guarantees a positive probability of selecting a node in $\mathcal{B}_{\delta}(\lexBall)$; Lemma~\ref{lemma:lexSST} ensures \textsc{PruneLexSet} retains at least one such node; and \cite[Theorem 17]{SST} guarantees a positive probability of transitioning to $\mathcal{B}_{\delta_c}(x^*_{\text{lex},j+1})$. Together, these ensure \lexSST makes progress toward a $\delta$-similar trajectory to $\lexMin$ at every iteration.
\end{proof}

\subsection{Proof for Theorem~\ref{thrm:lexOptimal}}
\label{proof:lexOptimal}
\begin{proof}
The proof follows that of Theorem~\ref{thm:poSST-optimal}, restricted to $\lexMin$. Since \lexSST is guaranteed to find a $\delta$-similar trajectory $\hat\pi_\text{lex}$ to $\lexMin$, the bound in \eqref{eq:bounds} gives:
\[
c_i(\hat\pi_\text{lex}) \leq \left(1 + \frac{\delta \cdot K_{c,i} + \varepsilon_i}{\Delta c_i}\right) \cdot c_i(\lexMin) \ \text{ for } i\in \{1, 2\}.
\]
By Lemma~\ref{lemma:lexSST}, the solution set $\Sols$ always contains a trajectory $\hat\pi_\text{lex}'$ with $C(\hat\pi_\text{lex}') \succeq C(\hat\pi_\text{lex})$. The returned solution $\lexApprox$ incurs an additional $\epsilon$-penalty on the primary cost bound due to the $\epsilon$-window, yielding: 
\begin{equation}
\label{eq:lex-optimality}
c_1(\lexApprox) \leq \left(1 + \frac{\delta \cdot K_{c,1} + \varepsilon_1}{\Delta c_1}\right) \cdot c_1(\lexMin) + \epsilon.
\end{equation}
\end{proof}

\subsection{Proof for Lemma~\ref{lemma:conSST}}
\label{proof:conSST}
\begin{proof}
The proof follows that of Lemma~\ref{lemma:lexSST}, replacing the $\epsilon$-window on the primary cost with per-constraint buffers $\bar\epsilon_i$. Since $\conMin$ minimizes $c_N$ and is thus Pareto-optimal, a node $x_j$ within the covering ball sequence of $\conMin$ satisfies \eqref{eq:bounds}:
\[
c_i(x_j) \leq c_i(\conBall) + j\cdot (\delta \cdot K_{c,i} + \varepsilon_i) \quad \forall i.
\]
\textsc{PruneConSet} removes nodes exceeding $\bar{c}_i + \bar\epsilon_i$. Since $\bar\epsilon_i \geq k(\delta \cdot K_{c,i} + \varepsilon_i)$, $\bar{c}_i \geq c_i(\conBall)$, and $k\geq i$, we get $c_i(x_j) \leq \bar{c}_i + \bar\epsilon_i$ and $x_j$ is retained in $\repCon$.
\end{proof}

\subsection{Proof for Theorem~\ref{thrm:conComplete}}
\label{proof:conComplete}

\begin{proof}
Follows the outline of Theorem~\ref{thrm:lexComplete}, with Lemma~\ref{lemma:conSST} replacing Lemma~\ref{lemma:lexSST} to guarantee that \textsc{PruneCoSet} retains at least one node in $\mathcal{B}_\delta(\conBall)$ at every iteration.
\end{proof}

\subsection{Proof for Theorem~\ref{thrm:conOptimal}}
\label{proof:conOptimal}

\begin{proof}
Follows directly from Theorem~\ref{thm:poSST-optimal}, restricted to $\conMin$, which ensure a motion plan $\conApprox$ will be generated that satisfies:
\begin{equation*}
\label{eq:con-optimality}
c_i(\hat\pi_\text{con}) \leq \left(1 + \frac{\delta \cdot K_{c,i} + \varepsilon_i}{\Delta c_i}\right) \cdot c_i(\conMin) \quad\forall i.
\end{equation*}
By Lemma~\ref{lemma:conSST}, the solution set $\Sols$ will contain a trajectory $\hat\pi_\text{con}'$ with $C(\hat\pi_\text{con}') \succeq C(\hat\pi_\text{con})$ for iterations $n' >n$. Then, the returned solution $\conApprox = \arg\min_{\pi \in \Sols} c_N(\pi)$ satisfies:
\begin{equation}
\label{eq:con-optimality2}
c_i(\conApprox) \leq \left(1 + \frac{\delta \cdot K_{c,i} + \varepsilon_i}{\Delta c_i}\right) \cdot c_i(\conMin) + \mathbf{1}_{i \neq N} \cdot \bar\epsilon_i,
\end{equation}
where $\mathbf{1}_{\{\cdot\}}$ is an indicator function applied to the buffer $\bar\epsilon_i$ which accounts for the fixed constraint window, analogous to the $\epsilon$-penalty in Eq.~\eqref{eq:lex-optimality}.
\end{proof}

\bibliographystyle{unsrt}
\bibliography{refs}

\end{document}